\documentclass[twoside,11pt]{article}

\usepackage[abbrvbib, usehyper]{jmlr2e}
\usepackage{amsmath}
\usepackage{amsfonts}
\usepackage{amssymb}
\usepackage{graphicx}
\usepackage{dsfont}
\usepackage{empheq}
\usepackage[dvipsnames]{xcolor}
\usepackage[protrusion=true,final]{microtype}
\usepackage{booktabs}
\usepackage{algorithm}
\usepackage{algpseudocode}
\usepackage{subcaption}
\usepackage{xprintlen}
\usepackage{xfrac}
\usepackage{float}
\usepackage{blkarray}
\newtheorem{assumption}{Assumption}

\newcommand\bovermat[2]{    \makebox[0pt][l]{$\smash{\overbrace{\phantom{                    \begin{matrix}#2\end{matrix}}}^{\text{#1}}}$}#2}

\DeclareMathOperator{\var}{Var}

\ShortHeadings{Empirical Risk Minimization under Random Censorship}{G. Ausset, S. Cl\'emen\c{c}on \& F. Portier}
\firstpageno{1}

\begin{document}

\title{Empirical Risk Minimization under Random Censorship:\\ Theory and Practice}

\author{\name Guillaume Ausset \email guillaume.ausset@telecom-paris.fr \\
              \addr LTCI, T\'el\'ecom Paris, Institut Polytechnique de Paris\\
              \addr BNP Paribas\\
       \AND
       \name Stephan Cl\'emen\c{c}on \email stephan.clemencon@telecom-paris.fr \\
                     \addr LTCI, T\'el\'ecom Paris, Institut Polytechnique de Paris\\       
                 \AND
                    \name Fran\c{c}ois Portier \email francois.portier@telecom-paris.fr \\
              \addr LTCI, T\'el\'ecom Paris, Institut Polytechnique de Paris}

\editor{}

\maketitle

\begin{abstract}
We consider the classic supervised learning problem, where a continuous non-negative random label $Y$ (\textit{i.e.} a random duration) is to be predicted based upon observing a random vector $X$ valued in $\mathbb{R}^d$ with $d\geq 1$ by means of a regression rule with minimum least square error. In various applications, ranging from  industrial quality control to public health through credit risk analysis for instance, training observations can be \textit{right censored}, meaning that, rather than on independent copies of $(X,Y)$, statistical learning relies on a collection of $n\geq 1$ independent realizations of the triplet $(X, \; \min\{Y,\; C\},\; \delta)$, where $C$ is a nonnegative r.v. with unknown distribution, modeling censorship and $\delta=\mathbb{I}\{Y\leq C\}$ indicates whether the duration is right censored or not. As ignoring censorship in the risk computation may clearly lead to a severe underestimation of the target duration and jeopardize prediction, we propose to consider a \textit{plug-in} estimate of the true risk based on a Kaplan-Meier estimator of the conditional survival function of the censorship $C$ given $X$, referred to as \textit{Kaplan-Meier risk}, in order to perform empirical risk minimization.  It is established, under mild conditions, that the learning rate of minimizers of this biased/weighted empirical risk functional is of order $O_{\mathbb{P}}(\sqrt{\log(n)/n})$ when ignoring model bias issues inherent to plug-in estimation, as can be attained in absence of censorship. Beyond theoretical results, numerical experiments are presented in order to illustrate the relevance of the approach developed.
\end{abstract}


\section{Introduction}\label{sec:intro}

Covering a wide variety of practical applications, distribution-free regression can be considered as one of the flagship problems in statistical learning. In the most standard setup, $(X,Y)$ is a random pair defined on a certain probability space with (unknown) joint probability distribution $P$, where the output r.v. $Y$ is a real-valued square integrable r.v. and $X$ models some input information, valued in $\mathbb{R}^d$, supposedly useful to predict $Y$. In this context, one is interested in building a (measurable) function $f:\mathbb{R}^d\rightarrow \mathbb{R}$ minimizing  the (expected quadratic) risk
\begin{equation}\label{eq:risk}
 R_P(f)=\mathbb{E}\left[\left(Y- f(X)\right)^2\right],
\end{equation}
which is finite as soon as the r.v. $f(X)$ is square integrable. Obviously, the minimizer of \eqref{eq:risk} is the \textit{regression function} $f^{\star}(X)=\mathbb{E}[Y\mid X]$. As the distribution of $(X,Y)$ is unknown in practice, the Empirical Risk Minimization paradigm (ERM in abbreviated form, see \textit{e.g.} \cite{GKKW06}) suggests considering solutions $\widehat{f}_n$ of the minimization problem, also referred to as \textit{least squares regression},
$
\min_{f\in \mathcal{F}}\widehat{R}_n(f),
$
where $\widehat{R}_n(f)$ is a statistical estimate of the risk $R_P(f)$ computed from a training sample $\mathcal{D}_n=\{(X_1,Y_1),\; \ldots,\; (X_n,Y_n) \}$ of independent copies of $(X,Y)$. In general the empirical version
\begin{equation}\label{eq:emp_risk}
\widehat{R}_n(f)=\frac{1}{n}\sum_{i=1}^n\left( Y_i-f(X_i)\right)^2
\end{equation}
is considered. This boils down to replacing $P$ in the risk functional $R_P(\cdot)$ with the empirical distribution of the $(X_i,Y_i)$'s.   The class  $\mathcal{F}$ of predictive functions is  supposed to be of controlled complexity (\textit{e.g.} of finite {\sc VC} dimension), while being rich enough to contain a reasonable approximant of the minimizer of $R_P$, $f^{\star}(x)$.
In a framework stipulating in addition that the random variables $Y$ and $f(X)$, $f\in \mathcal{F}$, are sub-Gaussian, ERM is proved to yield rules with good generalization properties, see \textit{e.g.} \cite{GKKW06, BBM05, LM16} (notice, however, that, in heavy-tail situations, alternative strategies are preferred, refer to \cite{LM17} for instance).

In many applications such as  industrial reliability, see \cite{Mann75}, or clinical trials, the r.v. $Y$ to be predicted represents a duration, \textit{e.g.} the lifespan of a manufactured component or the time to recovery of a diseased patient, and it is far from uncommon in survival analysis that the data at disposal to learn a predictive rule are not composed of independent realizations $(X_1,Y_1),\; \ldots,\; (X_n,Y_n)$ of distribution $P$ but of observations $(X_1,\; \tilde{Y}_1,\; \delta_1),\; \ldots,\; (X_n,\; \tilde{Y}_n,\; \delta_n)$, where the observed durations are of the form
\begin{equation}\label{eq: train_data_cens}
\tilde{Y}_i=\min\{C_i,\; Y_i \}\; \text{ with } i\in\{1,\; \ldots,\; n\},
\end{equation}
the random variables $C_i$'s modelling a possible right censorship, and the $\delta_i$'s are binary variables indicating whether censorship has occurred for each duration.
Of course, other types of censorship (\textit{e.g.} left/interval/progressive censorship) can be encountered in practice and result in partially observed durations. Since the results established in this paper can be straightforwardly extended to a more general framework, focus is on the right censorship case here.
Whereas the asymptotic theory of statistical estimation based on censored data is very well documented in the literature (see \textit{e.g.} \cite{fleming+h:2011,andersen:2012} and the references therein), the issues raised by censorship in statistical learning has received much less attention and it is the major purpose of this article to investigate how ERM can be extended to this setup with sound generalization guarantees. As the empirical risk \eqref{eq:emp_risk} cannot be computed from the data available, we propose to build first a plug-in (biased) estimator of the risk \eqref{eq:risk} by means of a Kaplan-Meier type estimator of the conditional survival function of the censorship \citep{beran:1981,Dabrowska89,vankeilegom+v:1996} and minimize next the resulting risk estimate, referred to as \textit{Kaplan-Meier risk} and that can be interpreted as a weighted version of the empirical risk process based on the observations. The use of weights to account for the presence of censorship has been first considered in the seminal contributions of \cite{stute:1993,stute:1996} and refined recently in \cite{lopez:2011,lopez+p+v:2013}, where the asymptotics of such weighted averages are studied. In this paper, more in the spirit of the popular statistical learning theory of empirical risk minimization, nonasymptotic maximal deviation bounds for this risk functional, much more complex than a basic empirical process due to the strong dependency exhibited by the terms averaged to compute it, are established by means of linearization techniques combined with concentration results pertaining to the theory of $U$-processes. We prove that, under appropriate conditions, minimizers of the Kaplan-Meier risk proposed have good generalization properties, achieving learning rate bounds of order $O_{\mathbb{P}}(\sqrt{\log(n)/n})$ when ignoring the model bias impact on the plug-in estimation step, as ERM in absence of any censorship. Beyond this theoretical analysis, illustrative numerical results are also displayed, providing strong empirical evidence of the relevance of the approach promoted. They reveal in particular that, even if the estimator of the conditional survival function plugged is only moderately accurate, Kaplan-Meier risk minimizers significantly outperform approaches ignoring censorship. Eventually, we point out that some of the results established in this paper have been preliminarily presented in an elementary form at the 2018 NeurIPS ML4Health Workshop, see \cite{ML4H}.

The rest of the paper is organized as follows. The framework we consider for statistical learning based on censored training data is detailed in section \ref{sec:prel}, where notions pertaining to  survival data analysis involved in the subsequent study are also briefly recalled and a nonasymptotic uniform bound for a kernel-based Kaplan-Meier estimator of the conditional survival function of the censorship is also stated.  In section \ref{sec:main}, the statistical version of the expected quadratic risk we propose, based on the conditional Kaplan-Meier estimator previously studied, is introduced and the performance of its minimizers is analysed. Illustrative numerical results are displayed in section \ref{sec:exp}, while several concluding remarks are collected in section \ref{sec:conclusion}. Technical proofs are postponed to the Appendix section.
 
\section{Background - Preliminaries}\label{sec:prel}

In this section, we first describe at length the probabilistic setup considered in this paper and recall basic concepts of \textit{censored data analysis}, which the subsequent analysis relies on, such as (conditional) Kaplan-Meier estimation. Next, we establish a nonasymptotic bound for the deviation between the conditional survival function of the random censorship and its Kaplan-Meier estimator under adequate smoothness assumptions. Here and throughout, the indicator function of any event $\mathcal{E}$ is denoted by $\mathbb{I}\{ \mathcal{E}\}$, the Dirac mass at any point $x$ by $\delta_x$. When well-defined, the convolution product between two real-valued Borelian functions on $\mathbb{R}^d$ $g(x)$ and $w(x)$ is denoted by $(g\ast w)(x) = \int_{x'\in \mathbb{R}^d}g(x-x')w(x')dx'$. The left-limit at $s>0$ of any c\`adla\`g function $S$ on $\mathbb{R}_+$ is denoted by $S(s-)=\lim_{t\uparrow s}S(t)$.

\subsection{The Statistical Framework}
In this paper,  we consider a pair $(X,Y)$ of random variables defined on the same probability space $(\Omega,\; \mathcal{A},\mathbb{P})$, with unknown joint distribution $P$ and where $Y$, representing a duration, takes nonnegative values only and $X$ models some information valued in $\mathbb{R}^d$, $d\geq 1$, a priori useful to predict $Y$. We assume that $X$'s marginal distribution has a density $g(x)$ w.r.t. Lebesgue measure on $\mathbb{R}^d$. We are concerned with building a prediction rule $f:\mathbb{R}^d\rightarrow \mathbb{R}_+$ with minimum expected quadratic risk $R_P(f)$, see Eq. \eqref{eq:risk}, based on a training dataset $\tilde{D}_n=\{(X_1,\; \tilde{Y}_1,\; \delta_1),\; \ldots,\; (X_n,\; \tilde{Y}_n,\; \delta_n)\}$ composed of $n\geq 1$ independent realizations of the random triplet $(X,\; \tilde{Y},\; \delta)$, where $\tilde{Y}=\max\{Y,\; C\}$, $C$ is a nonnegative r.v. defined on $(\Omega,\; \mathcal{A},\mathbb{P})$ and $\delta=\mathbb{I}\{Y\leq C\}$ indicates whether the duration is (right) censored ($\delta=0$) or not ($\delta=1$). The following hypothesis is required in the present study.
\begin{assumption}\label{hyp:cond_ind} {\sc (Conditional independence)}
The random variables $Y$ and $C$ are conditionally independent given the input $X$ and we have $Y \neq C$ with probability one.
\end{assumption}
 Naturally, many other types of censorship can be encountered in practice. However, since the goal  of the present paper is to explain the main ideas to apply the ERM principle to censored data rather than dealing with the problem at the highest level of generality, we restrict our attention to the type of right random censorship introduced above. Though simple, it covers many situations. Addressing the problem in a more complex probabilistic framework, where $Y$ and $C$ are not conditionally independent given $X$ anymore for instance, will be the subject of future research. The assumption stipulating that $\{Y=C\}$ is a zero-probability event is quite general, insofar as it allows considering situations where $Y$ and/or $C$ are discrete variables. Under conditional independence, it is obviously satisfied when the r.v. $Y$ is continuous.

 Easy to state but difficult to solve, the statistical learning problem we consider here is of considerable importance. In a wide variety of applications, the input information is of increasing granularity  and described by a random vector of very large dimension $d$, while (censored) data are progressively becoming massively available. Machine-learning techniques are thus expected  to complement traditional approaches, based on statistical modelling, in order to produce more flexible/accurate predictive models based on censored data. Incidentally, we point out that the problem under study can be viewed as a very specific type of \textit{transfer learning} problem, see \textit{e.g.} \cite{5288526} insofar as,
due to the censorship, the distribution of the training/source data is not that of the test/target data. However, the source domain coincides here with the target one and the predictive task (regression) remains the same.
\medskip

\noindent {\bf Weighted empirical risk.} Discarding censored observations to evaluate the risk of a candidate function $f(x)$ would lead to the quantity
\begin{equation}\label{eq:naive_risk}
\bar{R}_n(f)=\sum_{i=1}^n\delta_i\left( \tilde{Y}_i-f(X_i) \right)^2/\sum_{i=1}^n\delta_i,
\end{equation}
with $0/0=0$ by convention, which is clearly a biased estimate of $R_P(f)$ in general, since, by virtue of the strong law of large numbers, it converges to $\mathbb{E}[(Y-f(X))^2\mid Y\leq C]$ with probability one. One may easily check that the minimizer of this functional is given by
$$
\bar{f}^*(X)=\mathbb{E}[Y\mathbb{I}\{Y\leq C \}\mid X]/\mathbb{P}\{Y\leq C\mid X\},
$$
which significantly differs from $f^*(X)$ in general. Observing that, by means of a straightforward conditioning argument, one can write the risk as
\begin{equation}\label{eq:risk2}
R_P(f)=\mathbb{E}\left[\frac{\delta(\tilde{Y}-f(X))^2 }{S_{C}( \tilde{Y}-\mid X)} \right],
\end{equation}
where $S_{C}(u\mid x)=\mathbb{P}\{ C> u \mid X=x \}$ denotes the conditional survival function of the random right censorship given $X$, we propose to estimate the risk \eqref{eq:risk} by computing first a nonparametric estimator $\hat{S}_{C}(u\mid x)$ of $S_{C}(u\mid x )$ and by plugging it next into \eqref{eq:risk2}, so as to obtain
\begin{equation}\label{eq:KMemp_risk}
\widetilde{R}_n(f)=\frac{1}{n}\sum_{i=1}^n\frac{\delta_i(\tilde{Y}_i-f(X_i))^2}{\hat{S}_{C}(\tilde{Y}_i - \mid X_i)},
\end{equation}
which approximates the unknown quantity whose expectation is equal to \eqref{eq:risk2}
\begin{equation}\label{eq:emp_risk2}
\frac{1}{n}\sum_{i=1}^n\frac{\delta_i(\tilde{Y}_i-f(X_i))^2}{S_{C}(\tilde{Y}_i - \mid X_i)},
\end{equation}
the conditional survival function of $C$ given $X$ being itself unknown.
Observe that the risk estimate \eqref{eq:KMemp_risk} can be viewed as a \textit{weighted version} of the sum of the observed squared errors $(\tilde{Y}_i-f(X_i))^2$, just like \eqref{eq:naive_risk} except that the $i$-th weight is not $\delta_i/\sum_{j\leq n}\delta_j$ anymore but $\delta_i/\hat{S}_{C}(\tilde{Y}_i - \mid X_i)$. In the terminology of survival analysis, the weighted empirical risk \eqref{eq:KMemp_risk} is usually referred to as an IPCW risk estimate, IPCW standing for \textit{inverse of the probability of censoring weight}, \textit{i.e.} the squared error related to the observation $(X_i,\tilde{Y}_i)$ being weighted by the inverse of the conditional probability of not being censored.
A natural strategy to learn a predictive function in the censored framework described above then consists in solving the minimization problem
\begin{equation}\label{eq:ERM_KM}
\inf_{f\in \mathcal{F}}\widetilde{R}_n(f),
\end{equation}
over an appropriate class $\mathcal{F}$. When using the Kaplan-Meier approach (\textit{cf} \cite{KM58}) to estimate $S_{C}(u\mid x)$, as detailed in the next subsection, the functional \eqref{eq:KMemp_risk} is referred to as the Kaplan-Meier risk throughout the article. Based on accuracy results for kernel-based Kaplan-Meier estimators of the conditional survival function $S_{C}(\cdot \mid x)$ such as those subsequently presented, the performance of solutions of \eqref{eq:ERM_KM} is investigated in the next section. We point out that, as highlighted in section \ref{sec:exp}, alternative inference strategies for conditional survival function estimation can be considered. For simplicity, here we restrict our attention to kernel-smoothing techniques, although the analysis carried out can be extended to other nonparametric methods (\textit{e.g.} partition-based techniques, nearest neighbours).
\medskip

\noindent {\bf Integration domain.} As any (conditional) survival function, $S_C(y\mid x)$ vanishes as $y$ tends to infinity. In order to avoid dealing with the asymptotic behaviour of the conditional survival function of the censorship and stipulating decay rate assumptions for its tail behaviour, in the analysis carried out in section \ref{sec:main}  we restrict the study of the prediction problem to a (borelian) domain $\mathcal{K}\subset \mathbb{R}_+\times \mathbb{R}^d$ such that  $S_C(y\mid x)$ stays bounded away from $0$ on it and consider the risk
\begin{equation}\label{eq:risk_K}
R_{P,\mathcal{K}}(f)=\mathbb{E}\left[ \frac{\delta \left(\tilde{Y}-f(X)\right)^2}{S_C(\tilde{Y}-\mid X)}\mathbb{I}\{ (\tilde{Y}, X)\in \mathcal{K} \} \right],
\end{equation}
as well as its Kaplan-Meier counterpart
\begin{equation}\label{eq:KMemp_risk_K}
\frac{1}{n}\sum_{i=1}^n\frac{\delta_i(\tilde{Y}_i-f(X_i))^2}{S_{C}(\tilde{Y}_i - \mid X_i)}\mathbb{I}\{(\tilde{Y}_i, X_i)\in \mathcal{K}  \}.
\end{equation}

\noindent {\bf Related work.}
Because the risk considered here can be expressed as an integral with respect to the joint distribution of $(Y,X)$, the predictive problem under study can be linked to other works, dealing with the estimation of the joint distribution of $(Y,X)$ in particular. This problem is investigated in \cite{stute:1993,stute:1996} where the authors propose a weighted approach based on the estimation of the conditional survival function of $C$ given $X$. Incidentally, observe that, even if the censorship model is free from any parametric modelling, the assumptions involved in this analysis are quite strong as the distribution of $C$ is supposed to be independent from $X$. In particular, the weights used are independent from $X$. Application to parametric predictive modelling such as \textit{linear} regression is also considered. Other approaches are considered in  \cite{akritas:1994,vankeilegom+a:1999}, where the joint distribution estimator is computed from an empirical average over $X$ of the Kaplan-Meier estimate of the conditional distribution of $Y$ given $X$.
In \cite{lopez:2011}, the author proposes a kernel-based weighted method, more general than that proposed in \cite{stute:1993,stute:1996} relaxing in particular the restrictive assumption on the dependence between $C$ and $X$.  An asymptotic representation of the estimation error is established when the input variable is univariate ($d=1$). An extension with a single index model is considered in \cite{lopez+p+v:2013}. The proof technique is based on the asymptotic equicontinuity of the empirical process and imposes strong conditions on the bandwidth choice, e.g.  $nh^{3}\to \infty $ (see Theorem 3.3 in \cite{lopez:2011} and Theorem 3.1 in \cite{lopez+p+v:2013}). The (nonasymptotic) analysis carried out in this paper is quite different, since it is carried out in two steps: 1) linearize the risk estimate and 2) use concentration results for generalized $U$-processes to describe its behaviour (see \textit{e.g.} \cite{clemenccon+p:2018}). Notice additionally that the approach we adopted to establish nonasymptotic rate bounds requires weaker conditions, only that $nh^d/|\log(h)|\to \infty$ in the $d$-dimensional case.
Similar approaches were proposed in \cite{bang+t:2002} and \cite{orbe_comparing_2002} where $S(Y \mid X)$ is modelled in a parametric fashion and the Kaplan-Meier (KM) risk formulation with nonparametric Kaplan-Meier weights is then used to estimate the parameters. Alternatively, it is possible to use parametric estimate of  $S_C(Y \mid X)$ (instead of KM) in order to obtain an estimator of a certain risk, as in \cite{rotnitzky_recovery_1992, van_der_laan_unified_2003} for instance. Other related approaches can be found in \cite{gerds_kaplan-meier_2017}.

\subsection{Preliminary Results}\label{subsec:prel}
 In this subsection, we briefly recall the Kaplan-Meier approach to estimate a (conditional) survival function by means of a kernel smoothing procedure and state a uniform bound for the deviations between the conditional survival function of $C$ given $X$ and its Kaplan-Meier estimator, involved in statistical learning framework developed in the next section for distribution-free censored regression. As shall be discussed below, this result refines those obtained in \cite{Dabrowska89} and \cite{du+a:2002}, which are of similar nature, except that they are related to the estimation of the conditional survival function of the duration $Y$ given $X$, denoted by $S_Y(u\mid x)=\mathbb{P}\{ Y>u \mid X=x\}$, rather than that of the conditional survival function of the censorship $C$ given $X$. Define the conditional integrated hazard function of the right censorship $C$ given $X$
\begin{align}\label{eq:cond_int_hazard}
\Lambda_C(u  \mid x) &= - \int_0^u \frac{S_C(ds\mid x)}{S_C(s-\mid x)} .
\end{align}
and the conditional subsurvival functions
$H(u \mid x) = \mathbb{P} \{ \tilde{Y} > u \mid X=x\}$ and
  $H_0 (u \mid x) = \mathbb{P} \{ \tilde{Y} >u,\;  \delta = 0 \mid X=x\}$ for $u\geq 0$ and $x\in \mathbb{R}^d$.
  As we have (under Assumption \ref{hyp:cond_ind}), $H_0(du\mid x)=S_Y(u-\mid x)S_C(du\mid x)$ and $H(u- \mid x)=S_Y(u- \mid x)S_C(u-\mid x)$, we obtain
\begin{align*}
\Lambda_C(u  \mid x) &=- \int_0^u \frac{H_0(ds\mid x)}{H(s -\mid x)}.
\end{align*}
Here, we propose to build an estimate of $\Lambda_{C}(u\mid x)$ by plugging into formula \eqref{eq:cond_int_hazard} Nadaraya-Watson type kernel estimates of the conditional subsurvival functions and derive from it an estimator of $S_{C}(u\mid x)$. Of course, alternative estimation techniques can be considered for this purpose.
 Throughout the paper, $K:\mathbb{R}^d\rightarrow \mathbb R^+$ is a symmetric bounded \textit{kernel function}, \textit{i.e.} a bounded nonnegative Borelian function, integrable w.r.t. Lebesgue measure such that $\int K(x)dx=1$, $K(x)=K(-x)$ for all $x\in \mathbb{R}^d$, see \cite{WandJones94}. We assume it lies in the linear span of functions $w$, whose subgraphs $\{(s, u): w(s) \geq u\}$, can be represented as a finite number of
Boolean operations among sets of the form $\{(s, u): p(s,u)\geq \zeta(u)\}$, where $p$ is
a polynomial on $\mathbb{R}^d \times \mathbb{R}$ and $\zeta$ an arbitrary real-valued function. This assumption guarantees that the collection of functions
$$
\{  K((x-\cdot)/h):\; x\in \mathbb{R}^d,\; h>0\}
$$
is a bounded {\sc VC} type class, see \cite{gine2004}.
Although very technical at first glance, this hypothesis is very general and is satisfied by kernels of the form $K(x) = \zeta(p(x))$, $p$ being any polynomial and $\zeta$ any
bounded real function of bounded variation (see \cite{nolan+p:1987}) or when the graph of $K$ is a pyramid (truncated or not).
For any bandwidth $h>0$ and $x\in \mathbb R^d$, we set $K_h(x)=K(h^{-1}x)/h^d$. Based on the kernel estimators given by
\begin{eqnarray}
 \hat{H}_{0,n}(u,x)&=& \frac{1}{n}\sum_{i=1}^n \mathbb{I}\{\tilde{Y}_i>u,\;  \delta_i=0\}K_h(x-X_i),\label{eq:kern1}\\
 \hat{H}_n(u,x)&=& \frac{1}{n}\sum_{i=1}^n \mathbb{I}\{\tilde{Y}_i>u\}K_h(x-X_i),\label{eq:kern2}\\
\hat{g}_n(x) &=& \frac{1}{n}\sum_{i=1}^n K_h(x-X_i),\label{eq:kern3}
 \end{eqnarray}
 define the conditional subsurvival function estimates
 \begin{equation*}
 \hat{H}_{0,n}(u\mid x)=\frac{\hat{H}_{0,n}(u,x)}{\hat{g}_n(x)} \text{ and }
  \hat{H}_n(u\mid x)=\frac{\hat{H}_{n}(u,x)}{\hat{g}_n(x)},
 \end{equation*}
as well as the (biased) estimators of $\Lambda_C(u\mid x)$ and $S_C(u\mid x)$
\begin{eqnarray}
\hat{\Lambda}_{C,n}(u\mid x)&=& - \int_0^u  \frac{\hat{H}_{0,n}(ds\mid x)}{\hat{H}_n(s-\mid x)},\label{eq:CNA_est}\\
\hat{S}_{C,n}(u\mid x)&=&\prod_{s\leq u} \left(1-d\hat{\Lambda}_{C,n}(s\mid x) \right)\label{eq:CKM_est}
\end{eqnarray}
which are classically referred to as the conditional Nelson-Aalen and Kaplan-Meier estimators \citep{Dabrowska89}.
 Let $b>0$ and define the set
\begin{align*}
{\Gamma_{b}=\left\{ (y,x) \in \mathbb R_+ \times \mathbb R^d \,:\, S_Y(y|x) \wedge S_C(y|x) \wedge g(x) \geq b \right\},}
\end{align*}
which is supposed to be non-empty. On this set, one may guarantee that $\hat H_{0,n}(y, x)$ and $\hat H_{0,n}(y, x)$  are both away from $0$ with high probability, which permits the study of the fluctuations of \eqref{eq:CKM_est}. The mild H\"older smoothness assumption below is also required in the analysis, the definition of  H\"older classes is recalled in the Appendix section for completeness.

\begin{assumption}\label{hyp:smooth} For all $u\in \mathbb R_+$, the functions $x\mapsto H(u\mid x)g(x)$ and $x \mapsto H_0(u\mid x)g(x)$ belong to the H\"older class $\mathcal{H}_{2, L}( \mathbb{R}^d)$.
\end{assumption}
\begin{assumption}\label{hyp:bound1} The density $g$ is bounded by $R<+\infty$, \textit{i.e.} $\vert\vert g\vert\vert_{\infty}\leq R$.
\end{assumption}
The result stated below provides a uniform bound for the deviation between $S_C(u \mid x)$ and its estimator \eqref{eq:CKM_est}.

\begin{proposition}\label{prop:bound_CKM} Suppose that Assumptions \ref{hyp:cond_ind}, \ref{hyp:smooth} and \ref{hyp:bound1} are fulfilled. Then, there exist constants $M_1>0$, $M_2>0$ and $h_0>0$ depending on $b$, $R$, and $K$ only such that, for all $\epsilon\in (0,1)$, we have with probability greater than $1-\epsilon$:
\begin{equation*}
\sup_{(t,x) \in {\Gamma_{b}}} | \hat S_{C,n}(t\mid x) -  S_{C}(t\mid x)|  \leq  M_1 \times
\left\{ \sqrt{\frac{|\log(  h^{d/2} \epsilon)|  }{nh^d}} + h^2\right\},
\end{equation*}
as soon as $h\leq h_0$ and $nh^d\geq M_2 |\log (h^{d/2} \epsilon)|  $.
\end{proposition}

The technical proof is given in the Appendix section (refer to the latter for a description of the constants $C_1$, $C_2$ and $h_0$ involved in the result stated above). A similar result, for the conditional survival function of $Y$ given $X$, is proved in \cite{Dabrowska89}, see Theorem 2.1 therein. Observe also that choosing $h=h_n\sim n^{-1/(d+4)}$ yields a rate bound of order $O_{\mathbb{P}}(\sqrt{\log(n)/n^{d/(d+4)}})$.
Finally, as previously mentioned, alternative (local averaging) methods could be used to compute estimators of $H_0(u,x)$, $H(u,x)$ and $g(x)$ and consequently estimators of $S_C(u\mid x)$ and $\Lambda_C(u\mid x)$, including \textit{$k$-nearest neighbours}, \textit{decision trees} or \textit{random forest}. Refer to section \ref{sec:exp} for further details.
 
\section{Generalization Bounds for Kaplan-Meier Risk Minimizers}\label{sec:main}

It is the purpose of this section to investigate the excess of risk \eqref{eq:risk_K} related to a domain $\mathcal{K}\subset \mathbb{R}_+\times \mathbb{R}^d$ of minimizers $\widetilde{f}_n(x)$ of the Kaplan-Meier risk \eqref{eq:KMemp_risk_K} over a class $\mathcal{F}$ of predictive functions that is of controlled complexity (see the technical assumptions below), while being rich enough to yield a small bias $R(f^*)-R(\bar{f}^*)$, denoting $R_{P,\mathcal{K}}(\cdot)$ by $R(\cdot)$ for simplicity throughout the present section.
We consider here the situation where, for all $i\in\{1,\; \ldots,\; n\}$, the estimate of the quantity $S_C(\tilde{Y}_i\mid X_i)$ plugged into \eqref{eq:emp_risk2} is obtained by evaluating the kernel smoothing estimator of $S_C(y\mid x)$ investigated in subsection \ref{subsec:prel} and based on the subsample $\{(X_j,\; \tilde{Y}_j,\; \delta_j):\; 1\leq j\leq n,\; j\neq i  \}$ at $(y,x)=(\tilde{Y}_i, X_i)$. The corresponding versions of the kernel estimators \eqref{eq:kern1}, \eqref{eq:kern2}, \eqref{eq:kern3} and those of \eqref{eq:CNA_est} and \eqref{eq:CKM_est} are respectively denoted by $\hat{H}^{(i)}_{0,n}(y\mid x)$, $\hat{H}^{(i)}_{n}(y\mid x)$, $\hat{g}^{(i)}_n(x)$, $\hat{\Lambda}^{(i)}_{C,n}(y\mid x)$ and $\hat{S}^{(i)}_{C,n}(y\mid x)$. This yields the \textit{leave-one-out} estimator of the risk of any candidate $f$
\begin{equation}\label{eq:KMemp_risk2}
\widetilde{R}_n(f)=\frac{1}{n}\sum_{i=1}^n\frac{\delta_i(\tilde{Y}_i-f(X_i))^2}{\hat{S}^{(i)}_{C,n}(\tilde{Y}_i - \mid X_i)}\mathbb{I}\{ (\tilde{Y}_i, X_i)\in \mathcal{K} \},
\end{equation}
that is well-defined on the event $\bigcap_{i=1}^n \{ \hat{S}^{(i)}_{C,n}(\tilde{Y}_i - \mid X_i)>0  \}$. As we clearly have
$$
R(\widetilde{f}_n)-\inf_{f\in \mathcal{F}}R(f)\leq 2 \sup_{f\in \mathcal{F}}\left\vert \widetilde{R}_n(f)-R(f) \right\vert,
$$
the key of the analysis is the control of the fluctuations of the process $\{ \widetilde{R}_n(f)-R(f)  :\; f\in \mathcal{F}  \}$.
Slightly more generally, we establish below a uniform deviation bound for processes of type
\begin{equation*}
  Z_n(\varphi) = \left( \frac{1}{n}\sum_{i=1}^n\frac{\delta_i  \varphi(\tilde Y_i, X_i)}{\hat{S}^{(i)}_{C,n} (\tilde{Y}_i- \mid X_i)}\right) - \mathbb{E} \left[ \varphi(Y , X ) \right],\;\; \varphi \in \Phi,
\end{equation*}
where the indexing class $\Phi$ fulfils the following property.
\begin{assumption}\label{hyp:control}
There exists a domain $\mathcal{K}\subset \Gamma_{b}$ such that $\varphi(y,x)=0$ as soon as $(y,x)\notin \mathcal{K}$ for all $\varphi \in \Phi$.
\end{assumption}
Equipped with these notations, observe that $\widetilde{R}_n(f)-R(f) =Z_n(\varphi)$ when $\varphi(Y,X)=(Y-f(X))^2\mathbb{I}\{ (\tilde{Y}, X)\in \mathcal{K} \}$.
\medskip

\noindent {\bf Linearization.} Whereas in the standard regression framework or in classification ERM can be straightforwardly studied by means of maximal deviation inequalities for empirical processes, the form of the process $\{Z_n(\varphi):\; \varphi \in \Phi \}$ of interest is very complex since the terms averaged in \eqref{eq:KMemp_risk} are obviously far from being independent due to the presence of the plugged leave-one-out estimators of the quantities $S_C(\tilde{Y}_i-\mid X_i)$.
Our approach to the study of the fluctuations of the process $Z_n$ consists in linearizing the statistic $Z_n(\varphi)$, \textit{i.e.} approximating $Z_n(\varphi)$ by a standard i.i.d. average in the $L_2$-sense, as stated in the next proposition. The theory of $U$-processes is used next to describe the uniform behaviour of the residual. Such concentration results are also used in \cite{CLV08} and \cite{PCB16} in simpler situations, where the residuals take the form of a degenerate $U$-statistic, see \cite{GinePena}.
In order to make this decomposition explicit, further notations are needed.  Define
\begin{eqnarray*}
H_{0,h}(y,x) &=&\mathbb{E}\left[ \hat{H}_{0,n}(y,x) \right],\\
 H_{h}(y,x)&=&\mathbb{E}\left[ \hat{H}_{n}(y,x) \right],
\end{eqnarray*}
as well as the conditional hazard function
\begin{equation*}
 {\Lambda}_{C,h}(u\mid x)  =  - \int_{s=0}^u\frac{ d{H}_{0,h}(s,  x)}{{H}_h(s-, x)},
\end{equation*}
and the related conditional survival function ${S}_{C,h}(t\mid x)$ and $c_h(s\mid x)  =  S_{C,h} (s-\mid x) /S_{C,h} (s\mid x)$.
 We also set
 \begin{align*}
 \hat  \Delta^{(i)}_{n} (u\mid x)&=\hat{\Lambda}^{(i)}_{C,n}(u\mid x)-\Lambda_{C,h}(u\mid x),\\
\hat a_{n}^{(i)}(t\mid x)& = - \int_0^t \frac{ c_{h} (u\mid x)   }{H_h(u,x) }    d ( \hat H_{0,n}^{(i)} (u, x)  - H_{0,h} (u, x)) \\
&\qquad  +  \int_0^t  \frac{  c_{h} (u\mid x)   }{   H_h(u,x)^2 } ( \hat H_n^{(i)}(u,x) -H_h(u,x) )   d   \hat H_{0,n}^{(i)} (u, x)  ,\\
\hat b_{n}^{(i)}(t\mid x)& =  -\int_0^t  \frac{  c_{h} (u\mid x)   }{  H_h(u,x)^{2}\hat H^{(i)}_n(u,x)} ( \hat H_n^{(i)}(u,x) -H_h(u,x) )^2   d   \hat H^{(i)}_{0,n} (u, x) \\
&\qquad  -  \int_0^t ( \hat S_{C,n}^{(i)} (u - \mid x) -  S_{C,h} (u - \mid x) ) d  \hat  \Delta^{(i)}_{n} (u\mid x),
\end{align*}
for all $i\in\{1,\; \ldots,\; n\}$. Equipped with these notations, we can now state the following result.

\begin{proposition}\label{prop:decomp}{\sc (KM risk decomposition)} Suppose that Assumptions \ref{hyp:cond_ind}, \ref{hyp:smooth}, \ref{hyp:bound1} and \ref{hyp:control} are fulfilled.
There exist constants $h_0>0$ and $M_1>0$ that depends on $b$, $R$ and $K$ only such that
\begin{itemize}
\item[(i)] $\forall (y,x)\in \mathcal{K}$, $S_{C,h}(y\mid x)\geq {b/2}$, $H_{h}(y, x)\geq {3 b^3/4}$, provided that $h\leq h_0$.
\item[(ii)] Moreover, for any $n\geq 2$ and $\epsilon\in (0,1)$, provided that $h\leq h_0$ and $nh^d\geq M_1 |\log(h^{d/2} \epsilon)|$, the event
$$\mathcal{E}_n\overset{def}{=}\bigcap_{i\leq n}\left\{\forall (t,x)\in \mathcal{K},\;\; \hat{S}^{(i)}_{C,n}(t,x)\geq b/2 \text{ and } \hat H^{(i)}_n(t,x)\geq  3 b^3 / 4 \right\}$$
 occurs with probability greater than $1-\epsilon$.
\item[(iii)]
For all $\varphi \in \Phi$ and $n\geq 2$, we have on the event $\mathcal{E}_n$:
$$
 Z_n(\varphi)=B_n(\varphi) + L_n(\varphi) + V_n(\varphi) + R_n(\varphi),
$$
where
\begin{align*}
&B_n(\varphi) =   \mathbb{E}  \left[ \delta  \frac{\varphi(\tilde Y, X)}{{S}_{C, h}(\tilde Y \mid X)} \right]  -  \mathbb{E}  \left[ \delta  \frac{\varphi(\tilde Y, X) }{{S}_{C}(\tilde Y \mid X)} \right], \\
& L_n(\varphi) = \frac{1}{n} \sum_{i=1}^n  \left( \delta_i  \frac{ \varphi(\tilde Y_i , X_i)}{ {S}_{C, h}(\tilde Y_i \mid X_i)} - \mathbb{E}  \left[ \delta \frac{ \varphi(\tilde Y_i , X_i)}{{S}_{C, h}(\tilde Y \mid X)} \right] \right), \\
& V_n(\varphi) =  - \frac{1}{n} \sum_{i=1}^n \delta_i  \varphi(\tilde Y_i , X_i) \frac{\hat  a_{n}^{(i)} (\tilde Y_i \mid X_i) }{{S}_{C, h} (\tilde Y_i \mid X_i)}, \\
&R_n (\varphi)=   \frac{1}{n} \sum_{i=1}^n  \frac{ \delta_i   \varphi(\tilde Y_i , X_i) }{ {S}_{C, h} (\tilde Y_i \mid X_i)} \left\{ - \hat b_{n}^{(i)} (\tilde Y_i \mid X_i)   +  \frac{\left( {S}_{C, h} (\tilde Y_i \mid X_i) - \hat{S}^{(i)}_{C,n} (\tilde Y_i \mid X_i) \right)^2}{{S}_{C, h} (\tilde Y_i  \mid X_i) \hat{S}_{C,n}^{(i)} (\tilde Y _i \mid X_i)} \right\}.
\end{align*}
\end{itemize}

\end{proposition}

The proof is given in the Appendix section.
Observe that the non-random quantity $B_n(\varphi)$ stands as a bias term in the decomposition. It vanishes at a rate depending on the smoothness assumptions stipulated.
The term $L_n(\varphi)$ is a basic centred i.i.d. sample mean statistic and its uniform rate of convergence $1/\sqrt n $ can be recovered by applying maximal deviation bounds for empirical processes under classic complexity assumptions such as those stipulated below, whereas the term $V_n(\varphi)$ is more complicated, since it involves multiple sums. It is dealt with by means of results pertaining to the theory of $U$-processes, by showing that it can be decomposed as $V_n (\varphi)= L_n'(\varphi) +R_n'(\varphi)$, the sum of a linear term and a second-order term.
The term $R_n(\varphi)+ R_n'(\varphi)$ is a remainder term (second order) and shall be proved to be negligible with respect to $L_n(\varphi)+L_n'(\varphi)$.

\begin{assumption}\label{hyp:complex} The set $\Phi$ of real-valued functions on $\mathbb R_+ \times \mathbb R^d$ forms a separable bounded class  of {\sc VC} type (w.r.t. the constant envelope $M_{\Phi}$), \textit{i.e.} there exist nonnegative constants $A$ and $v$ such that for all probability measures $Q$ on $\mathbb R_+ \times \mathbb R^d$ and any $\epsilon\in (0,1)$:
$
\mathcal{N}(\Phi, L_2(Q),\epsilon)\leq ( AM_{\Phi}/\epsilon )^{v}
$,
where $\mathcal{N}(\Phi, L_2(Q),\epsilon)$ denotes the smallest number of $L_2(Q)$-balls of radius less than $\epsilon$ required to cover class $\Phi$ (covering number), see \textit{e.g.} \cite{gine+g:2001}.
\end{assumption}
\begin{assumption}\label{hyp:bound2}
The densities $H_0(y\mid x)g(x)$ and $H(y\mid x)g(x)$ are both bounded by $R<+\infty$.
\end{assumption}
By means of these assumptions, the following result, proved in the Appendix section, describes the order of magnitude of the fluctuations of the process $Z_n$.

\begin{proposition}\label{prop:main_result}
Suppose that Assumptions \ref{hyp:cond_ind}-\ref{hyp:bound2} are fulfilled. There exist constants $h_0$, $M_1, M_2$ and $M_3$ that depend on $(A,v)$, $M_\Phi$, $R$, $K$ and $b$ only, such that, for all $n\geq 2$ and $\epsilon\in (0,1)$, the event
\begin{equation*}
\vert Z_n(\varphi)\vert\leq  M_1 \left( \sqrt{\frac{ \log\left(M_2 /\epsilon\right)}{n}} + \frac{ |  \log (\epsilon h^{d/2} ) | }{n h^d} + h^2 \right),
\end{equation*}
occurs with probability greater than $1-\epsilon$ provided that $h\leq h_0$, $nh^d\geq M_3|  \log(\epsilon h^{d/2}) |$.
\end{proposition}

The risk excess probability bound stated in the following theorem shows that, remarkably, minimizers of the Kaplan-Meier risk attain the same learning rate as that achieved by classic empirical risk minimizers in absence of censorship, when ignoring the model bias effect induced by the plug-in estimation step (\textit{cf} choice of the bandwidth $h$).

\begin{theorem}\label{thm:unif_control}
Suppose that Assumptions \ref{hyp:cond_ind}-\ref{hyp:bound2} are fulfilled. There exist constants $h_0$, $M_1, M_2$ and $M_3$ that depend on $(A,v)$, $M_\Phi$, $R$, $K$ and $b$ only, such that, for all $n\geq 2$ and $\epsilon\in (0,1)$, the event
\begin{equation*}
\lvert R (\tilde{f}_n) - R(f^\star) \rvert   \leq  M_1 \left( \sqrt{\frac{ \log\left(M_2 /\epsilon\right)}{n}} + \frac{ |  \log (\epsilon h^{d/2} ) | }{n h^d} + h^2 \right),
\end{equation*}
occurs with probability greater than $1-\epsilon$ provided that $h\leq h_0$, $nh^d\geq M_3|  \log(\epsilon h^{d/2}) |$.
\end{theorem}

The proof is a direct application of Proposition \ref{prop:main_result}. A similar bound for the expectation of the risk excess of minimizers of the empirical Kaplan-Meier risk can be classically derived with quite similar arguments, details are left to the reader.
 
\section{Numerical Experiments}\label{sec:exp}

Beyond the theoretical generalization guarantees established in the previous section, we now examine at length the predictive performance of the approach we propose for distribution-free regression based on censored training observations through various experiments based on synthetic/real data, and compare it to that of alternative methods documented in the survival analysis literature standing as natural competitors. As shall be seen below, the experimental results we obtained provide strong empirical evidence of the relevance of the Kaplan-Meier empirical risk minimization approach.
All the experiments and figures displayed in this article can be reproduced using the code available at \url{https://github.com/aussetg/ipcw}.

Before presenting and discussing the numerical results obtained, a few remarks are in order.
In the theoretical analysis carried out in the previous section, we placed ourselves on a restricted set $\Gamma_{b}$. However, in practice, we simply remove the last jump in \eqref{eq:CKM_est} and plug the estimator:
\begin{align*}
  \tilde{S}_{C, n} (y \mid x) = \prod_{\substack{\tilde{Y_i} \leq y \\ \tilde{Y_i} < \max_{\delta = 0} Y_i}} \left(1-d\hat{\Lambda}_{C,n}(\tilde{Y}_i \mid x) \right) , \, y \geq 0, x \in \mathbb{R}^d.
\end{align*}
Observe that, though $\tilde{S}_{C, n}$ is not a survival function anymore, it is still an accurate estimator. This alleviates possible difficulties caused by the frequent edge case where the last individual is observed ($\delta = 1$), since, in the case where \eqref{eq:CKM_est} is used, we have then $\delta_n/\hat{S}_C ( \tilde{Y}_n \mid X_n ) = \infty$.

\subsection{Experimental Results based on Synthetic Data}

In the synthetic experiments detailed below, we generated train and test data according to a simple Cox proportional hazard model \citep{cox_analysis_2018}:

\begin{equation*}
  S_Y(y \mid x) = \exp\left(-e^{\beta^T x} y\right) \text{ and } S_C(y \mid x) = \exp\left(-e^{\beta_C^T x} y\right),
\end{equation*}
with $t \leq 0, x \in [0, 1]$ where $X \sim \mathcal{U}([0, 1]^d)$.
This model is easy to generate, since $Y \mid X \sim \mathcal{E} (\exp \beta^T X)$ and $C \mid X \sim \mathcal{E} (\exp \beta_C^T X)$.
So that the censoring is informative, we use
\hspace{1em}
\begin{align*}
      &\beta^T = \begin{bmatrix} \bovermat{$\lceil n/2 \rceil$}{1 & \cdots & 1} & 0 & \cdots & 0 \end{bmatrix} \\
      &\beta_c^T = \lambda \begin{bmatrix} 1 & 0 & 1 & 0 & 1 & \cdots \end{bmatrix}
\end{align*}
where the tuning parameter $\lambda$ controls the level of censorship $1 - p$ with $p = \mathbb{E} [\delta]$. We chose the appropriate $\lambda$ for the desired $p$ by Monte-Carlo simulations.
On the training set we only observe $\tilde Y_ i = (Y_i\wedge C_i)$, while we observe the true $Y$ on the test set in order to measure the performance without any special consideration for censorship.
We consider several approaches to build a function that nearly achieves the same predictive performance as  $f^\star(x)=\mathbb{E}[Y\mid X=x]$, consisting respectively in minimizing the (IPCW/weighted) empirical risks
\begin{equation}\label{losses}
\begin{aligned}
  \text{IPCW} \quad &\frac{1}{n} \sum_{i=1}^n \delta_ i \frac{(\tilde Y_i - f(X_i))^2}{\hat S_{C} (\tilde{Y}_i|X_i)} \qquad &\text{IPCW LoO} \quad &\frac{1}{n} \sum_{i=1}^n \delta_ i  \frac{(\tilde Y_i - f(X_i))^2}{\hat S_{C}^{(i)} (\tilde{Y}_i|X_i)} \\
  \text{IPCW Forest} \quad &\frac{1}{n} \sum_{i=1}^n \delta_ i \frac{(\tilde Y_i - f(X_i))^2}{\hat S_{C}^{\text{RF}}(\tilde{Y}_i|X_i)} \qquad &\text{IPCW Stute} \quad &\frac{1}{n} \sum_{i=1}^n \delta_ i \frac{(\tilde Y_i - f(X_i))^2}{\hat S_{C} (\tilde{Y}_i)} \\
  \text{IPCW KNN}  \quad &\frac{1}{n} \sum_{i=1}^n \delta_ i \frac{(\tilde Y_i - f(X_i))^2}{\hat S_{C}^{\text{KNN}}(\tilde{Y}_i|X_i)} \qquad & \text{IPCW Oracle} \quad &\frac{1}{n} \sum_{i=1}^n \delta_ i \frac{(\tilde Y_i - f(X_i))^2}{S_{C} (\tilde{Y}_i|X_i)} \\
  \text{Naive} \quad &\frac{1}{n} \sum_{i=1}^n (\tilde Y_i - f(X_i))^2 \qquad &\text{Observed} \quad &\frac{1}{n} \sum_{i=1}^n \delta_i (\tilde Y_i - f(X_i))^2\\
  \text{Oracle} \quad &\frac{1}{n} \sum_{i=1}^n (Y_i - f(X_i))^2 & &
\end{aligned}\nonumber
\end{equation}
where $\hat S_{C}^{(i)}$ is the leave-one-out version of $\hat S_{C}$, i.e. the same estimate but dropping-out the $i$-th observation, $\hat S_{C}^{\text{RF}}$ is estimated using random forests \citep{ishwaran_random_2008} and $\hat S_{C}^{\text{KNN}}$ uses a leave-one-out nearest neighbours approach instead of kernels. Observe incidentally that selection of the related hyperparameters is tricky, insofar as the estimator is itself involved in the definition of the objective risk function.
We use the notation $\hat S_{C} (\cdot)$ for the standard non-conditional Kaplan-Meier estimate of the survival function which coincides with the case of non-informative censorship found in \citet{Stute1995}. The last two risk functionals are oracle estimators and serve as a benchmark to quantify the negative impact of the plug-in estimation.
The various approaches are compared through the accuracy regarding the prediction and estimation tasks. 

\subsubsection{Prediction error}

We study the prediction risk $\mathds{E} [ ( Y - f(X) )^2 ]$ for several classes of functions $\mathcal{F}$ where $\mathcal{F}$ is either a RKHS (SVR), a collection of orthogonal piecewise constant functions (Breiman Random Forests) or a space of linear functions (Linear Regression). We also set the level of censorship $p$ to $1/4$,  $1/2$ or $3/4$. All synthetic results are presented for $d = 4$ but other values of $d$ are presented in Appendix \ref{appendix:results}. The prediction error is estimated by Monte Carlo when running the experiments $50$ times: each time a train set is generated and an estimator of $f^\star$ is learnt by minimizing one of the losses in \ref{losses}. The error is then measured on a completely observed test set (generated in the same way as the training set) of size $5000$ $\mathcal{D}_T$. The test error is then $\sum_{(Y_i, X_i) \in \mathcal{D}_T} \lvert Y_i - f^\star (X_i) \rvert^2$.
Regarding the choice of hyperparameters of $S_C$ we use for the kernelized Kaplan-Meier estimator $h = 5 \sigma  n^{-1 / (d + 2)}$ which follows (up to a constant) from Proposition \ref{prop:main_result}. For the KNN estimator, we use $K = 5$ and for the random survival forest version, we keep the default hyperparameters of the \texttt{randomForestSRC} \citep{randomForestSRC} package.

\begin{figure}[htbp]
  \centering
  \begin{subfigure}[b]{1.\textwidth}
        \includegraphics[width=\textwidth]{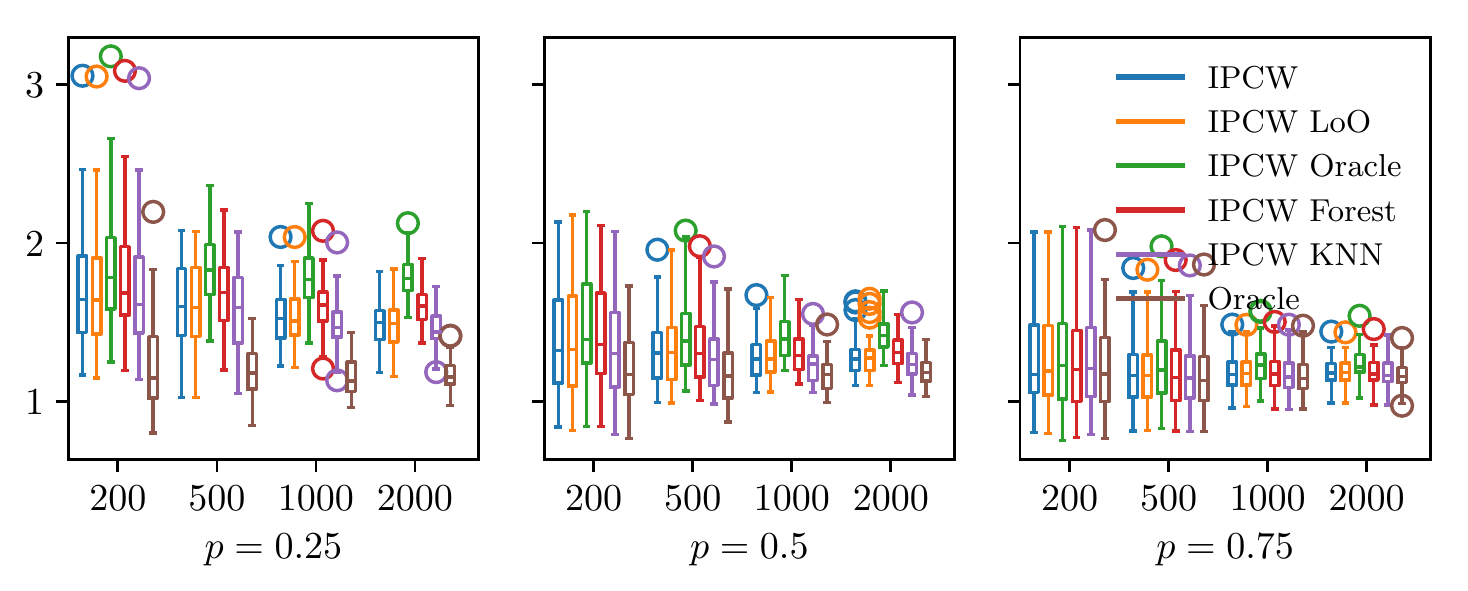}
        \caption{Linear Regression}
  \end{subfigure}
  \begin{subfigure}[b]{1.\textwidth}
        \includegraphics[width=\textwidth]{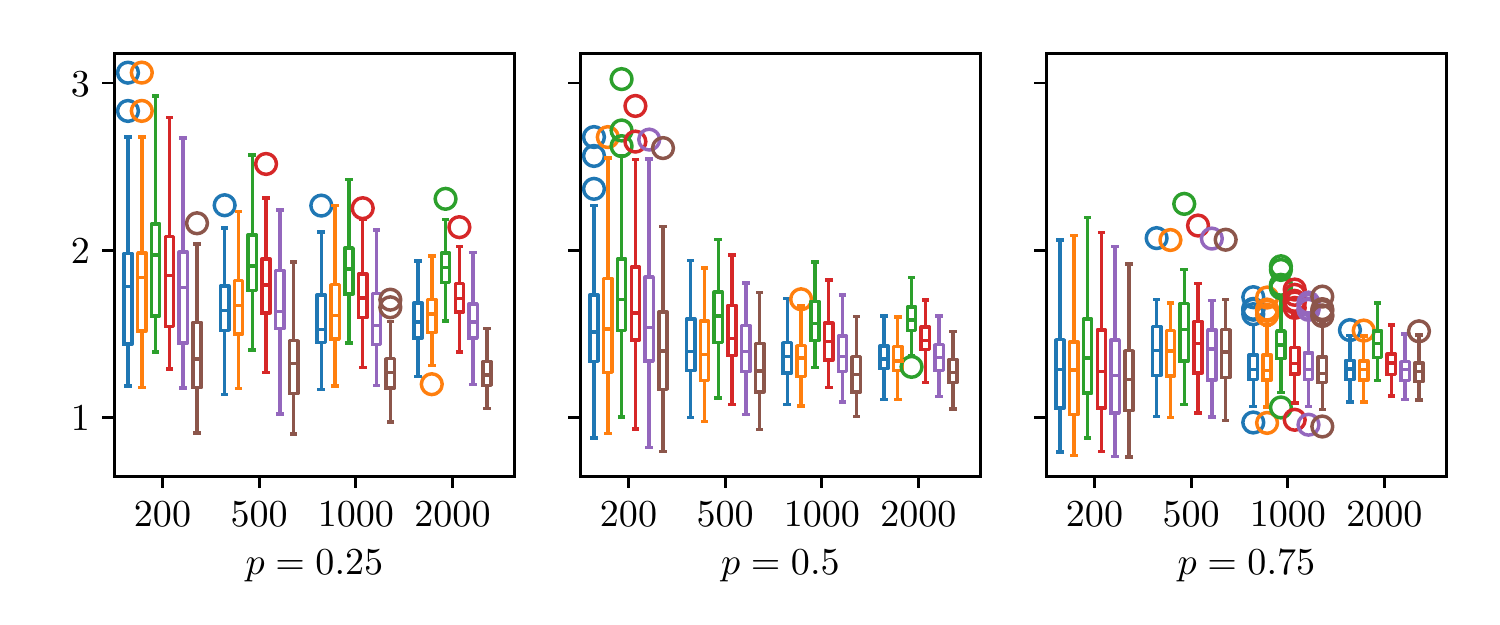}
        \caption{SVR}
  \end{subfigure}
  \begin{subfigure}[b]{1.\textwidth}
        \includegraphics[width=\textwidth]{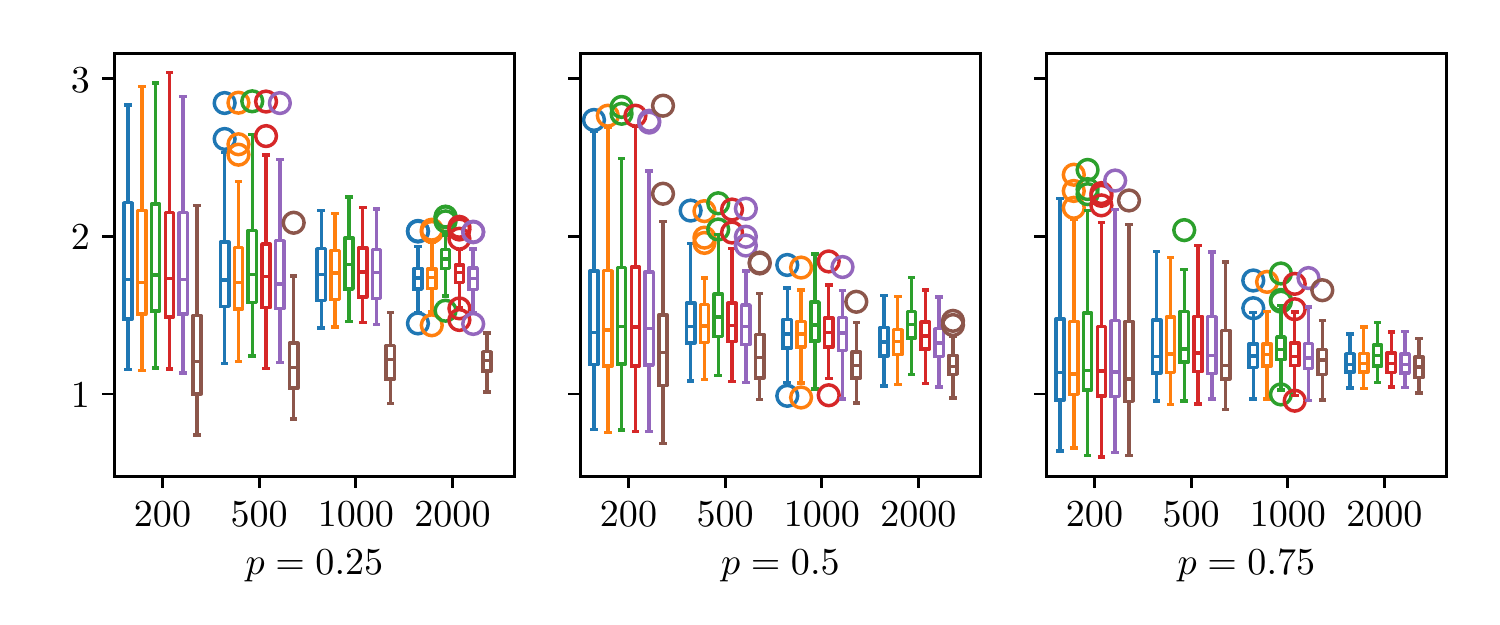}
        \caption{Random Forest}
  \end{subfigure}
  \caption{$L^2$ error of the different IPCW estimators}
  \label{fig:error}
\end{figure}

As shown in Figure \ref{fig:error}, the IPCW KNN estimator systematically outperforms the other estimators in our experiments, no matter the level of censorship $p$ or class $\mathcal{F}$, as such any further mention of IPCW implicitly refers to the IPCW KNN version without further notice.
In order to show the improvements brought by the IPCW reweighting, we compare the IPCW estimator to various naive approaches to the problem: one could either decide to fit an estimator directly from the censored values without any corrections or else discard the censored observations and next fit the estimator based on the uncensored values. These two approaches, corresponding to the Naive and Observed losses in \ref{losses}, are biased as they respectively estimate $\mathbb{E} [ (\tilde{Y} - f(X) )^2 )$ and $\mathbb{E} [ (Y - f(X) )^2 \mid \delta = 1]$. The former method can still be of interest in certain edge cases: when there are too few non-censored observations compared to the total number of available observations (i.e. $p n \ll n$) then the biased version may yield a better predictor simply because of the disparity of available effective data. Results are presented in Figure \ref{fig:error_vs}.

\begin{figure}[htbp]
  \centering
  \begin{subfigure}[b]{1.\textwidth}
        \includegraphics[width=\textwidth]{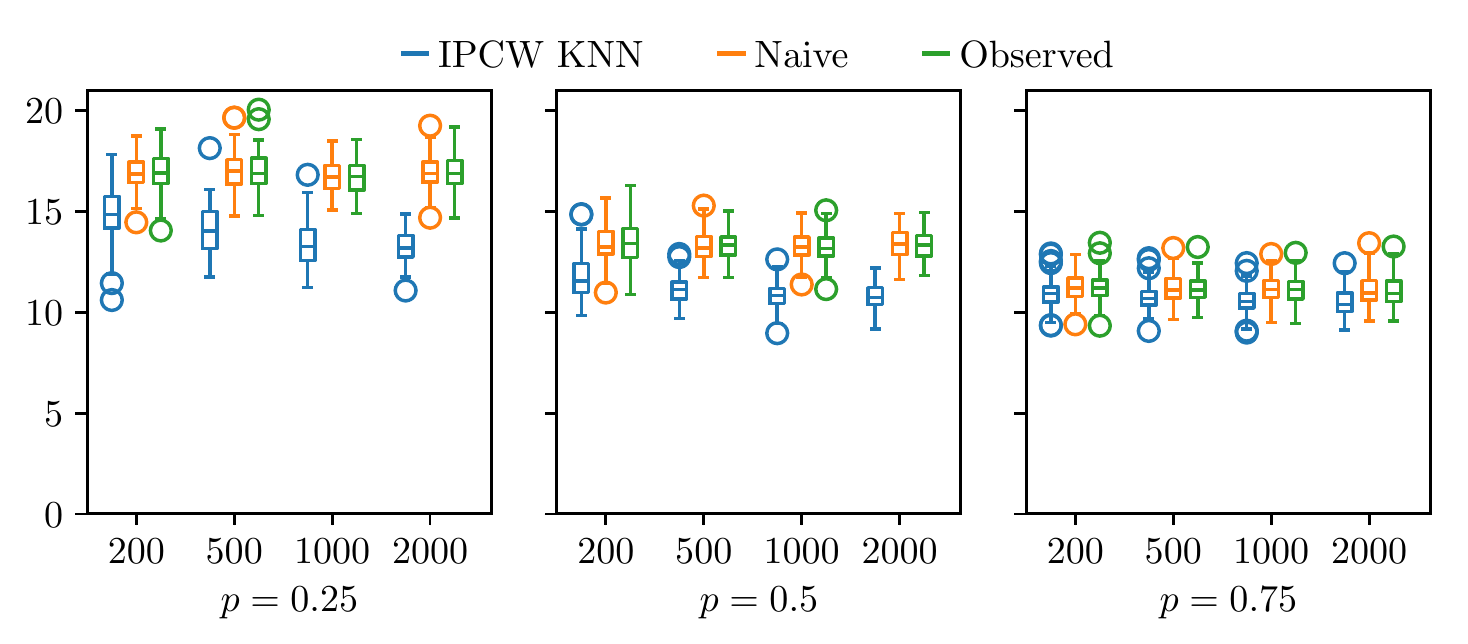}
        \caption{Linear Regression}
  \end{subfigure}
  \begin{subfigure}[b]{1.\textwidth}
        \includegraphics[width=\textwidth]{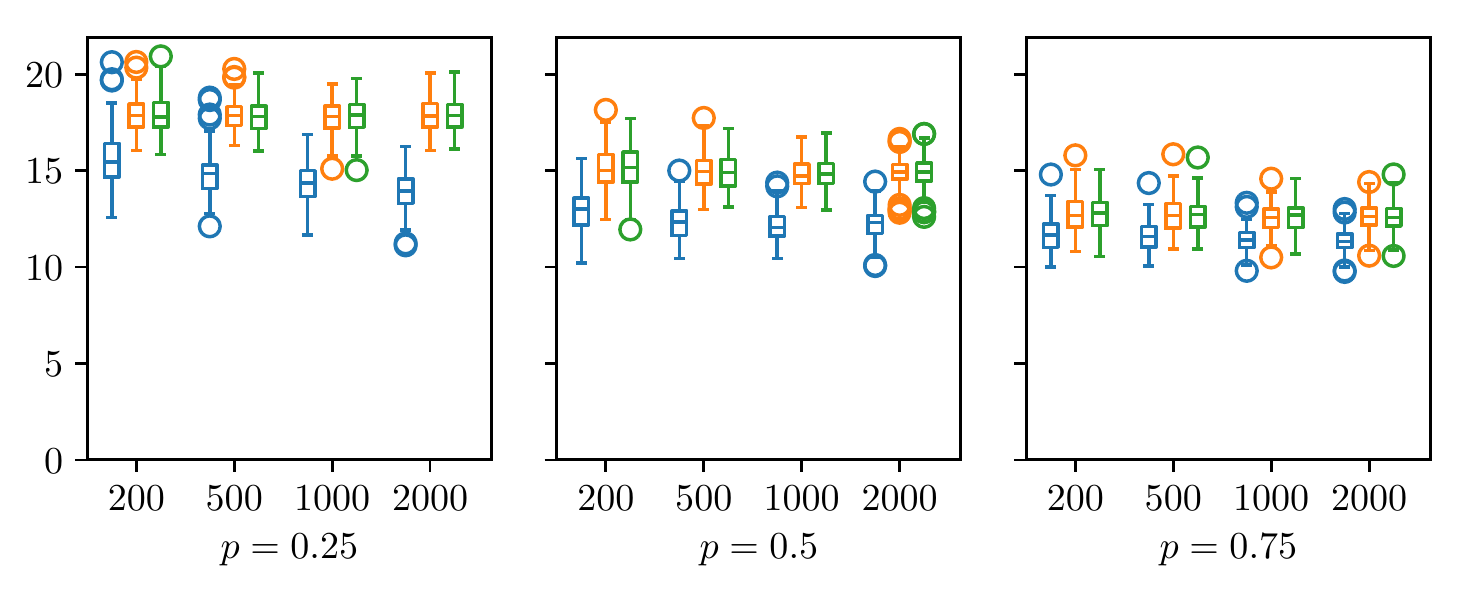}
        \caption{SVR}
  \end{subfigure}
  \begin{subfigure}[b]{1.\textwidth}
        \includegraphics[width=\textwidth]{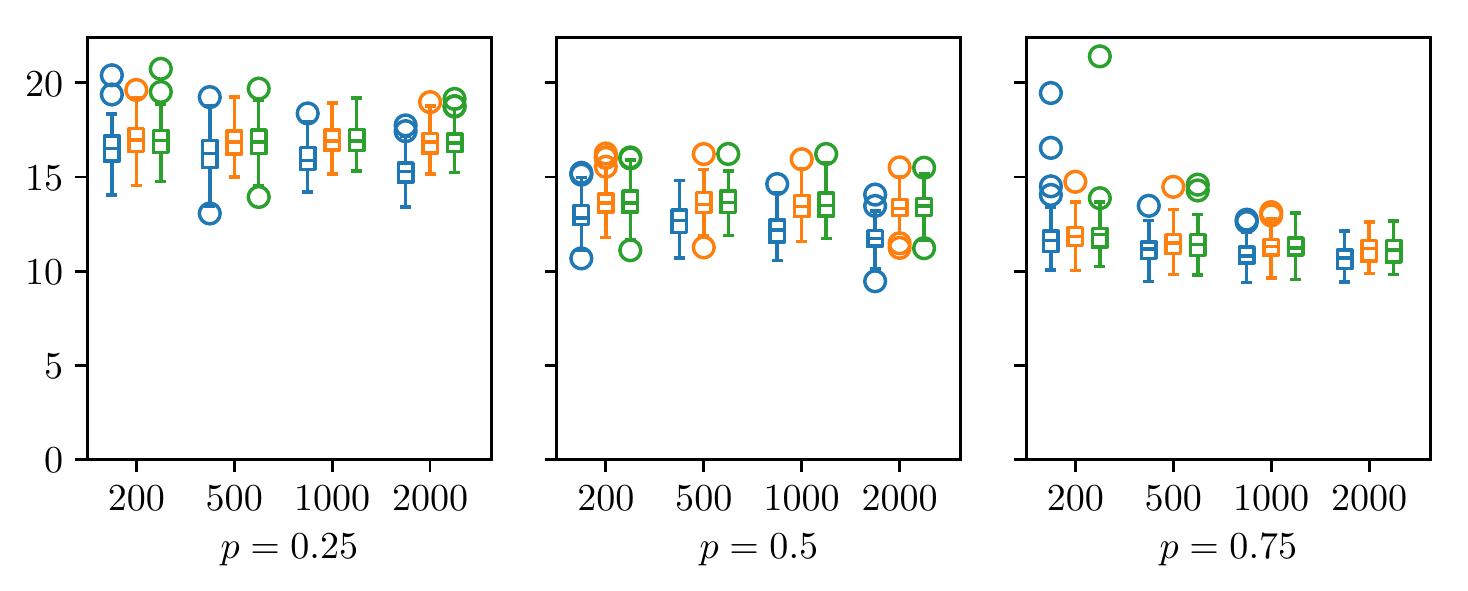}
        \caption{Random Forest}
  \end{subfigure}
  \caption{Prediction error $\mathbb{E}[( Y - \tilde{f}_n (X) )^2 ]$ of the IPCW estimators compared to the naive methods}
  \label{fig:error_vs}
\end{figure}

Learning $f$ on the corrected IPCW loss always outperforms the naive alternatives, with the differences in predictive performance becoming more pronounced as the censorship level $1 - p$ increases. Unsurprisingly, when most of the points are observed ($p \to 1$) all methods reach roughly the same error as all the losses in \ref{losses} are equal for $p = 1$.
We empirically observe that the IPCW problem with oracle weights (i.e. $\delta / S_C$ ) can yield worse estimators than the plug-in version (i.e. $\delta / \hat{S}_C$) and exhibits a much higher variance. Intuitively, this phenomenon can be explained by the fact that the active weights $1 / \hat{S}_C$ have a low variance while their oracle version $1 / S_C$ has a higher variance and can grow arbitrarily large for observations in the tail. Therefore it is advisable, and one can empirically verify this easily, to choose an estimator of $S_C$ with a low variance. Even the limit case of the non-conditional Kaplan-Meier estimate (corresponding to $h \to \infty$ in our estimator) offers reasonable performances, still in the case of informative censorship.
Finally, we compare popular machine learning methods reweighted by the IPCW technique to standard state-of-the-art procedures.
These include standard statistical methods based on the estimation of the survival already mentioned in Section \ref{sec:intro} or in \citet{van_der_laan_unified_2003} where the estimated survival can then be used to compute the downstream quantity of interest provided it can be written as an integral w.r.t. the survival function, for example the conditional mean $\int S(\mathrm{d} y \mid X = x)$. The other family of methods is more rooted in the machine learning methodology and designs losses specifically adapted to the censored regression problem, either through transformation models in \citet{van_belle_learning_2011}, or by adapting the SVM methodology as done in \citet{van_belle_support_2007, polsterl_fast_2015, polsterl_efficient_2016}.
We also include the method of \citet{hothorn_survival_2006} that follows the same methodology as this paper and uses a boosting technique to optimize a loss reweighted by (unconditional) Kaplan-Meier weights as well as the method of \citet{ishwaran_random_2008} that builds a recursive splitting of the feature space $\mathcal{X}$ by maximizing a notion of inter-cluster dissimilarity of the survival functions, the final clusters are then used for downstream tasks (classification, regression, quantile estimation). We compare $10$ estimators from the survival literature compared to $5$ standard learners with IPCW weights.
The standard machine learning techniques reweighted by IPCW, the methodology promoted in this paper, have been implemented by means of the software \citet{scikit-learn} coupled with our own implementation of our proposed LoO IPCW estimator, while the specific survival machine learning methods have been implemented using \texttt{scikit-surv}. Finally, we use the original Random Survival Forest of \citet{randomForestSRC}. The default values for the hyperparameters are used in every case.
All experiments are based on $200$ observations only, insofar as some of the SVM based techniques are impractical with large $n$.
Results for all methods can be found in Table \ref{table:soa_vs_ipcw}.

\begin{table}[h]
\centering
\begin{tabular}{@{}lrrr@{}} \toprule
& \multicolumn{3}{c}{$\sqrt{L^2 \text{ Error}}$} \\
\cmidrule(r){2-4} \\
Method & $\delta = 0.25$ & $\delta = 0.5$ & $\delta = 0.75$ \\
\midrule \\
Survival Gradient Boosting & $3.19$ & $3.55$ & $3.61$ \\
Component-wise Survival Gradient Boosting & $3.19$  & $3.87$ & $4.23$ \\
Cox Proportional Hazards & $7.86$ & $7.61$ & $7.03$ \\
Coxnet & $7.62$ & $7.39$ & $6.85$ \\
Kernel Survival SVM & $4.02$ & $3.92$ & $4.13$ \\
Survival SVM & $4.04$ & $4.09$ & $3.94$ \\
Hinge Loss Survival SVM & $8.10$ & $8.28$ & $8.09$ \\
Minlip Survival SVM & $3.27$  & $3.96$ & $4.22$ \\
Random Survival Forest & $2.01$  & $2.94$ & $2.78$ \\
\midrule \\
Ridge + IPCW & $\textbf{1.75}$ & $\textbf{1.49}$ & $\textbf{1.24}$ \\
Kernel Ridge + IPCW & $2.07$ & $1.60$ & $1.35$ \\
Linear Regression + IPCW & $1.81$ & $\textbf{1.49}$ & $\textbf{1.24}$ \\
Random Forest + IPCW & $1.85$ & $1.57$ & $1.36$ \\
SVR + IPCW & $1.87$ & $1.66$ & $1.42$ \\
\bottomrule
\end{tabular}
\caption{}
\label{table:soa_vs_ipcw}
\end{table}

\subsubsection{Error estimation}

While not the focus of our method, it is of interest to study the quality of the approximation of the risk $R(f) = \mathbb{E} \left[ \mathcal{L} (Y, f(X) ) \right]$.
We modify here slightly the estimator $\tilde{R}_n (f)$ presented in \ref{eq:KMemp_risk} by normalizing the weights; while not necessary for the regression problem, this ensures that our estimator represents an integral w.r.t. a proper measure.

To make things easier we can study risks of the form $R(\varphi) = \mathbb{E} \left[ Y \varphi(X) ) \right]$, by choosing $\varphi(X) = e^{- X^T \beta}$ we have $R(f) = 1$.
We sample the error $\lvert R(\varphi) - \tilde{R}_{n, \mathcal{D}} (\varphi)  \rvert$ for $N = 100$ random training sets $\mathcal{D} = \{ (X_i, \tilde{Y}_i, \delta_i) \}$.

\begin{figure}[htbp]
  \centering
  \includegraphics[width=1.\textwidth]{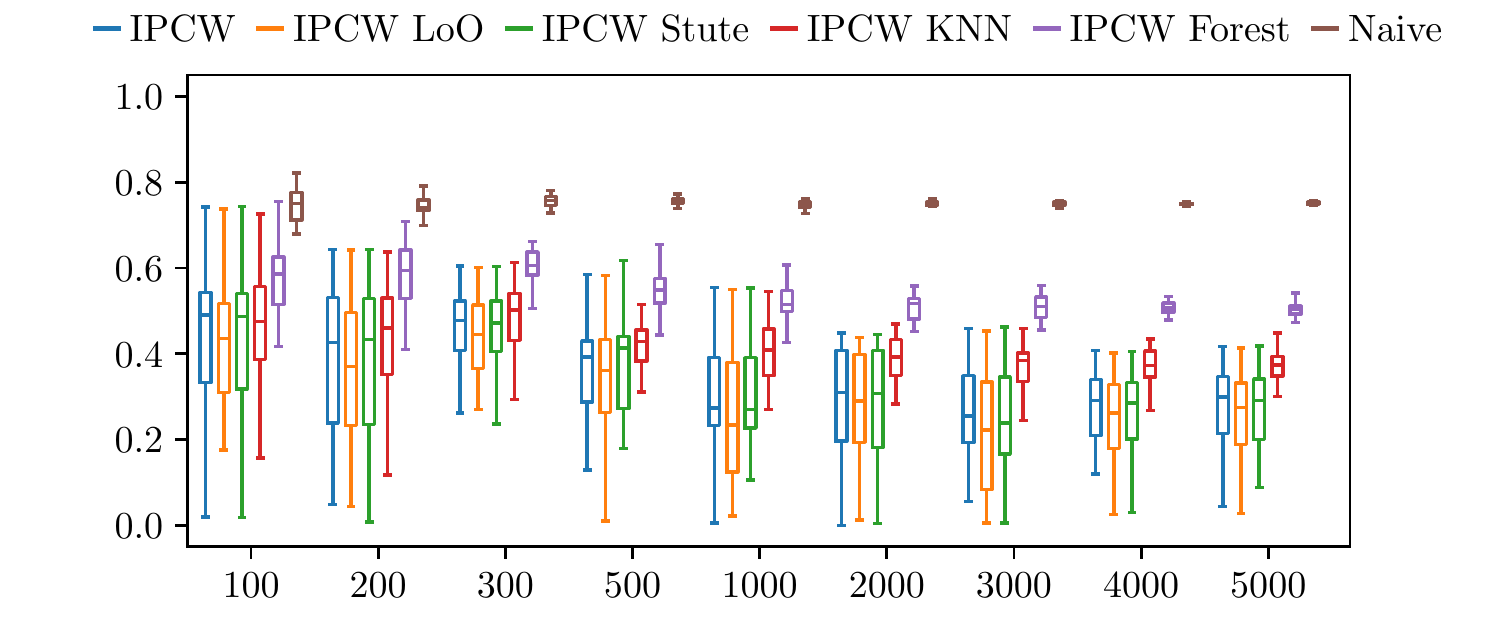}
  \caption{Estimated $L^2$ error of the IPCW estimator compared to naive methods for $p = \sfrac{1}{4}$}
  \label{fig:loss}
\end{figure}

As can be seen in Figure \ref{fig:loss}, while both naive methods offer poor approximations of the loss (as expected since they are biased), the IPCW reweighting methods converge towards the correct value of $R(f)$ at the expected rate.
We observe here that the best estimator of the error is based on the IPCW LoO reweighting while the best prediction error is achieved by IPCW KNN as shown by Figure \ref{fig:error}. We conjecture that low variance estimators of $S_C$ achieve better results for the prediction task. Our different experiments seem to empirically validate this hypothesis as high $K$ KNN estimators and large $h$ kernel estimators showed the best regression performances.

\subsection{Real Data}

The performance of the IPCW approach is now investigated on the TCGA Cancer data \citep{grossman_toward_2016} using solely the RNA transcriptomes as informative variables. All models are trained on $8080$ patients with a censorship rate of $18\%$, we measure on the remaining $1449$ observed patients the error as well as the concordance index $\frac{1}{\lvert \delta \rvert}\sum_{\delta_i = 1} \sum_{\tilde{Y}_j > \tilde{Y}_i} \mathds{1}_{f(X_i) < f(X_j)}$ and only use IPCW KNN without any tuning of $K$ ($K = 5$).

\begin{table}[h]
      \centering
      \begin{tabular}{@{}lrrrrrr@{}} \toprule
      & \multicolumn{2}{c}{IPCW} & \multicolumn{2}{c}{Naive} & \multicolumn{2}{c}{Observed} \\
      \cmidrule(r){2-3} \cmidrule(r){4-5} \cmidrule(r){6-7} \\
      Method & $\sqrt{L^2 \text{ Error}}$ (\emph{years}) & Concordance & $\sqrt{ \text{Error}}$ & C & $\sqrt{\text{Error}}$ & C \\
      \midrule \\
      Cox\footnotemark & \multicolumn{2}{r}{$76.4$} & \multicolumn{2}{r}{$0.6071$} \\
      \midrule \\
      SVR & $\mathbf{2.768}$ & $0.563$ & $2.796$ & $\mathbf{0.575}$ & $2.795$ & $0.543$ \\
      Linear Regression & $\mathbf{3.193}$ & $\mathbf{0.594}$ & $4.971$ & $0.557$ & $3.898$ & $0.508$ \\
      Ridge & $\mathbf{3.193}$ & $0.594$ & $4.962$ & $0.5573$ & $3.896$ & $0.5077$ \\
      Kernel Ridge & $\mathbf{2.683}$ & $\mathbf{0.597}$ & $2.704$ & $0.592$ & $2.956$ & $0.513$ \\
      Random Forest & $\mathbf{2.577}$ & $\mathbf{0.630}$ & $2.636$ & $0.603$ & $2.878$ & $0.542$ \\
      \bottomrule
      \end{tabular}
      \label{table:tcga}
\end{table}
 
\section{Conclusion}\label{sec:conclusion}

In the present article, we have presented both theoretical and experimental work on statistical learning based on censored data. Precisely, we considered the problem of learning a predictive/regression function when the output variables related to the training observations are subject to random right censorship under mild assumptions. Following in the footsteps of the approach introduced in \cite{Stute1995}, we studied from a nonasymptotic perspective the performance of predictive functions built by minimizing a weighted version of the empirical (quadratic) risk, constructed by means of the Kaplan-Meier methodology. Learning rate bounds describing the generalization ability of such predictive rules have been proved, through the study of the fluctuations of the Kaplan-Meier risk functional, relying on linearization techniques combined with concentration results for $U$-processes. These theoretical results have also been confirmed by various numerical experiments, supporting the approach promoted. A difficult question, that will be the subject of further research, is the design of model selection methods (structural risk minimization) to pick automatically the optimal hyperparameters for the plugged estimator $\hat{S}_{C,n}$. Indeed, this is far from straightforward, insofar as changing the hyperparameters or the model modifies the loss that is being optimized, which makes standard methods such as cross-validation unsuitable.
 
\appendix
\section{Auxiliary Lemmas}\label{AppendixA}
For completeness, classic approximation and concentration results, extensively used in the subsequent proofs, are recalled.

\noindent {\bf Kernel approximation.} We first recall the following classical approximation bound, see \textit{e.g.} Proposition 1.2 in \cite{Tsybakov09}. Define $\mathcal{H}_{\beta, L}(\Omega)$ as the space of functions $g$ in $\mathcal{C}^{\lfloor\beta\rfloor}(\Omega)$ with all derivatives up to order $\lfloor\beta\rfloor$ bounded by $L$ and such that, for any multi-index $\alpha\in \mathbb{N}^d$ with $\vert \alpha\vert\leq \lfloor\beta\rfloor$:
$$\forall (x,y)\in \Omega^2,\;\; \vert \partial_{\alpha} f(x)-\partial_{\alpha}f(y)\vert \leq L \vert\vert x-y\vert\vert^{\beta-\vert \alpha\vert},$$
denoting by $\vert\vert.\vert\vert$ the usual Euclidean norm on $\mathbb{R}^d$.

\begin{lemma}\label{lem:approx} Let $\beta>0$, $L>0$ and $\Omega$ an open convex subset of $\mathbb{R}^d$.
Suppose that $f$ belongs to the H\"older class $\mathcal{H}_{\beta, L}(\Omega)$, then, if the kernel $K$ is of order $\lfloor \beta \rfloor$,  we have: for all $h>0$,
\begin{equation}\label{eq:convol_bound}
\sup_{x\in \Omega}\left\vert  (K_h\ast f)(x)-f(x)\right\vert \leq C h^{\beta},
\end{equation}
where $ C=\frac{L}{\lfloor\beta\rfloor!} \sum_{\alpha\in \mathbb{N}^d:\; \vert \alpha \vert = \lfloor\beta\rfloor} \int_{z\in \mathbb{R}^d}\vert K(z)\vert \prod_{i=1}^d \vert z_i\vert^{\alpha_i} dz$.
\end{lemma}

\noindent {\bf Concentration of empirical processes.}
We recall the following useful concentration inequality for empirical processes over VC classes. It is stated in \cite{einmahl+m:2000,gine+g:2001} under various forms and the following version is taken from \cite{gine+s:2010}.

\begin{lemma}\label{lem:concentration_ineq_empi_sum}
Let $\xi_1,\xi_2,\ldots$ be i.i.d. r.v.'s valued in a measurable space $(S,\mathcal S)$ and $\mathcal U$ be a class of functions on $S$, uniformly bounded and of VC-type with constant $(v,A)$ and envelope $U:S\rightarrow \mathbb R$. Set $\sigma^2(u)=var(u(\xi_1))$ for all $u\in \mathcal U$.
There exist constants $C_1>0,\,C_2\geq 1,\,C_3>0$ (depending on $v$ and $A$) and $\sup_{u\in \mathcal U} |\sigma^2 (u) |\leq  \sigma^2\leq \|U\|_\infty^2$, such that $\forall t>0$ satisfying
\begin{align}\label{eq:range0}
C_1   \sigma \sqrt { n \log\left( \frac{2\|U\|_\infty}{\sigma}\right) } \leq t\leq \frac{n\sigma^2}{\|U\|_\infty},
\end{align}
then
\[
\mathbb{P}\left\{ \left\|  \sum_{i=1}^n \{u(\xi_i ) - \mathbb E [ u(\xi_i) ]\}  \right\|_{\mathcal U}  >t  \right\}  \leq C_2 \exp\left(-C_3 \frac{t^2}{n\sigma^2}\right).
\]
\end{lemma}

The previous result is extended to the case of degenerated $U$-processes over VC classes \citep[Theorem 2]{major:2006}.

\begin{lemma}\label{lem:concentration_ineq_Ustat_cor}
Let $\xi_1,\xi_2,\ldots$ be an i.i.d. sequence of random variables taking their values in a measurable space $(S,\mathcal S)$ and distributed according to a probability measure $P$. Let $\mathcal H$ be a class of functions on $S^k$ uniformly bounded such that $\mathcal H$ is of {\sc VC} type with constants $(v,A)$ and envelope $G$. For any $H\in \mathcal{H}$, set $\sigma^2(H)=var(H(\xi_1,\; \ldots,\; \xi_k))$ and assume that
\begin{equation}\label{eq:degeneracy}
 \forall j\in\{1,\; \ldots,\;   k\},\;\;
\mathbb{E}[H(\xi_1,\; \ldots,\; \xi_k)\mid \xi_1,\;\ldots,\, \xi_{j-1},\; \xi_{j+1},\; \ldots,\; \xi_k  ]=0 \text{ with probability one.}
\end{equation}
Then, there exist constants $C_1>0,\,C_2\geq 1,\,C_3>0$ (depending on $v$ and $A$) and $\sup_{g\in \mathcal G}\sigma^2(g)\leq  \sigma^2\leq \|G\|_\infty^2$, such that for all $t>0$ satisfying
\begin{equation}\label{eq:range}
C_1 \sigma \left(n \log\left( \frac{2\lVert G \rVert_\infty}{\sigma}\right) \right)^{k/2}\leq t \leq \sigma \left( \frac{n \sigma}{\lVert G \rVert_\infty} \right)^k,
\end{equation}
then
\[
\mathbb{P}\left\{ \left\|  \sum_{(i_1,\; \ldots,\;, i_k)} H(\xi_{i_1},\; \ldots,\; \xi_{i_k})   \right\|_{\mathcal H} > t  \right\}  \leq C_2 \exp\left(-C_3 \frac{1}{n} \left(\frac{t}{\sigma} \right)^{2/k}\right).
\]
where \[ \lVert G \rVert^2_\infty \geq \sigma^2 \geq \lVert \operatorname{Var} (H) \rVert^2_{\mathcal H} \]
\end{lemma}

The following result is directly derived from that stated above by specifying an appropriate value of $t $.

\begin{corollary}\label{lem:concentration_ineq_Ustat}
Let $\xi_1,\xi_2,\ldots$ be an i.i.d. sequence of random variables taking their values in a measurable space $(S,\mathcal S)$ and distributed according to a probability measure $P$. Let $\mathcal H$ be a class of functions on $S^k$ uniformly bounded such that $\mathcal H$ is of {\sc VC} type with constants $(v,A)$ and envelope $G$. For any $H\in \mathcal{H}$, set $\sigma^2(H)=var(H(\xi_1,\; \ldots,\; \xi_k))$ and assume that
\begin{align*}
 \forall j\in\{1,\; \ldots,\;   k\},\;\;
\mathbb{E}[H(\xi_1,\; \ldots,\; \xi_k)\mid \xi_1, \; \ldots, \, \xi_{j-1},\; \xi_{j+1},\; \ldots,\; \xi_k  ]=0 \text{ with probability one.}
\end{align*}
Then, there exist constants $C_1>0\,,C_2\geq 1\,, C_3>0$ (depending on $v$ and $A$) such that
\[
\mathbb{P}\left\{ \left\|  \sum_{(i_1,\; \ldots,\;, i_k)} H(\xi_{i_1},\; \ldots,\; \xi_{i_k})   \right\|_{\mathcal H} \leq t(n,\sigma, \epsilon) \right\}  > 1- \epsilon.
\]
with
\begin{align*}
t(n,\sigma, \epsilon) = \sigma n^{k/2} \left( C_1  \left( \log\left( \frac{2\lVert G \rVert_\infty}{\sigma}\right) \right)^{k/2}  +   \left( \frac{\log(C_2/\epsilon )}{C_3}\right)^{k/2}  \right)
\end{align*}
provided that
\begin{gather*}
\lVert G \rVert_\infty^2  \left( C_1^{2 / k}  \log\left( \frac{2 \lVert G \rVert_\infty}{\sigma}\right) +  \frac{\log(C_2/\epsilon )}{C_3} \right) \leq n \sigma^2 \\
\sup_{H\in \mathcal H}\sigma^2(H)\leq  \sigma^2\leq \|G\|_\infty^2.
\end{gather*}

\end{corollary}

\noindent {\bf {\sc VC} type classes of functions - Permanence properties.} In the subsequent sections, many results are obtained by applying the concentration bounds recalled above to specific classes of functions/kernels built up from the elements of the class $\Phi$ and other functions such as $K_h(x)$, $S_C(u\mid x)$ or $g(x)$. The following lemmas exhibit situations where the {\sc VC} type property is preserved, while controlling the constants $(v,A)$ involved. In what follows the kernel $K$ is assumed to satisfy the hypotheses introduced in section \ref{subsec:prel}.

\begin{lemma}(see \cite{nolan+p:1987}, Lemma~22, Assertion~(ii))
The class $\{z\mapsto K(h^{-1}({x-z})) \ :\ x\in \mathbb{R}^d,\ h>0\}$ is a bounded {\sc VC} class of functions.
\end{lemma}
The following result is established in \cite{portier+s:2018} (see Proposition 8 therein). Its proof is recalled below for clarity's sake.

\begin{lemma} Let $(V,W)$ be a pair of random variables taking their values in $\mathbb{R}^q$ and in $\mathbb{R}^d$ respectively, denote by $f_0(v\mid W)$ the density of the conditional distribution of the r.v. $V$ given $W$, supposed to be absolutely continuous w.r.t. Lebesgue measure on $\mathbb{R}^q$.
The class $\{w\in\mathbb{R}^d \mapsto \mathbb{E} [K(h^{-1}({z-V}))\mid W=w ]\ :\ z\in \mathbb{R}^q,\ h>0\}$ is a bounded {\sc VC} class of functions (with constants depending on $K$).
\end{lemma}

\begin{proof}
Let $Q$ be any probability measure on $\mathbb{R}^d$. Consider $\tilde  Q$ the probability measure defined through
\begin{align*}
  d \tilde Q (v) = \int f_0(v|w)\, d Q(w) \, d v.
\end{align*}
Let $\epsilon>0$ and consider the centres  $f_1,\; \ldots,\;  f _N$ of an $\epsilon$-covering of the {\sc VC} class $\mathcal L' = \{z\in \mathbb{R}^q\mapsto K(h^{-1}({v-z}))\ :\ v\in \mathbb{R}^q,\ h>0 \}$ (see the lemma above) with respect to the metric $L_2(\tilde Q) $. For any function $w\in \mathbb{R}^d\mapsto \mathbb{E}[ f(V) \mid W=w ]$ with $f$ in $ \mathcal L'$, there exists $k\in \{1,\ldots,N\} $ such that
\begin{align*}
\int \left(\mathbb{E}\bigl[f (V)\mid W=w\bigr]- \mathbb{E}\bigl[f_k(V)\mid W=w\bigr] \right)^2 \, d Q(v)&\leq \int  \mathbb{E} \bigl[ ( f (V)- f_k(V)) ^2\mid W=w\bigr] \, d Q(v)\\
&=\int \int ( f(v)- f_k(v)) ^2\, f_0(v|w) \, d v \, d Q(W)\\
&=  \int (f(v) - f_k(v))^2 \, d \tilde Q(v)\leq \epsilon^2 ,
\end{align*}
using Jensen's inequality and Fubini's theorem.
Consequently, we have:
\begin{align*}
  \mathcal N \bigl( \mathcal L, L_2(Q), \epsilon \bigr)
  \leq
  \mathcal N \bigl( \mathcal L', L_2(\tilde Q), \epsilon \bigr).
\end{align*}
Since the kernel $K$ is bounded, the constant $\vert\vert K\vert\vert_\infty$ is an envelope for both classes $\mathcal L$ and $\mathcal L'$. Denoting by $(v,A)$ the constants related to the {\sc VC} property of class $\mathcal L'$, it follows that
\begin{align*}
  \mathcal N \bigl( \mathcal L, L_2(Q), \epsilon \|K\|_\infty \bigr)
  \leq
  \mathcal N \bigl( \mathcal L', L_2(\tilde Q), \epsilon \|K\|_\infty \bigr)
  \leq
  \left(\frac{A}{\epsilon}\right)^{v},
\end{align*}
which establishes the desired result.

\end{proof}

The preservation result below is also used in the subsequent analysis.

\begin{lemma}\label{lemma:covering_number_convolution}
Suppose that $\eta:\mathbb{R}^d\to \mathbb{R}$ is a Lipschitz function with constant $\kappa>0$, \textit{i.e.} $\vert \eta(u)-\eta(u') \vert\leq \kappa \vert\vert  u-u'\vert\vert$ for all $u,\; u'$ in $\mathbb{R}^d$, and $L:\mathbb{R}^d\to \mathbb{R}$ a positive function such that $\int L (u) du =1$ and $v_L = \int \|u\|^2 L(u) du <\infty$. Let $\tilde{h}>0$.
The class $\mathcal L = \{z\mapsto ( \eta\ast  L_h(z) - \eta(z))  \ :\  0< h \leq \tilde h\}$ is a bounded measurable VC class of functions with constant envelope $\tilde h \kappa \sqrt{v_L}$.
\end{lemma}

\begin{proof}
Let $0< \epsilon \leq 1$ and $h_ k = k \epsilon \tilde h$, $k=1,\ldots, \lfloor 1/\epsilon\rfloor $, an $(\epsilon \tilde h)$-subdivision of the interval $(0,\tilde h]$. Since
\begin{align*}
 \eta\ast  L_h(z) -  \eta \ast  L_{h_k}(z) = \int (\eta(z-hu) - \eta(z-h_ku) ) L(u)  du ,
\end{align*}
we have
\begin{align*}
( \eta \ast  L_h(z) -  \eta \ast  L_{h_k}(z) )^2 \leq  \int (\eta(z-hu) - \eta(z-h_k u) )^2 L(u)  du\leq  (\epsilon \tilde h)^2 \kappa^2 v_L.
\end{align*}
This shows that $ \mathcal N \bigl( \mathcal L, \|\cdot \|_\infty, \epsilon   \tilde h \kappa \sqrt{v_L} \bigr)\leq 1/ \epsilon $. It remains to obtain that $ \tilde h \kappa \sqrt{v_L} $ is an envelope for the class $\mathcal L$. This is because
\begin{align*}
 | \eta\ast  L_h(z) -  \eta (z) |  \leq  \int  | \eta(z-hu) - \eta(z) | L(u)  du\leq h\kappa \int \|u\| L(u)  du\leq \tilde h \kappa \sqrt{v_L}  .
\end{align*}
\end{proof}
 \section{Preliminary Results}\label{sec:auxiliary_results}
As a first go, we start with establishing bounds for quantities involved in Proposition \ref{prop:bound_CKM}'s proof: integrals with respect to signed measures, survival functions 	and hazard functions namely. This corresponds to Lemmas \ref{lemma:dudley}, \ref{lemma:survival} and \ref{lemma:lambda}, respectively. In Lemma \ref{lemma:bound_prob}, the fluctuations of the two local averages $\hat H_{0,n}$ and $ \hat H_{n}$, involved in the definition of the estimated hazard, are studied.

\begin{lemma}\label{lemma:dudley}
Let $\theta\in (0,1)$, $h:\mathbb R_+ \rightarrow [1,\; \infty[$ be borelian, increasing, with limit $1/\theta$ at $+\infty$ and $\nu$ be any signed measure on $\mathbb R_+$. Then, we have:  $\forall T>0$, $\forall t\in [0,T]$,
\begin{equation*}  \left| \int_0^t h d \nu \right|  \leq \frac{2}{\theta}\sup_{s\in [0,T]}\left|  \int_0^s d\nu \right|.
\end{equation*}
\end{lemma}
\begin{proof}
Recall first the identity between the sup norm and the total variation norm of any signed measure $\nu$:
\begin{equation}\label{eq:dudley}
\sup_{t\geq 0 }  \left| \int_0^t  d\nu \right|  =   \sup_{f\in DE}  \left| \int f d\nu \right|,
\end{equation}
where $DE$ is the space of non-increasing functions valued in $[0,1]$ and vanishing at infinity (see \textit{e.g.} \cite{dudley:2010}). Since $h$ is increasing from $1$ to $1/\theta$, we have for any signed measure $\nu$ (whose restriction to $[0,T]$ is denoted by $\nu_{[0,T]}$),
\begin{align*}
\left|\int_0^t  h d \nu\right|   = \theta^{-1} \left|\int_0^t      d \nu    +   \theta \int_0^t  \left(h -   \theta^{-1} \right)   d \nu\right|
\leq 2\theta^{-1}   \sup_{f\in DE}  \left| \int f d\nu_{[0,T]} \right|.
\end{align*}
Then applying \eqref{eq:dudley} we obtain that
\begin{align*}
 \left|\int_0^t  h d \nu\right|  \leq \frac{2}{\theta} \sup_{s\geq 0 }  \left| \int_0^s  d\nu_{[0,T]} \right|  = \frac{2}{\theta}  \sup_{s\in [0,T] }  \left| \int_0^s  d\nu \right| .
\end{align*}
\end{proof}

\begin{lemma}\label{lemma:survival}
Let $\tau>0$. Let $S^{(1)}$ and $S^{(2)}$ be survival functions (\textit{i.e.} c\`ad-l\`ag non-increasing functions) on $\mathbb{R}_+$ such that $S^{(1)}(0) = S^{(2)}(0) = 1$ and $S^{(2)}(\tau)\geq \theta>0 $. For $k\in\{1,\; 2\}$, $\Lambda^{(k)}(t) = - \int_0^t dS^{(k)}(u)/S^{(k)}(u-)$ is the corresponding cumulative hazard function. We have:
\begin{align*}
\| S^{(1)} - S^{(2)} \|_{[0,\tau]}  \leq  2\theta^{-1} \| \Lambda^{(1)} -\Lambda^{(2)}\|_{[0,\tau]}  .
\end{align*}
\end{lemma}

\begin{proof}
Let $t\in [0,\tau]$. As $S^{(2)}(t)>0$, the integration by part argument of Theorem 3.2.3 in \cite{fleming+h:2011} yields
\begin{align}\label{eq:int_by_part}
 \frac{  S^{(1)}(t ) -  S^{(2)}(t)}{S^{(2)}(t )  } &= -  \int_0^t \frac{  S^{(1)} (u -  ) }{ S^{(2)} (u ) } d ( \Lambda^{(1)}(u) -\Lambda^{(2)}(u)).
\end{align}
Set $d\Delta_1  =       d ( \Lambda^{(1)} -\Lambda^{(2)}) / S^{(2)}$ and apply the integration by parts formula (refer to page 305 in \cite{shorack+w:2009} for instance) to get
\begin{equation*}
 \frac{  S^{(1)}(t ) -  S^{(2)}(t)}{S^{(2)}(t )  }  = -  \int_0^t  S^{(1)} (u -)  d   \Delta_1 (u )\\
 =  - S^{(1)} (t  )   \Delta_1 (t) +   \int_0^t    \Delta_1 (u) d S^{(1)}  (u  ).
\end{equation*}
Then, as $S^{(2)}(t)\leq 1$, we obtain that
\begin{equation*}
| S^{(1)}(t ) -  S^{(2)}(t) |  \leq  \left(  S^{(1)} (t  )  |   \Delta_1  (t)| + (1 -  S^{(1)}  (t ) )  \sup_{u\in [0,\tau]} |  \Delta_1  (u)|    \right)
 \leq   \sup_{u\in [0,\tau]} | \Delta_1 (u)|.
\end{equation*}
We conclude by using Lemma \ref{lemma:dudley} with $d\nu=d ( \Lambda^{(1)} -\Lambda^{(2)})$ and $h=1/S^{(2)}$.
\end{proof}

\begin{lemma}\label{lemma:lambda}
Let $0<\theta_1,\; \theta_2<1$ and $\tau>0$. For $k\in\{1,\; 2\}$, define $\Lambda^{(k)}(t) = \int_0^t dG ^{(k)} / H^{(k)}  $, where $G^{(k)}:[0,\tau] \to [0,\beta]$ is c\`ad-l\`ag non-decreasing and $H^{(k)}:[0,\tau] \to [\theta_k,1]$ is Borelian non-increasing. Then, we have:
\begin{align*}
\| \Lambda^{(1)} - \Lambda^{(2)} \|_{[0,\tau]} \leq (2/\theta_1) \| G^{(1)} -G^{(2)}\|_{[0,\tau]} + \beta/(\theta_1\theta_2)\| H^{(1)} -  H^{(2)}\|_{[0,\tau]} .
\end{align*}
\end{lemma}

\begin{proof}
Let $t\in [0,\tau]$. Observe that, by triangular inequality,
\begin{align*}
\left| \Lambda^{(1)}(t) - \Lambda^{(2)}(t)\right| &= \left| \int_0^t \frac{ d(G^{(1)} -G^{(2)}  ) }{  H^{(1)} } +\int_0^t \frac{( H^{(2)} -  H^{(1)}  )}{ H^{(1)} H^{(2)}} d G^{(2)}  \right|\\
& \leq {2}{\theta_1^{-1} } \| G^{(1)} - G^{(2)}\|_{[0,\tau]} +  {\beta}{\theta_1^{-1}\theta_2^{-1}}  \| H^{(2)} -  H^{(1)}\|_{[0,\tau]} ,
\end{align*}
where the bound for the second term on the right hand side is straightforward and that for the first term can be deduced from the application of Lemma \ref{lemma:dudley} with the measure $\nu$ equal to  $A\mapsto \int _ Ad(G^{(1)} - G^{(2)}) $ and the function $h$ equal to $1/H^{(1)}$.
\end{proof}

\begin{lemma}\label{lem:res_tool}
Let $\tau>0$. Let $S^{(1)}$ and $S^{(2)}$ be survival functions on $\mathbb{R}_+$ such that $S^{(1)}(0) = S^{(2)}(0) = 1$ and $S^{(2)}(\tau)\geq \theta>0 $. For $k\in\{ 1,\; 2\}$, define $\Lambda^{(k)}(t) = - \int_0^t  S^{(k)}(u-) dS^{(k)}(u)  $ and suppose that $\Lambda^{(k)}(t) = \int_0^t dG ^{(k)}(u) / H^{(k)}(u)  $, where  $G^{(k)}:[0,\tau] \to [0,\beta]$ and $H^{(k)}:[0,\tau] \to [\theta,1]$ are respectively non-decreasing and non-increasing borelian functions. Then, there exists a constant $C_{\theta,\beta}>0$, depending only on $\theta$ and $\beta $, such that
\begin{multline*}
\sup_{t\in [0,\tau]} \left| \int_0^t \frac{\left( S^{(1)} (u-) - S^{(2)} (u-) \right)}{S^{(2)}(u)} d\left( \Lambda^{(1)} (u) -  \Lambda^{(2)} (u) \right)   \right|  \leq \\
C_{\theta,\beta} \left(\|  H^{(1)} - H^{(2)} \|_{[0,\tau]}^2+\| G^{(1)} - G^{(2)} \|_{[0,\tau]} ^2+\|W\|_{[0,\tau]}\right),
\end{multline*}
where
\begin{equation*}
W(t) =  \int_{u=0}^t \int_{s=0}^u \frac{S^{(2)} (s-) d \left( G ^{(1)} (s)  -G ^{(2)} (s)\right)}{ S^{(2)} (s)H^{(2)}(s)}  \frac{d \left(G ^{(1)} (u)  -G ^{(2)} (u) \right)}{ S^{(2)} (u)  H^{(2)}(u) }.
\end{equation*}
\end{lemma}

\begin{proof}

The proof consists in showing first that there exist constants $C_{1,\theta,\beta}$ and $C_{2,\theta,\beta}$ such that
\begin{multline}\label{eq:lemma_bound_hats-s}
\sup_{t\in [0,\tau]} \left| \int_0^t \frac{(\hat S^{(1)} (u-) - S^{(2)} (u-) )}{S^{(2)}(u)} d( \Lambda^{(1)} (u) -  \Lambda^{(2)} (u) )   \right| \leq  \\ C_{1,\theta,\beta} (\| G^{(1)} -G^{(2)}\|^2_{[0,\tau]} + \| H^{(1)} -  H^{(2)}\|^2_{[0,\tau]})  + \| {\Pi} \|_{[0,\tau]},
\end{multline}
where
\begin{equation*}
 {\Pi}(t) =\int_0^t \Delta_{2} (u)d \Delta_{1}(u),\quad  \Delta_{2}(t)  =  \int_{0}^t S^{(2)} (u-)  d \Delta_{1}(u),\quad \Delta_{1}(t)  =  \int_{0}^t S^{(2)} (u)^{-1}  d \Delta(u),
\end{equation*}
and $\Delta = \Lambda^{(1)}  -  \Lambda^{(2)}$, and next that
\begin{equation}
 \| {\Pi}  - W \|_{[0,\tau]} \leq C_{2,\theta,\beta} \left(\|  H^{(1)} - H^{(2)} \|_{[0,\tau]}^2+\| G^{(1)} - G^{(2)} \|_{[0,\tau]}^2\right).
\label{eq:lemma_bound_K}
\end{equation}

In order to establish \eqref{eq:lemma_bound_hats-s}, we successively apply \eqref{eq:int_by_part}, Fubini's theorem and the integration by part formula:
\begin{multline}
\int_{u=0}^t ( S^{(1)}(u-) - S^{(2)} (u-) ) d \Delta_1 (u)= - \int_{u=0}^t  \int_{v=0}^ {u-}  S^{(1)} (v-) d \Delta_{1}(v)   S^{(2)} (u-)  d \Delta_{1}(u) \\
= - \int_{v=0}^{t}   \left(\int_{u=v}^t S^{(2)} (u-)  d \Delta_{1}(u)\right)  S^{(1)} (v-) d \Delta_{1}(v) \\
=   - \Delta_{2}(t) \int_0^{t}    S^{(1)}(v-) d \Delta_{1}(v)  +   \int_0^{t}      S^{(1)}(v-) d \Pi (v)\\
=  -\Delta_{2}(t) \left(  S^{(1)}(t)  \Delta_{1}(t) - \int_0^t  \Delta_{1}(u) d  S^{(1)}(u)  \right)
\quad  +      S^{(1)}(t) {\Pi}(t)   -   \int_0^t {\Pi}(u) d  S^{(1)}(u)\\
 \leq 2 \|\Delta_2\|_{[0,T]} \|\Delta_1\|_{[0,T]} +  2\|\Pi\|_{[0,T]}.\label{eq:tool1}
\end{multline}
From Lemma \ref{lemma:dudley}, we deduce that $\|\Delta_2\|_{[0,T]}\leq  \|\Delta_1\|_{[0,T]}$ and that $\|\Delta_1\|_{[0,T]} \leq 2  \theta ^{-1}\|\Delta\|_{[0,T]}$. Apply next Lemma \ref{lemma:lambda} to obtain
\[
\|\Delta_2\|_{[0,T]} \|\Delta_1\|_{[0,T]} \leq 8  \theta ^{-2} \left({4}{\theta^{-2} } \| G^{(1)} -G^{(2)}\|^2_{[0,\tau]} + {\beta^2}{\theta^{-4}}\| H^{(1)} -  H^{(2)}\|^2_{[0,\tau]} \right).
\]
Combined with \eqref{eq:tool1}, this proves \eqref{eq:lemma_bound_hats-s}. For \eqref{eq:lemma_bound_K}, the application of the Taylor expansion
\begin{align}\label{eq:taylor_1/x}
\frac 1 x = \frac 1 a  -\frac{(x- a)}{a^2 } +\frac{(x-a)^2}{xa^2}
\end{align}
yields
\begin{align}
\label{eq:useful_decomp_lambda}   d  \Delta  &=    \frac{ d (  G^{(1)}   - G^{(2)} )}{H^{(2)} }   -  \frac{(  H ^{(1)}   -H^{(2)}  )   d   G^{(1)}}{ (H^{(2)})^2 }
+     \frac{  (  H^{(1)} -H^{(2)}  )^2   d   G^{(1)}  } {(H^{(2)})^2 H^{(1)}}   .
\end{align}
Set $c(s)=S^{(2)}(s-)/S^{(2)}(s)$. It follows that
\begin{multline*}
{\Pi} (t)= \int_{u=0}^t \int_{s=0}^u c(s) \left( \frac{ d \left(  G^{(1)} (s)  - G^{(2)} (s)\right)}{H^{(2)}(s)}   -  \frac{\left(  H^{(1)}(s) -H^{(2)}(s) \right)   d    G^{(1)} (s)}{ H^{(2)}(s)^{2}}  \right. \\
 \qquad + \left.    \frac{  \left(  H^{(1)}(s) -H^{(2)}(s) \right)^2   d    G^{(1)} (s)}{H^{(2)}(s)^2  H^{(1)}(s)}\right)    d \Delta_{1}(u).
\end{multline*}
Observe that
\begin{align*}
{\Pi}(t)-W(t) &=  - \int_{u=0}^t \int_{s=0}^u c(s)  \frac{ d \left(  G^{(1)} (s)  - G^{(2)} (s)\right)}{H^{(2)}(s)} \frac{\left(  H^{(1)}(u)  -H^{(2)}(u)  \right) d  G^{(1)}(u)  }{ S^{(2)}(u) H^{(1)}(u) H^{(2)}(u) }\\
&\quad + \int_{u=0}^t \int_{s=0}^u c(s)   \frac{\left(  H^{(1)}(s) -H^{(2)}(s) \right)   d    G^{(1)} (s)}{ H^{(2)}(s)^{2}} \frac{\left(  H^{(1)}(u)  -H^{(2)}(u)  \right) d  G^{(1)}(u)  }{ S^{(2)}(u) H^{(1)}(u) H^{(2)} (u) }\\
& \quad- \int_{u=0}^t \int_{s=0}^u c(s)  \frac{\left(  H^{(1)}(s) -H^{(2)}(s) \right)   d    G^{(1)} (s)}{ H^{(2)}(s)^{2}}\frac{d \left( G^{(1)}(u)  -G^{(2)}(u) \right)}{ S^{(2)} (u)  H^{(2)}(u) }  \\
& \quad+  \int_{u=0}^t \int_{s=0}^u  \frac{  \left(  H^{(1)}(s) -H^{(2)}(s) \right)^2   d    G^{(1)} (s)}{H^{(2)}(s)^{2}  H^{(1)}(s)}  d \Delta_{1}(u)=A+B+C+D.
\end{align*}
We next bound each term on the right hand side of the equation above. Successively apply Lemma \ref{lemma:dudley} and \eqref{eq:dudley} to get
\begin{align*}
\left| \int_0^u c(s)  \frac{ d \left(  G^{(1)} (s)  - G^{(2)} (s)\right)}{H^{(2)}(s)} \right|&\leq  \frac{2}{ \theta^2 } \sup_u \left| \int_0^u S^{(2)} (s-) { d \left(  G^{(1)} (s)  - G^{(2)} (s)\right)}\right|\\
&=  \frac{2}{ \theta^2 } \sup_u \left| \int S^{(2)} (s-) \mathds{1}_{s \leq u} { d \left(  G^{(1)} (s)  - G^{(2)} (s)\right)}\right| \\
&\leq \frac{2}{ \theta^2 } \|G^{(1)}   - G^{(2)} \|_{[0,\tau]}.
\end{align*}
Because, for any $u\in [0,\tau]$, $  1 /  \{ S^{(2)}(u) H^{(1)}(u) H^{(2)} (u) \}\leq 1/\theta^{3}$, we can write
\begin{align*}
\vert A \vert &\leq  (1/ \theta^{3}) \int_{u=0}^t  \left|  \int_{s=0}^u c(s) \frac{ d \left(  G^{(1)} (s)  - G^{(2)} (s)\right)}{H^{(2)}(s)}\right| \left| H^{(1)}(u)  -H^{(2)}(u)  \right| d  G^{(1)}(u)   \\
 &\leq  (1 / (2\theta^{3}))  \int_{u=0}^t \left\{  \left(  \int_0^u c(s)  \frac{ d \left(  G^{(1)} (s)  - G^{(2)} (s)\right)}{H^{(2)}(s)}  \right)^2 +    \left( H^{(1)}(u)  -H^{(2)}(u)  \right)^2 \right\}d  G^{(1)}(u)\\
 &\leq  \beta \left(  (2/ \theta^{7}) \|G^{(1)}   - G^{(2)} \|_{[0,\tau]}^2 + \|H^{(1)}  - H^{(2)} \|_{[0,\tau]}^2\right).
\end{align*}
In addition, because for any $u\in [0,\tau]$, $ c(u)  / (H^{(2)}(u))^2\leq 1 / \theta^3 $ we have: $\forall t\in[0,\tau]$,
\begin{align*}
\vert B \vert &\leq  (1 / \theta^3 )^2  \int_{u=0}^t  \int_{s=0}^t   \left|  H^{(1)}(s) -H^{(2)}(s) \right|  d    G^{(1)} (s)  {\left|  H^{(1)}(u)  -H^{(2)}(u)  \right|  d  G^{(1)}(u)  }\\
&=   1 / \theta^6  \left( \int_{s=0}^t    \left|  H^{(1)}(s) -H^{(2)}(s) \right|  d    G^{(1)} (s)  \right)^2\\
& \leq ( \beta^2  / \theta^6 )     \left\|  H^{(1)}  -H^{(2)}  \right\|^2 _{[0,\tau]}.
\end{align*}
Define $ \Gamma_{2} (t) =  \int_0^t  \frac{d ( G^{(1)}(u)  - G^{(2)}(u) )  }{ S^{(2)} (u)  H^{(2)}(u) }$. Applying Fubini's theorem, we get
\begin{align*}
\vert C\vert &= \left| \int_{u=0}^t \int_{s=0}^u c(s)   \frac{\left(  H^{(1)}(s) -H^{(2)}(s) \right)   d    G^{(1)} (s)}{ H^{(2)}(s)^{2}} \frac{d \left( G^{(1)}(u)  -G^{(2)}(u) \right)  }{ S^{(2)}(u)  H^{(2)}(u) }\right|\\
&=\left|   \int_{s=0}^t  \int_{u=s}^t  \frac{d \left( G^{(1)}(u)  -G^{(2)}(u) \right)  }{ S^{(2)}(u)  H^{(2)}(u) }   c(s)   \frac{\left(  H^{(1)}(s) -H^{(2)}(s) \right)   d    G^{(1)} (s)}{ (H^{(2)}(s))^{2}} \right|\\
&\leq     (1 / \theta^3 ) \int_{s=0}^t \left\{ \left| \Gamma_{2}(t) -  \Gamma_{2} (s) \right|  \times  \left|  H^{(1)}(s) -H^{(2)}(s) \right|    \right\} d    G^{(1)} (s)\\
&\leq   2 (1 / \theta^3 ) \beta\| \Gamma_2 \|_{[0,\tau]} \| H^{(1)}  -H^{(2)} \|_{[0,\tau]}.
\end{align*}
Then, using Lemma \ref{lemma:dudley}, it follows that
\begin{align*}
\vert C\vert & \leq   2 (1 / \theta^3 ) \beta (2 / \theta^3 )  \|G^{(1)}   - G^{(2)} \|_{[0,\tau]} \| H^{(1)}  -H^{(2)} \|_{[0,\tau]}\\
& \leq    (1 / \theta^3 ) \beta (2 / \theta^3 ) (  \|G^{(1)}   - G^{(2)} \|_{[0,\tau]}^2 +  \| H^{(1)}  -H^{(2)} \|_{[0,\tau]}^2).
\end{align*}
The last term can be treated by means of Fubini's theorem. Indeed, because $\|\Delta_{1} \|_{[0,\tau]} \leq 2 (\beta / \theta)$ and for any $  u\in [0,\tau] $, $1 / \{ H^{(2)} (u)^{2}  H^{(1)}(u)\}\leq 1/\theta^3$, we have
\begin{align*}
  \lvert D \rvert &= \left| \int_{u=0}^t \int_{s=0}^u  \frac{  \left(  H^{(1)} (s) -H^{(2)} (s) \right)^2   d    G^{(1)} (s)}{(H^{(2)} (s))^{2}  H^{(1)}(s)}   d \Delta_{1}(u) \right| \\
 &\leq \int_{s=0}^t \left| \left(\int_{u=s}^t d \Delta_{1}(u)\right)   \frac{  \left( H^{(1)} (s) -H^{(2)} (s) \right)^2   d    G^{(1)} (s)}{H^{(2)} (s)^{2}  H^{(1)}(s)}  \right|  \\
 &\leq   2 (1/\theta^3) \beta  \|\Delta_{1} \|_{[0,\tau]}  \| H^{(1)}  -H^{(2)} \|_{[0,\tau]} ^2\\
 &\leq  4 (1/\theta^4) \beta^2 \| H^{(1)}  -H^{(2)} \|_{[0,\tau]}^2.
 \end{align*}
Putting all this together, the triangular inequality leads to \eqref{eq:lemma_bound_K} .

\end{proof}

Now these preliminary results are established, the proof of Proposition \ref{prop:bound_CKM} is then mainly based on the following lemmas. The first one states classic kernel smoothing approximation results, while the second one immediately results from the application of Lemma \ref{lem:concentration_ineq_empi_sum} to appropriate classes of functions.

\begin{lemma}\label{lem:bias}
Under Assumption \ref{hyp:smooth}, for all $h>0$,
\begin{eqnarray}
\sup_{(t,x)\in \mathbb{R}_+\times \mathbb{R}^d} \vert  H_{0,h}(t,x)-H_{0}(t\mid x)g(x)\vert &\leq & C_0h^2,\label{eq:approx2}\\
\sup_{(t,x)\in \mathbb{R}_+\times \mathbb{R}^d} \vert  H_{h}(t,x)-H(t\mid x)g(x)\vert &\leq & C_0 h^2,
\end{eqnarray}
where $C_0=(L/4)\sum_{\alpha\in \mathbb{N}^d,\; \vert\alpha\vert=2}\int_{z\in \mathbb{R}^d}\vert K(z)\vert \prod_{i=1}^d \vert z_i \vert^{\alpha_i}dz$.
\end{lemma}
\begin{proof}
The proof results from the application of Lemma \ref{lem:approx} combined with the smoothness assumptions stipulated.
\end{proof}

In the following the constants denoted by $M_i$, $i=1,2,\ldots,$ are understood to be constants depending on quantities which will be specified. Similarly constants denoted by $\tilde{M}_i$, $i=1,2,\ldots,$ will be used as intermediary constants in the proofs. These constants (contrary to $b$ or $R$) are not necessarily the same at each appearance.

\begin{lemma}\label{lemma:bound_prob}
Under assumption \ref{hyp:bound1}. There exist constants $M_1>0$ and $h_0>0$ depending only on $K$ and $R$ such that:
\begin{align*}
\mathbb{P}\left\{ \sup_{(t,x)\in \mathbb{R}_+\times \mathbb{R}^d}\vert \hat H_{0,n}(t , x) - H_{0,h}(t , x)  \vert  \leq \sqrt{ \frac{M_1 |\log(\epsilon h ^{d/2} ) |}{ nh^{d}}}  \right\} \geq  1 - \epsilon, \\
\mathbb{P}\left\{ \sup_{(t,x)\in \mathbb{R}_+\times \mathbb{R}^d}\vert \hat H_{n}(t, x) - H_h(t, x)  \vert  \leq \sqrt{ \frac{M_1 |\log(\epsilon h ^{d/2} ) |}{ nh^{d}}}   \right\} \geq  1 - \epsilon,
\end{align*}
 provided that $h  \leq h_0 $ and $M_1 |\log(\epsilon h ^{d/2} ) | \leq n h^d$.
\end{lemma}

\begin{proof}
The exponential inequalities stated above directly result from the application of Corollary \ref{lem:concentration_ineq_Ustat_cor} to the uniformly bounded {\sc VC}-type classes $\{x'\in \mathbb{R}^d\mapsto K((x-x')/h):\; (x,h)\in \mathbb{R}^d\times \mathbb{R}_+^* \}$ and $\{(y,\delta, x')\in \mathbb{R}_+\times \{0,1\}\times \mathbb{R}^d\mapsto \mathbb{I}\{y>u,\; \delta=0\}K((x-x')/h):\; (x,u,h)\in \mathbb{R}^d\times \mathbb{R}_+\times \mathbb{R}_+^* \}$
whose {\sc VC} constants are independent from $h$, with constant envelope $\vert \vert K  \vert\vert_{\infty} $, $k = 1$ and $\sigma^2=c^2_{K,R}h^d$ with $c_{K,R}=\sqrt{R\int K^2(x)dx}$. This gives that
\begin{align*}
\mathbb{P}\left\{ \sup_{(t,x)\in \mathbb{R}_+\times \mathbb{R}^d}\vert \hat H_{0,n}(t , x) - H_{0,h}(t , x)  \vert  \leq t \right\} \geq  1 - \epsilon ,\\
\mathbb{P}\left\{ \sup_{(t,x)\in \mathbb{R}_+\times \mathbb{R}^d}\vert \hat H_{n}(t, x) - H_h(t, x)  \vert  \leq t \right\} \geq  1 - \epsilon,
\end{align*}
with
\begin{equation*}
  t = \frac{c_{K, R}}{\sqrt{n h^d}} \left(  \left( \frac{1}{C_3} \log \left( \frac{C_2}{\epsilon} \right) \right)^{1/2} + C_1 \left( \log \left( \frac{2 \lVert K \rVert_\infty}{c_{K, R} h^{d/2}} \right) \right)^{1/2}  \right),
\end{equation*}
provided that $h^{d/2} c_{K,R} \leq \vert \vert K  \vert\vert_{\infty} $ and
\[
\frac{\lVert K\rVert^2_{\infty}}{c^2_{K,R}} \left(  \frac{1}{C_3} \log \left( \frac{C_2}{\epsilon} \right)  + C_1^2 \log \left( \frac{ 2 \lVert K\rVert_{\infty}}{c_{K,R}h^{d/2} } \right)\right) \leq n h^d.
\]
Since, for any positive numbers $a,b,\gamma$, it holds that $a^{\gamma} + b^{\gamma} \leq 2^\gamma ( a+ b )^\gamma$, we find that $t^2\leq  \tilde M_1 |\log(\epsilon h ^{d/2} ) |/ nh^{d} $ for some constant $M_1>0$. Finally, taking $h_0$ sufficiently small ensures that $\log(C_2)/C_3 + C_1^2\log( 2\|K\|_\infty / c_{K,R}) \leq C_1^2 \log(1/h^{d/2}) $, for any $h\leq h_0$, which permits to ensure that the previous condition is satisfied whenever $\tilde M_2 |\log(\epsilon h ^{d/2} ) | \leq nh^{d}$, for some $\tilde M_2>0$. Take $M_1 = \tilde M_1 + \tilde M_2$ to obtain the desired result.
\end{proof}

\begin{lemma}\label{lemma:event_Hn}
Suppose that Assumptions \ref{hyp:cond_ind}, \ref{hyp:smooth} and \ref{hyp:bound1} are fulfilled. There exist constants $M_1>0$ and $h_0>0$ depending only on $b$, $R$ and $K$ such that:
\begin{align*}
\mathbb P  \left\{ \inf_{(t,x)\in {\Gamma_{b}}} \hat{H}_n(t, x)\geq  {b^3/4}  \right\} \geq 1 - \epsilon,
\end{align*}
provided that $h\leq h_0$ and $M_1 |\log(  \epsilon h^{d/2}  ) | \leq nh^d  $.
\end{lemma}

\begin{proof}
Define
\begin{align*}
\mathcal{A}_{n} =\left\{ \sup_{(t,x)\in {\Gamma_{b}}}\vert H(t\mid x)g(x)- \hat{H}_{n}(t, x)\vert \leq {3b^3/4}  \right\}.
\end{align*}
By virtue of Assumption \ref{hyp:cond_ind}, for any $(t,x) \in {\Gamma_{b}}$, we have: $H (t|x) = S_C(t|x) S_Y(t|x) \geq  b^2 $.
As a consequence of $\hat{H}_n(t, x) \geq H(t|x) g(x)  - | H(t|x) g(x)  - \hat H_n(t,x)|$, $\mathcal{A}_{n} \subset \{ \inf_{(t,x)\in {\Gamma_{b}}} \hat{H}_n(t, x)\geq {b^3/4}  \}$. Hence we only have to prove that event $ \mathcal{A}_{n} $ occurs with probability $1-\epsilon$ at least. By virtue of Lemma \ref{lem:bias}, as soon as $h\leq   {\sqrt{3b^3/(8C_0)}}$, we have
$$
\sup_{(t,x)\in {\Gamma_{b}}}\left\vert H_h(t,x) -H(t\mid x)g(x)\right\vert\leq {3b^3 / 8},
$$
and thus
$$
 \left\{  \sup_{(t,x)\in \Gamma_{b}} \vert\hat H_{{n}}(t, x)-H_{h}(t, x) \vert \leq  {3b^3/8}  \right\}\subset \mathcal{A}_n .
$$
Simply use Lemma \ref{lemma:bound_prob} to ensure that the event in the right-hand side holds with probability $1-\epsilon$ whenever $M_1 |\log(  \epsilon h^{d/2}  ) | \leq nh^d$ (where $M_1$ now depends on $b$, $R$ and $K$) on $h\leq h_0$.
\end{proof}

\section{Proof of Proposition \ref{prop:bound_CKM}}
In the whole proof, we suppose that the assumptions of Lemma \ref{lemma:event_Hn} are satisfied so that $\inf_{(t,x)\in {\Gamma_{b}}} \hat{H}_n(t, x)\geq {b^3/4} $ happens with probability $1-\epsilon/3$. We suppose that this event is realized in the following. Let $(t,x) \in {\Gamma_{b}}$ and define
$$ \tau_x=\sup\{t\geq 0:\; \min\{ S_C(t|x),\; S_Y(t|x)\} \geq b   \}.$$
Observing that the choice of kernel $K$ guarantees that $\hat S_{C,n} (\cdot | x)$ is a (random) survival function, we first apply Lemma \ref{lemma:survival} with $S^{(1)}=\hat S_{C,n} (\cdot | x)$, $S^{(2)}=S_C (\cdot | x)$ and  $\theta=b$ to get:
\begin{equation}\label{eq:1}
\| \hat S_{C,n} (\cdot | x)- S_C (\cdot | x) \|_{[0,\tau_x]} \leq(2/b)\| \hat \Lambda_{C,n} (\cdot | x)- \Lambda_C (\cdot | x) \|_{[0,\tau_x]} .
\end{equation}
Applying Lemma \ref{lemma:lambda} with $\Lambda^{(1)}(u)=\Lambda_C(u \mid x)=-\int_0^uH_{0}(ds\mid x)g(x)/(H(s-\mid x)g(x))$, $\Lambda^{(2)}(u)=\hat{\Lambda}_{C,n}(u \mid  x)=-\int_0^u\hat{H}_{0,n}(ds, x)/\hat{H}_n(s-, x)$, $\beta=1$, $\theta_1={b^3}\leq H(s\mid x)g(x)$, $\theta_2={b^3/4}$ (because $\inf_{(t,x)\in {\Gamma_{b}}} \hat{H}_n(t, x)\geq {b^3/4}$), next yields
\begin{multline}\label{eq:2}
\|\hat \Lambda_{C,n} (\cdot | x)- \Lambda_C (\cdot | x) \|_{[0,\tau_x]}\leq {\frac{2}{b^3}}\|\hat H_{0,n}(\cdot , x)-H_0(\cdot \mid  x)g(x)  \| _{[0,\tau_x]}\\+ {\frac{4}{b^6}}  \| \hat H_n(\cdot ,  x) -H(\cdot \mid  x)g(x) \|_{[0,\tau_x]}.
\end{multline}
Combining \eqref{eq:1} and \eqref{eq:2}, { using Lemma~\ref{lem:bias}} and taking the supremum over $x$, we obtain that, the following bound holds true:
\begin{align}
\nonumber &\sup_{(t,x)\in {\Gamma_{b}}} \vert \hat S_{C,n}(t\mid x)- S_C(t\mid x) \vert\\
\nonumber &\leq {\frac{4}{b^4}} \sup_{(t,x)\in {\Gamma_{b}}} \vert\hat H_{0,n}(t, x)-H_0(t\mid x)g(x) \vert+
{\frac{8}{b^7}} \sup_{(t,x)\in {\Gamma_{b}}} \vert \hat H_n(t, x) -H(t \mid x)g(x) \vert\\
\label{eq:3}&\leq  {\frac{4}{b^4}} \sup_{(t,x)\in {\Gamma_{b}}} \vert\hat H_{0,n}(t, x)-H_{0,h}(t, x) \vert + {\frac{4}{b^4}}C_0h^2
+ {\frac{8}{b^7}} \sup_{(t,x)\in {\Gamma_{b}}} \vert \hat H_n(t, x) -H_h(t, x) \vert+ {\frac{8}{b^7}} C_0h^2.
\end{align}
Lemma \ref{lemma:bound_prob} with the probability level $\epsilon/3$ allows us to bound the $2$ previous random terms. Combined with the union bound (with $3$ events having probability smaller than $\epsilon/3$), permits claiming that with probability greater than $1-\epsilon$:
\begin{align*}
\sup_{(t,x)\in {\Gamma_{b}}} \vert \hat S_{C,n}(t\mid x)- S_C(t\mid x) \vert\leq {\frac{4}{b^4}}\left(1+{\frac{2}{b^3}}  \right)
 \left\{ C_0h^2+ \sqrt{ \frac{M_1 |\log(\epsilon h ^{d/2} ) |}{ nh^{d}}}\right\},
\end{align*}
provided that (to apply Lemma \ref{lemma:bound_prob}) $h\leq h_0$ and $nh^d \geq M_1 |\log(3\epsilon h ^{d/2} ) |$. Examining the different terms and taking $h_0$ small enough lead to the stated result.
 \section{Proof of Proposition \ref{prop:decomp}}

\noindent{\textit{Proof of (i):}
The fact that $H_{h}(t, x)$ is bounded away from zero on the domain $\mathcal{K}$, provided that $h$ is small enough, is obvious from Lemma \ref{lem:bias}. The parameter $h$ involved must satisfy $C_0h^2\leq  b^3 / 4$.
Concerning $S_{C,h}$, we work under the previous assumption to reproduce the argument of Proposition \ref{prop:bound_CKM}'s proof (see Eq. \eqref{eq:1},\eqref{eq:2},\eqref{eq:3}) combined with Lemma \ref{lem:bias}, we obtain that: $\forall (t,x)\in \mathcal{K}$,
\begin{equation}\label{eq:dev}
{\left\vert S_{C,h}(t\mid x)-S_{C}(t\mid x)  \right\vert\leq \frac{4C_0h^2}{b^4}(1+ 2/b^3)}.
\end{equation}
Hence, under the assumption that $\mathcal{K}\subset  {\Gamma_{b}}$, we deduce that $\inf_{(t,x)\in \mathcal{K}}S_{C,h}(t\mid x)\geq b/2$ as soon as
$h\leq b^{5/2} / \sqrt{8C_0(1+2/b^3)}$ from the bound above, which terminates the proof of $(i)$.

\medskip

\noindent{\textit{Proof of (ii):}
Observe that: $\forall i\in\{1,\; \ldots,\; n  \}$,
\begin{eqnarray}
\label{eq:4}\sup_{(t,x)\in \mathcal{K}}|\hat H^{(i)}_{0,n}(t , x)-\hat{H}_{0,n}(t,  x)| &\leq & 2\vert\vert K\vert\vert_{\infty}/((n-1)h^d),\\
\label{eq:5}\sup_{(t,x)\in \mathcal{K}}|\hat H^{(i)}_{n}(t , x)-\hat{H}_{n}(t,  x)|&\leq & 2 \vert\vert K\vert\vert_{\infty}/((n-1)h^d).
\end{eqnarray}
The result follows from the union bound and that each of these events
$$\mathcal{B}_n^{(1)} \overset{def}{=}\bigcap_{i\leq n}\left\{\forall (t,x)\in \mathcal{K},\;\; \hat H^{(i)}_n(t,x)\geq b^3 / 2  \right\},$$
$$\mathcal{B}_n^{(2)} \overset{def}{=}\bigcap_{i\leq n}\left\{\forall (t,x)\in \mathcal{K},\;\; \hat{S}^{(i)}_{C,n}(t,x)\geq b/2  \right\},$$
has probability $1-\epsilon/2$ under the mentioned condition on $(n,h)$.
Apply Lemma \ref{lemma:event_Hn} to choose $(n,h)$ such that with probability $1-\epsilon/2 $,
\begin{align*}
\inf_{(t,x)\in \mathcal{K}}  \hat{H}_{n}(t,  x) \geq 3b^3/4.
\end{align*}
Using \eqref{eq:4} and the triangle inequality, we get that $\mathcal{B}_n^{(1)} $ has probability $1-\epsilon/2$ provided that $2\vert\vert K\vert\vert_{\infty}/((n-1)h^d)\leq  b^3/4$.

Suppose that event $ \mathcal{B}_n^{(1)}  $ is realized. The same reasoning as that used in the proof of Proposition \ref{prop:bound_CKM} (see \eqref{eq:1},\eqref{eq:2},\eqref{eq:3}), with $S^{(1)} (\cdot) = S_C (\cdot | x) $, $S^{(2)} (\cdot) =S^{(i)}_{C,n}(\cdot | x)  $, $\theta_1 = b^3$ and $\theta_2 = b^3/4$ (as $ \mathcal{B}_n^{(1)}  $ is realized),  combined with the triangular inequality, yields: $\forall i\in\{1,\; \ldots,\; n  \}$,
\begin{multline*}
\sup_{(t,x)\in \mathcal{K}}| \hat S^{(i)}_{C,n} (t | x)- S_C (t | x) | \leq \\
{\frac{4}{b^4}}\left( \sup_{(t,x)\in \mathcal{K}}|\hat H^{(i)}_{0,n}(t , x)-\hat{H}_{0,n}(t,  x)|+\sup_{(t,x)\in \mathcal{K}}|\hat H_{0,n}(t , x)-H_{0}(t \mid  x)g(x)  |  \right)\\
+  {\frac{8}{b^7}}\left( \sup_{(t,x)\in \mathcal{K}}|\hat H^{(i)}_{n}(t , x)-\hat H_{n}(t , x) |+\sup_{(t,x)\in \mathcal{K}}|\hat H_{n}(t , x)-H(t \mid  x)g(x) | \right).
\end{multline*}
Hence, because of \eqref{eq:5}, if
$$
\left( {\frac{4}{b^4}}+{\frac{8}{b^7}} \right)2\vert\vert K\vert\vert_{\infty}/((n-1)h^d) \leq b/4,
$$
From $\mathcal{K}\subset {\Gamma_{b}}$, it results that
$$
\left\{\sup_{(t,x)\in \mathcal{K}}|\hat H_{0,n}(t , x)-H_{0}(t \mid  x)g(x)  | \leq {\frac{b^5}{16}} \right\} \bigcap \left\{ \sup_{(t,x)\in \mathcal{K}}|\hat H_{n}(t , x)-H(t \mid  x)g(x) | \leq { \frac{b^8}{32}} \right\}
$$
 is included in the set $\mathcal{B}_n^{(2)}$. Following the treatment of \eqref{eq:3}, it is easy to see that the latter event occurs with probability $1-\epsilon/2$ whenever $h\geq h_0$ is small enough (for the bias) and $nh^d \geq M_1  |\log(h^{d/2}\epsilon)| $.

\medskip

\noindent\textit{Proof of (iii).}
From now on, we assume that the conditions in (i)-(ii) are fulfilled and place ourselves on the event $\mathcal{B}_n$.
For all $i\in\{1,\; \ldots,\; n  \}$, recall that
\begin{align*}
&\hat \Lambda^{(i)}_{C,n} (u\mid x)  =  - \int_{s=0}^u\frac{ d\hat{H}^{(i)}_{0,n}(s,  x)}{\hat{H}^{(i)}_n(s-, x)},\\
&d \hat \Delta_n^{(i)} = d (\hat \Lambda_{C,n}^{(i)}  -\Lambda_{C,h} ),
\end{align*}}
and that $c_h(s\mid x) =  S_{C,h} (s-\mid x) /S_{C,h} (s\mid x)$. It results from Theorem 3.2.3 in \cite[pg.97]{fleming+h:2011} that
\begin{multline*}
\frac{\hat S^{(i)}_{C,n}(t\mid x) -  S_{C,h}(t\mid x)}{S_{C,h}(t\mid x)  } =\\
 - \int_0^t  {c_h(u\mid x)} d \hat \Delta_n^{(i)}(u\mid x)   -  \int_0^t \frac{( \hat S^{(i)}_{C,n} (u - \mid x) -  S_{C,h} (u - \mid x) )}{S_{C,h}(u\mid x)} d  \hat  \Delta_{n}^{(i)} (u\mid x) .
\end{multline*}
The Taylor expansion \eqref{eq:taylor_1/x} gives that
\begin{align}
 d  \hat  \Delta_{n}^{(i)} (u \mid x)
\label{eq:useful_decomp_lambda} &=    \frac{ d ( \hat H_{0,n}^{(i)} (u, x)  - H_{0,h} (u, x))}{H_h(u,x)}   -  \frac{( \hat H_n^{(i)}(u,x)  -H_h(u,x) )   d   \hat H_{0,n} ^{(i)}(u, x)}{ H_h(u,x)^{2}}  \\
\nonumber & \qquad +     \frac{  ( \hat H_n^{(i)}(u,x) -H_h(u,x) )^2   d   \hat H_{0,n} (u, x)}{H_h(u,x)^{2} \hat H_n^{(i)}(u,x)}   ,
\end{align}
which implies that
\begin{align}\label{eq:decomp_survival_1}
&\frac{\hat S_{C,n}^{(i)}(t\mid x) -  S_{C,h}(t\mid x)}{S_{C,h}(t\mid x)  }  = \hat a_{n}^{(i)}(t\mid x) + \hat b_{n}^{(i)}(t\mid x),
\end{align}
where
\begin{align*}
\hat a_{n}^{(i)}(t\mid x)& = - \int_0^t \frac{ c_{h} (u\mid x)   }{H_h(u,x) }    d ( \hat H_{0,n}^{(i)} (u, x)  - H_{0,h} (u, x)) \\
&\qquad  +  \int_0^t  \frac{  c_{h} (u\mid x)   }{   H_h(u,x)^2 } ( \hat H_n^{(i)}(u,x) -H_h(u,x) )   d   \hat H_{0,n}^{(i)} (u, x)  ,\\
\hat b_{n}^{(i)}(t\mid x)& =  -\int_0^t  \frac{  c_{h} (u\mid x)   }{  H_h(u,x)^{2}\hat H^{(i)}_n(u,x)} ( \hat H_n^{(i)}(u,x) -H_h(u,x) )^2   d   \hat H^{(i)}_{0,n} (u, x) \\
&\qquad  -  \int_0^t { \frac{( \hat S^{(i)}_{C,n} (u - \mid x) -  S_{C,h} (u - \mid x) )}{S_{C,h}(u\mid x)}} d  \hat  \Delta^{(i)}_{n} (u\mid x)  .
\end{align*}
The decomposition \eqref{eq:decomp_survival_1} involves two types of terms: the leading terms are the $\hat a^{(i)}_{n}$,'s while the $\hat b^{(i)}_{n}$'s are shown to be negligible. Now, using \eqref{eq:taylor_1/x}, we obtain that: $\forall \varphi \in \Phi$,
\begin{eqnarray*}
Z_n(\varphi)  &=& \frac{1}{n} \sum_{i=1}^n \left\{ \delta_i \frac{\varphi(\tilde Y_i , X_i)}{\hat{S}_{C, n}^{(i)} (\tilde Y_i \mid X_i)} - \mathbb{E} \left[ \delta  \frac{\varphi(\tilde Y, X)}{{S}_{C}(\tilde Y  \mid X)} \right] \right\} \\
  &=&   \mathbb{E}  \left[ \delta \frac{\varphi(\tilde Y, X)}{{S}_{C, h}(\tilde Y \mid X)}  \right]  -  \mathbb{E}  \left[ \delta \frac{\varphi(\tilde Y, X) }{{S}_{C}(\tilde Y \mid X)}  \right]\\
  &  +& \frac{1}{n} \sum_{i=1}^n  \left\{ \delta_i \frac{ \varphi(\tilde Y_i , X_i)}{{S}_{C, h}(\tilde Y_i \mid X_i)} - \mathbb{E}  \left[ \delta  \frac{ \varphi(\tilde Y_i , X_i)}{{S}_{C, h}(\tilde Y \mid X)} \right] \right\} \\
  & -& \frac{1}{n} \sum_{i=1}^n \delta_i  \varphi(\tilde Y_i , X_i) \left(\frac{  \hat{S}_{ C,n}^{(i)} (\tilde Y_i \mid X_i) - {S}_{C, h} (\tilde Y _i \mid X_i)}{{S}_{C, h}^2 (\tilde Y_i \mid X_i)}\right) \\
  &+& \frac{1}{n} \sum_{i=1}^n  \delta_i  \varphi(\tilde Y_i , X_i) \frac{({S}_{C, h} (\tilde Y_i \mid X_i) - \hat{S}_{C,n}^{(i)} (\tilde Y_i \mid X_i) )^2}{{S}_{C, h}^2 (\tilde Y_i  \mid X_i) \hat{S}_{C,n}^{(i)} (\tilde Y _i \mid X_i)}.
\end{eqnarray*}
Then, using \eqref{eq:decomp_survival_1}, we retrieve the expected terms
\begin{equation*}
Z_n(\varphi) = B_n (\varphi)+ L_n(\varphi) + V_n(\varphi) + R_n(\varphi),
\end{equation*}
which proves $(iii)$.
 \section{Proof of Proposition \ref{prop:main_result}}

The proof is based on the decomposition stated in Proposition \ref{prop:decomp}, combined with the lemmas below that permit to control each term involved in it. Their proof is given in the next section of the Appendix.
The next lemma provides a bound for the deterministic term.

\begin{lemma}\label{lem:B_term} Under the assumptions of Proposition \ref{prop:main_result}, there exists a constant $M_1>0$ depending on $M_\Phi$, and $b$ such that, $\forall n\geq 1$,
$$
\sup_{\varphi \in \Phi}\left\vert B_n(\varphi) \right\vert \leq M_1 h^2.
$$
\end{lemma}

The term $L_n(\varphi)$ is a basic i.i.d. (centred) average. As shown in the lemma stated below, its uniform fluctuations can be controlled by standard results in empirical process theory.

\begin{lemma}\label{lem:L_1} Suppose that the hypotheses of Proposition \ref{prop:main_result} are fulfilled. Then, for any $\epsilon \in(0,1)$, we have with probability at least $1-\epsilon$:
\begin{equation*}
\sup_{\varphi\in \Phi}\left\vert L_n(\varphi)\right\vert\leq  \sqrt{ \frac{M_1 \log(M_2/\epsilon)}{n}},
\end{equation*}
provided that $n\geq M_1 \log(M_2/\epsilon)$,
where $M_1>0$ and $M_2>1$ are constants depending on $(A,v)$, $K$, $M_{\Phi}$, and $b$ only.
\end{lemma}

We now turn to the term $V_n(\varphi)$. Observe it can be decomposed as $V_n(\varphi)=V_{n,1}(\varphi)+V_{n,2}(\varphi)$, where
\begin{eqnarray*}
V_{n,1}(\varphi)&=& \frac{1}{n} \sum_{i=1} ^n \frac{\delta_i \varphi(\tilde{Y}_i, X_i)}{S_{C,h}(\tilde{Y}_i \mid X_i)} \int_0^{\tilde Y_i}\frac{c_h (u \mid X_i)}{H_h(u,X_i)^{2} }   \left( \hat H_n^{(i)}(u,X_i) -H_h(u,X_i)\right)  d   \hat H_{0,n}^{(i)} (u, X_i) ,\\
V_{n,2}(\varphi)&=& - \frac{1}{n} \sum_{i=1} ^n \frac{\delta_i \varphi(\tilde{Y}_i, X_i)}{S_{C,h}(\tilde{Y}_i \mid X_i)} \int_0^{\tilde Y_i}\frac{c_h (u \mid X_i)}{H_h(u,X_i) }   d \left( \hat H^{(i)}_{0,n} (u, X_i)  - H_{0,h} (u, X_i)\right).
\end{eqnarray*}
We first consider $V_{n,1}(\varphi)$. For simplicity, we set $K_{i,j}=K_h(X_i-X_j)$ for $1\leq i,\; j\leq n$. We have:
\begin{multline*}
V_{n,1}(\varphi)= \\\frac{1}{n(n-1)}\sum_{\scriptsize \begin{array}{c} (i,j) \\
i\neq j \end{array}} \frac{\delta_i  \varphi(\tilde{Y}_i, X_i)}{S_{C,h}(\tilde{Y}_i \mid X_i)} \mathbb{I}\{\tilde{Y}_j \leq \tilde{Y}_i\}\frac{ (1 - \delta_j) K_{ij}  c_h (\tilde{Y}_j  \mid X_i) }{ H_h(\tilde{Y}_j , X_i)^{2}} \left( \hat H_n^{(i)}(\tilde{Y}_j ,X_i) -H_h(\tilde{Y}_j ,X_i) \right)\\
=\frac{1}{n(n-1)^2} \sum_{\scriptsize \begin{array}{c} (i,j,k) \\
i\neq j, i\neq k\end{array}} v_{i,j,k}(\varphi)= V'_{n,1}(\varphi)+V''_{n,1}(\varphi),
\end{multline*}
where, for all $1\leq i,\; j,\; k\leq n$,
$$
v_{i,j,k}(\varphi)=\frac{\delta_i \varphi(\tilde{Y}_i, X_i)}{S_{C,h}(\tilde{Y}_i \mid X_i)} \mathbb{I}\{\tilde{Y}_j \leq \tilde{Y}_i\} \frac{(1 - \delta_j) K_{ij}  c_h (\tilde{Y}_j  \mid X_i)}   {H_h(\tilde{Y}_j , X_i)^{2}} \left( \mathbb{I}\{\tilde{Y}_k > \tilde{Y}_j\}  K_{ik} - H_h(\tilde{Y}_j , X_i) \right)
$$
and
\begin{eqnarray*}
V'_{n,1}(\varphi)&=& \frac{1}{n(n-1)^2} \sum_{\scriptsize \begin{array}{c} (i,j,k) \\
 i\neq j, i\neq k, j\neq k \end{array}} v_{i,j,k}(\varphi) ,\\
 V''_{n,1}(\varphi)&=&\frac{1}{n(n-1)^2} \sum_{\scriptsize \begin{array}{c} (i,j) \\
   i\neq j  \end{array}} v_{i,j,j}(\varphi).
\end{eqnarray*}
The lemma stated below provides a uniform bound for $V''_{n,1}(\varphi)$.

\begin{lemma} \label{lem:V''_term}
Under the assumptions of Proposition \ref{prop:main_result}, we have, for any $\epsilon \in(0,1)$, with probability greater than $1-\epsilon$:
\begin{equation*}
\sup_{\varphi \in \Phi}\left\vert V''_{n,1}(\varphi) \right\vert \leq  M_1  \left(\frac{1}{(n-1) } + \sqrt{\frac{ \log(2 / \epsilon)}{  n}} \frac{ 1}{(n-1) h^d } \right).
\end{equation*}
whenever $ h\leq h_0$ where $h_0>0$ and $M_1 >0$ depends only on $M_{\Phi}$, $K $, $R$ and $b$.
\end{lemma}

We now consider $V'_{n,1}(\varphi)$. Set $Z_k=(X_k,\tilde{Y}_k,\delta_k)$ for $k\in\{1,\; \ldots,\; n  \}$. It can be decomposed as follows:
\begin{equation}
 V'_{n,1}(\varphi)=\frac{n-2}{n-1}\left\{U^{(1)}_{n,1}(\varphi)+U^{(2)}_{n,1}(\varphi)+U^{(3)}_{n,1}(\varphi)+L_n'(\varphi)\right\},
\end{equation}
with
\begin{eqnarray*}
&& U^{(1)}_{n,1}(\varphi)=\frac{1}{n(n-1)(n-2)}\times\\
&& \sum_{\scriptsize \begin{array}{c} (i,j,k) \\
 i\neq j, i\neq k,j\neq k \end{array}} \left\{ v_{i,j,k}(\varphi) - \mathbb{E} [ v_{i,j,k}(\varphi) | Z_j,\;  Z_k ] - \mathbb{E} [ v_{i,j,k}(\varphi) | Z_i,\; Z_k ] + \mathbb{E} [ v_{i,j,k}(\varphi) |  Z_k ] \right\} ,\\
&& U^{(2)}_{n,1}(\varphi)=\frac{1}{n(n-1)}  \sum_{\scriptsize \begin{array}{c} (j,k) \\
  j\neq  k \end{array}} \left\{  \mathbb{E} [ v_{i,j,k}(\varphi)  | Z_j,\; Z_ k ] - \mathbb{E} [ v_{i,j,k}(\varphi)  | Z_k ] \right\} ,\\
&& U^{(3)}_{n,1}(\varphi)= \frac{1}{n(n-1)} \sum_{\scriptsize \begin{array}{c} (i,k) \\
 i\neq k  \end{array}}\left\{ \mathbb{E} [ v_{i,j,k}(\varphi) | Z_i,\; Z_ k ] - \mathbb{E} [ v_{i,j,k}(\varphi) | Z_k ] \right\} ,\\
&& L_n'(\varphi)=\frac{1}{n} \sum_k \mathbb{E} [ v_{i,j,k}(\varphi) | Z_k],\\
\end{eqnarray*}
where $i$, $j$ and $k$ always denote pairwise distinct indexes.
Observe that, for all $\varphi\in \Phi$ and pairwise distinct indexes $i,\; j$ and $k$ in $\{1,\; \ldots,\; n\}$, we have with probability one:
$$
\mathbb{E}[v_{i,j,k}(\varphi)\mid Z_i,\; Z_j  ]=\mathbb{E}[v_{i,j,k}(\varphi)\mid Z_i ]=\mathbb{E}[v_{i,j,k}(\varphi)\mid Z_j  ]=0.
$$
The quantities $U_{n,1}^{(k)}(\varphi)$, $k\in\{1,\; 2,\; 3  \}$ are thus degenerate $U$-statistics of degree $3$, $2$ and $2$ respectively, whereas $L_n'(\varphi)$ is a basic (centred) i.i.d. average. Maximal deviation inequalities for the latter can be obtained by means of classical results in empirical process theory, like for $L_n(\varphi)$.

\begin{lemma}\label{lem:L_2}
Suppose that the hypotheses of Proposition \ref{prop:main_result} are fulfilled. Then, for any $\epsilon\in(0,1)$, we have with probability at least $1-\epsilon$:
\begin{equation*}
 \sup_{\varphi\in \Phi}\lvert L_n'(\varphi)\rvert \leq  \sqrt{ \frac{M_1 \log(M_2/\epsilon )}{ n }  },
\end{equation*}
as soon as $n\geq M_1 \log(M_2/\epsilon)$ and $h\leq h_0$
where $h_0$, $M_1>0$ and $ M_2>1$ are constants depending on $(A,v)$, $K$, $M_{\Phi}$, $R$ and $b$ only.

\end{lemma}
The following result is essentially proved by applying Corollary \ref{lem:concentration_ineq_Ustat}, once the complexity assumptions related to the classes of kernels involved in the definition of these degenerate $U$-processes have been established. It shows that the terms $U_{n,1}^{(k)}(\varphi)$'s are uniformly negligible.

\begin{lemma}\label{lem:Ustat1} Suppose that the hypotheses of Proposition \ref{prop:main_result} are fulfilled. There exist constants $M_1$, $M_2$ and $h_0$ depending on $(A,v)$, $M_\Phi$, $R$, $K$ and $b$ only, such that for any $\epsilon\in(0,1)$, each of the following events holds true with probability at least $1-\epsilon$:
\begin{align}\label{eq:U1}
&\sup_{\varphi\in \Phi}\lvert U_{n,1}^{(1)}(\varphi)\rvert\leq    \left( \frac{M_1    |\log(\epsilon h^{d/2} )| }{ n h^{d}  } \right)^{3/2},  \\
&\sup_{\varphi\in \Phi}\lvert U_{n,1}^{(2)}(\varphi)\rvert\leq  \frac{M_1   |\log(\epsilon h^{d/2} )| }{ n h^{d}  }, \label{eq:U2} \\
&\sup_{\varphi\in \Phi}\lvert U_{n,1}^{(3)}(\varphi)\rvert\leq  \frac{M_1    |\log(\epsilon h^{d/2} )| }{ n h^{d}  } \label{eq:U3},
\end{align}
as soon as $h\leq h_0$  and $M_2   |\log(\epsilon h^{d/2} )|\leq nh^d $.
\end{lemma}

The two preceding lemmas combined with the union bound directly yield the following result.

\begin{corollary} Suppose that the hypotheses of Proposition \ref{prop:main_result} are fulfilled. There exist constants $M_1$, $M_2$, $M_3$, $M_4$, $M_5$ and $h_0$ depending on $(A,v)$, $M_\Phi$, $R$, $K$ and $b$ only such that for any $\epsilon\in (0,1)$, we have with probability greater than $1-\epsilon$:
\begin{equation*}
\sup_{\varphi\in \Phi}\left\vert V'_{n,1}(\varphi) \right\vert \leq M_1 \left( \sqrt{ \frac{\log(M_2/\epsilon )}{ n }  } + \frac{  |\log(\epsilon h^{d/2} )| }{ n h^{d}  } +  \left( \frac{    |\log(\epsilon h^{d/2} )| }{ n h^{d}  } \right)^{3/2}
 \right),
\end{equation*}
as soon as $h\leq h_0$, $M_3   |\log(\epsilon h^{d/2} )|\leq nh^d $ and $M_4 \log(M_5/\epsilon) \leq n $.
\end{corollary}

We next deal with the term $V_{n,2}(\varphi)$.
\begin{lemma}\label{lem:V2} Suppose that the hypotheses of Proposition \ref{prop:main_result} are fulfilled. There exist constants $M_1$, $M_2$, $M_3$, $M_4$, $M_5$ and $h_0$ depending on $(A,v)$, $M_\Phi$, $R$, $K$ and $b$ only such that for any $\epsilon\in (0,1)$, we have with probability greater than $1-\epsilon$:
\begin{equation*}
\sup_{\varphi\in \Phi}\left\vert V_{n,2}(\varphi)\right\vert\leq M_1 \left( \sqrt{ \frac{\log(M_2/\epsilon )}{ n }  } + \frac{  |\log(\epsilon h^{d/2} )| }{ n h^{d}  } \right),
\end{equation*}
as soon as $h\leq h_0$, $M_3   |\log(\epsilon h^{d/2} )|\leq nh^d $ and $M_4 \log(M_5/\epsilon) \leq n $.
\end{lemma}

Finally, we consider the residual $R_n(\varphi)$. Recall first that, for all $\varphi\in \Phi$, we have $R_n (\varphi)= R'_n(\varphi)+R''_n(\varphi)$, where
\begin{eqnarray}
R'_n(\varphi)&=&-\frac{1}{n} \sum_{i=1}^n  \frac{ \delta_i   \varphi(\tilde Y_i , X_i) }{ {S}_{C, h} (\tilde Y_i \mid X_i)}  \hat b_{n}^{(i)} (\tilde Y_i \mid X_i),\label{eq:res1}\\
R''_n(\varphi)&=&   \frac{1}{n} \sum_{i=1}^n  \delta_i   \varphi(\tilde Y_i , X_i)   \frac{\left( {S}_{C, h} (\tilde Y_i \mid X_i) - \hat{S}^{(i)}_{C,n} (\tilde Y_i \mid X_i) \right)^2}{{S}^2_{C, h} (\tilde Y_i  \mid X_i) \hat{S}_{C,n}^{(i)} (\tilde Y _i \mid X_i)}.\label{eq:res2}
\end{eqnarray}
Each of the quantities, $R'_n(\varphi)$ and $R''_n(\varphi)$, is treated separately. We start with $R''_n(\varphi)$.

\begin{lemma}\label{lemma:res''}
Suppose that the assumptions of Proposition \ref{prop:main_result} are satisfied. Then, for all $\epsilon\in (0,1)$, we have with probability greater than $1-\epsilon$
$$
\sup_{\varphi\in \Phi}\left\vert R''_n(\varphi) \right\vert \leq  M_1 \left( \frac{ \vert\log(\epsilon h^{d/2} )  \vert }{nh^d} + \frac{ 1 }{(nh^d)^2}  \right),
$$
as soon as $h\leq h_0$ and $
M_2 |\log(\epsilon h^{d/2})|  \leq {nh^d}$, where $M_1$ and $M_2$ are nonnegative constants depending on $K$, $R$, $M_{\Phi}$ and $b$ only.
\end{lemma}

We now state a uniform bound for $R'_n(\varphi)$.

\begin{lemma}\label{lemma:res'} Suppose that the assumptions of Proposition \ref{prop:main_result} are satisfied. Then, for all $\epsilon\in (0,1)$, we have with probability greater than $1-\epsilon$
$$
\sup_{\varphi\in \Phi}\left\vert R'_{n}(\varphi)\right\vert \leq M_1 \left( \frac{ \vert\log(\epsilon h^{d/2} )  \vert }{nh^d} +  \frac{ \sqrt{|\log(\epsilon  h^{d/2} )|} }{ (nh^d)^{3/2}}+ \frac{ 1 }{nh^d}  + \frac{ 1 }{(nh^d)^2} \right) ,
$$
as soon as $h\leq h_0$ and $
M_2 |\log(\epsilon h^{d/2})|  \leq {nh^d}$, where $M_1$ and $M_2$ are nonnegative constants depending on $K$, $R$, $M_{\Phi}$ and $b$ only.

\end{lemma}

Now we can conclude the proof of Proposition \ref{prop:main_result} by gathering each of the previous results. First note that they all are valid under the condition that  $h\leq h_0$ and $
M_1 |\log(\epsilon h^{d/2})|  \leq {nh^d}$ and $n\geq M_2 \log(M_3/\epsilon)$. By taking $h_0$ small enough, the last requirement is no longer necessary as it ensures that $nh^d>1$ and $ |\log(\epsilon h^{d/2}) |>1$ which implies that $|\log(\epsilon h^{d/2}) | / nh^d \geq { 1 } / {nh^d}  \geq { 1 }/ (nh^d)^2$ and $|\log(\epsilon h^{d/2}) |^{1/2} \leq |\log(\epsilon h^{d/2}) |^{3/2}$. Now choose $M_2\geq 1$, it holds that
\begin{align*}
\frac{ \sqrt{|\log(\epsilon  h^{d/2} )|} }{ (nh^d)^{3/2}}\leq  \left( \frac{ |\log(\epsilon  h^{d/2} )| }{ nh^d}\right)^{3/2} \leq \frac{ |\log(\epsilon  h^{d/2} )| }{ nh^d} .
\end{align*}
 \section{Intermediary Results}
Here we prove lemmas involved in the argument of Proposition \ref{prop:main_result}'s proof.

\subsection{Proof of Lemma \ref{lem:B_term}}
Under the assumptions stipulated, using point (i) of Proposition \ref{prop:decomp} and \eqref{eq:dev}, observe that we have with probability one:
\begin{align*}
\left\vert \delta  \frac{\varphi(\tilde Y, X)}{{S}_{C, h}(\tilde Y \mid X)}  -  \delta  \frac{\varphi(\tilde Y, X) }{{S}_{C}(\tilde Y \mid X)}\right\vert &\leq \frac{2M_{\Phi}}{b^2}\sup_{(y,x)\in \mathcal{K}}\left\vert S_{C,h}(y\mid x)-S_C(y\mid x) \right\vert\\
&\leq {\frac{8M_{\Phi}C_0h^2}{b^6}(1+2/b^3)} .
\end{align*}

\subsection{Proof of Lemma \ref{lem:L_1}}

The proof is a direct application of Corollary \ref{lem:concentration_ineq_Ustat}  to the i.i.d. sequence $\{(X_n,\tilde{Y}_n,\delta_n):\; n\geq 1\}$ and the class of functions (bounded by $2M_{\Phi}/b$)
$$
((x,u),\delta)\in \mathcal{K}\times \{0,1\}\mapsto \frac{\delta\varphi(u, x)}{S_{C,h}(u\mid x)},
$$
indexed by $(\varphi,\; h)\in \Phi\times ]0,h_0]$, the latter being of {\sc VC} type by virtue of classic permanence properties of {\sc VC} type classes of functions (refer to section 2.6.5 in \cite{vandervaart+w:1996}) combined with those recalled in Appendix \ref{AppendixA}. Choosing $\sigma = \|G\|_\infty = 2M_{\Phi}/b$, the bound obtained for $L_n(\varphi)$ is simply
\begin{align*}
&(2M_{\Phi}/b) n^{-1/2} \left(  \left( C_1 ^2\log\left( 2\right) \right)^{1/2}  +   \left( \frac{\log(C_2/\epsilon )}{C_3}\right)^{1/2}  \right)\\
&\leq (4M_{\Phi}/b) n^{-1/2}    \left( C_1 ^2\log\left( 2\right)   +  \frac{\log(C_2/\epsilon )}{C_3}  \right)^{1/2}  ,
\end{align*}
where the constants $C_1,C_2,C_3$ are the ones of Corollary \ref{lem:concentration_ineq_Ustat}. Easy manipulations give the result.

\subsection{Proof of Lemma \ref{lem:V''_term}}
Observe that, for $i\neq j$, we have
$$
v_{i,j,j}(\varphi)=-\frac{\delta_i \varphi(\tilde{Y}_i, X_i)}{S_{C,h}(\tilde{Y}_i \mid X_i)} \mathbb{I}\{\tilde{Y}_j \leq \tilde{Y}_i\} \frac{(1 - \delta_j) K_{ij}  c_n (\tilde{Y}_j  \mid X_i)}   {H_h(\tilde{Y}_j , X_i)},
$$
and recall that, under the assumptions stipulated ($h\leq { h_0}$): $\forall (t,x)\in \mathcal{K}$,
\begin{align}\label{eq:lower_bound1}
H_h(t,x)& \geq { 3 b^3/4},\\
S_{C,h}(t\mid x) &\geq b/2, \\
c_h(t\mid x) &\leq 2/b. \label{eq:lower_bound2}
\end{align}
Considering $\sup_{\varphi\in \Phi}\vert V''_n(\varphi) \vert$ as a function of the $n\geq 1$ independent random pairs $(\tilde{Y}_i,X_i,\delta_i)$ and observing that changing the value of one of them, say $(\tilde{Y}_j,X_j,\delta_j)$ is replaced by $(\tilde{Y}_j',X_j',\delta_j')$, can change its value by at most
\begin{align*}
&\frac{1}{n(n-1)^2} \sum_{i\neq j} | v_{i,j,j} (\varphi)- v_{i,j',j'}(\varphi) | + \sum_{i\neq j} | v_{j',i,i}(\varphi) - v_{j',i,i}(\varphi) |\\
& \leq \frac{1}{n(n-1)} \sup_{1\leq i\leq n } \{ |   v_{i,j,j} (\varphi)- v_{i,j',j'}(\varphi) | + | v_{j',i,i}(\varphi) - v_{j',i,i}(\varphi) | \} \\
& \leq \frac{4}{n(n-1)} \sup_{1\leq i,j\leq n} |v_{i,j,j}(\varphi) | 
\leq \frac{64 M_{\Phi} \lVert K \rVert_\infty}{3 b^5 h^d n (n-1)}.
\end{align*}
The application of the bounded differences inequality, see \cite{mcdiarmid_1989}, yields: $\forall t>0$,
\begin{equation}\label{eq:prob1}
\mathbb{P}\left\{ \left| \sup_{\varphi\in \Phi}\left\vert V''_n(\varphi) \right\vert-\mathbb{E}\left[ \sup_{\varphi\in \Phi}\left\vert V''_n(\varphi) \right\vert \right] >t \right|   \right\}\leq { 2 \exp\left( -\frac{2  t^2 n (n-1)^2 b^{10} h^{2 d}}{(64/3)^2 M_{\Phi}^2 \lVert K \rVert_\infty^2} \right)}.
\end{equation}
In addition, we have:
\begin{equation}\label{eq:exp1}
\mathbb{E}\left[ \sup_{\varphi\in \Phi}\left\vert V''_n(\varphi) \right\vert\right]\leq \frac{1}{(n-1)}\mathbb{E}[ v_{1,2,2}(\varphi)]\leq { \frac{16 M_{\Phi} }{3 (n-1) b^5} \mathbb{E}[K_h(X_1-X_2)]\leq  \frac{16 M_{\Phi} R}{3 (n-1) b^5}}.
\end{equation}
Combining \eqref{eq:prob1} and \eqref{eq:exp1}, we obtain that, for any $\epsilon \in (0,1)$, we have with probability at least $1-\epsilon$:
$$
\sup_{\varphi\in \Phi}\vert V''_n(\varphi) \vert\leq { \frac{16 M_{\Phi} R}{3 (n-1) b^5} + \sqrt{\frac{ \log(2 / \epsilon)}{ 2 n}} \frac{ (64/3)M_{\Phi} \lVert K \rVert_\infty}{(n-1) b^5 h^d} },
$$
and the stated result follows.

\subsection{Proof of Lemma \ref{lem:L_2}}

Let $(i, j, k)\in \{1,\ldots, n\}^3$ be pairwise distinct. Using \eqref{eq:lower_bound1}-\eqref{eq:lower_bound2}, we have that
\begin{align}\label{eq:bound_useful_v}
|v_{i,j,k}|  \leq \left(\frac{64M_\Phi}{9b^8}\right)  K_{ij} (K_{ik} +R).
\end{align}
Based on conditioning arguments, we have
\begin{align*}
 \mathbb E [ v_{i, j, k} (\varphi) |Z_k]  =  A_{k} (\varphi) - \mathbb E [ A_k(\varphi) ],
\end{align*}
where $A_k $ follows directly from the definition of $v_{i, j, k}$. Hence we are in position to apply Corollary \ref{lem:concentration_ineq_Ustat}. Observe that for all $(\varphi, h) \in \Phi \times ]0, h_0]$, using \eqref{eq:lower_bound1}-\eqref{eq:lower_bound2}, we have almost surely,
\begin{align*}
  \lvert A_{k} (\varphi) \rvert  \leq \frac{64 M_{\Phi}}{9 b^8}  \mathbb E [K_{ij}   K_{ik} | Z_k] .
\end{align*}
Because
\begin{align*}
  \mathbb E [K_{ij}   K_{ik} | Z_k] &= \iint K_{h}(x-y) K_{h}(x-X_k) g(x) g(y) dx dy \\
  &= \iint K(z) K_{h}(x-X_k)  g(x) g(x-hz) dx dz \\
  &\leq  R \int K_{h}(x-X_k) g(x) \underbrace{\left(\int K(z) dz \right)}_{1}  dx \leq R^2,
\end{align*}
applying Corollary \ref{lem:concentration_ineq_Ustat} with $ \|G\|_\infty ^2 = \sigma^2 =   ( {64 M_{\Phi}} / {9 b^8} )^2  R^4$ and $k = 1$ yields the bound
\begin{align*}
\sup_{\varphi\in \Phi}\lvert L_n'(\varphi)\rvert\leq \frac{ 64 M_{\Phi} R^2 }{ {9 b^8}    \sqrt n} \left(  C_1 \sqrt{ \log(2  )  } + \sqrt{\log(C_2/\epsilon ) /  C_3}  \right),
\end{align*}
with probability $1-\epsilon$, provided that $ n \geq  C_1^2 \log(2) + \log(C_2/\epsilon)/C_3$. Straightforward calculations then give the desired result.

\subsection{Proof of Lemma \ref{lem:Ustat1}}
The elements of the collection of kernels related to the $U$-process $\{h^{2d}U_{n,1}^{(1)}(\varphi)\}$ are bounded by {$4 \times 64 M_{\Phi}\lVert K \rVert_{\infty} (\lVert K\rVert_{\infty} + h^d R)/( 9 b^8)$} (because of (\ref{eq:bound_useful_v})). The permanence properties recalled in Appendix allows \ref{AppendixA}  establishing that this class of {\sc VC} type with constants depending on $(v,A)$, $K$ and $h_0$ only. Observe also that
\begin{align*}
\mathbb E \left[ h^{4d} v_{i,j,k}(\varphi)^2  \right] &\leq   \left( \frac{64M_{\Phi}}{9 b^{8}} \right)^2   \mathbb E [K((X_1-X_2)/h)^2 (K((X_1-X_3)/h)     +  Rh^{d} )^2] \\
&\leq  2 \times   \left( \frac{64M_{\Phi}}{9 b^{8}} \right)^2   \mathbb E [K((X_1-X_2)/h)^2 (K((X_1-X_3)/h)^2     + R ^2 h^{2d} )] \\
&\leq 2 \times  \left( \frac{64M_{\Phi}}{9 b^{8}} \right)^2 c^2_{K,R} h^{d} \left(c^2_{K,R} h^{d} +R^2 h^{2d}\right)\leq 4 \times  \left( \frac{64M_{\Phi}}{9 b^{8}} \right)^2 c^4_{K,R} h^{2d}   ,
\end{align*}
assuming $R^2 h^{d} \leq c^2_{K,R} $. Since we have a sum of $4$ terms in the $U$-statistics of interest, $h^{2d}U_{n,1}^{(1)}(\varphi)$, each having a $L_2$-norm smaller that $ \mathbb E [ h^{4d} v_{i,j,k}(\varphi)^2  ]$ (by Jensen's inequality), we obtain a bound, for the resulting variance, in $4^2 \times  4 ( {64M_{\Phi}} / {9 b^{8}} )^2 c^4_{K,R} h^{2d} $.
We apply Corollary \ref{lem:concentration_ineq_Ustat} with $k=3$ and a value for $\sigma$ larger than the previous bound. We take $\sigma = 4^2 \times  4 ( {64M_{\Phi}} / {9 b^{8}} )^2 c^4_{K,R} h^{d} h_0^d$ (note that $h\leq h_0$)  and $\lVert G \rVert_\infty = 2 \times 64 M_{\Phi}\lVert K \rVert_{\infty} ^2  /( 9 b^8)$ (assuming that $h^dR\leq \|K\|_\infty$). The conditions are
\begin{equation}
\frac{ \lVert K \rVert_\infty^4}{c^4_{K, R} h_0^d } \left( C_{1}^{2/3} \log\left( \frac{ \lVert K \rVert_\infty^2 }{h^{d/2} c^2_{K, R} } \right) + \frac{\log(C_{2}/\epsilon)}{C_{3}} \right) \leq n h^d \label{lambda1}
\end{equation}
and
\[
 c^4_{K, R}  h^d h_0 ^d \leq  \lVert K \rVert_\infty^4,
\]
where $C_1$, $C_2$ and $C_3$ are the constants in Corollary \ref{lem:concentration_ineq_Ustat}. The latter conditions are indeed of the type $h\leq h_0$ and $nh^d\geq M_2|\log(\epsilon h^{d/2})| $. To recover the stated result, one just needs to multiply the bound (obtained in Corollary  \ref{lem:concentration_ineq_Ustat}) by $1 / (n(n-1)(n-2)h^{2d})$. This gives
\begin{align*}
&\sup_{\varphi\in \Phi}\lvert U_{n,1}^{(1)}(\varphi)\rvert\\
&\leq \frac{\tilde{M}_1}{h^{3d/2} n^{-1/2}(n-1)(n-2)} \times  \left( C_1  \left( \log\left( \frac{\tilde{M}_2}{h^{d/2}}\right) \right)^{3/2} + \left(\frac{\log(C_2/\epsilon)}{C_3}\right)^{3/2} \right),
\end{align*}
where $\tilde{M}_1$ and $\tilde{M}_2$ are constants depending on $M_\Phi$, $R$, $K$, $b$, and $h_0$. Using similar manipulations as the ones presented at the end of the proof of Lemma \ref{lemma:bound_prob}, we obtain the stated result.

By virtue of permanence properties given in Appendix \ref{AppendixA}, the kernels related to the degenerate $U$-process $\{ h^d U_{n,1}^{(2)}(\varphi):\; (\varphi,h)\in \Phi\times ]0,h_0]\}$ form a {\sc VC} class of functions.
Now we derive some bounds on the variance and the uniform norm. Start by noticing that
\begin{align*}
\mathbb{E} \left[ K_{ij} K_{ik} \mid Z_j, Z_k \right] &= \int K_h(x - X_j) K_h(x - X_k) g(x) \mathrm{d}x \\
&\leq R \int K(u) K_h(X_j - X_k + h u) \mathrm{d}u \leq R L_h(X_k - X_j),
\end{align*}
where $L = K \ast K$ and $L_h(u) = L(u / h)/h^d$ (note that $\int L (u) \, du = 1$ and $\|L\|_\infty \leq \|K\|_\infty$). Using (\ref{eq:bound_useful_v}) we obtain that
\begin{align*}
h^d \mathbb{E}[v_{i,j,k}(\varphi)\mid Z_j,\; Z_k] &\leq \left(\frac{64 M_{\Phi}}{9 b^8}\right) h^{d}  R  \left(  L_{jk} +  \mathbb{E} \left[ K_{ij} \mid Z_j \right] \right) \\
& \leq \left(\frac{64 M_{\Phi}}{9 b^8}\right) h^{d}  R  \left(  L_{jk} +  R \right).
\end{align*}
This implies that the terms involved in the $U$-process  $\{ h^d U_{n,1}^{(2)}(\varphi):\; (\varphi,h)\in \Phi\times ]0,h_0]\}$ are bounded by $2 \times 64M_{\Phi}R (\lVert K \rVert_\infty + Rh^d)/(9 b^8)$. We also have
\begin{align*}
\var\left(h^d \mathbb{E}[v_{i,j,k}(\varphi)\mid Z_j,\; Z_k] \right)&\leq \left(\frac{64 M_{\Phi}}{9 b^8}\right)^2 h^{2d} R^2\mathbb{E} \left[ ( L_{jk}  + R )^2 \right]\\
&\leq 2\times \left(\frac{64 M_{\Phi}}{9 b^8}\right)^2 h^{d} R^2 \left( c_{L, R}^2 + h^d R^2 \right)\\
& \leq 4\times \left(\frac{64 M_{\Phi}}{9 b^8}\right)^2 h^{d} R^2  c_{L, R}^2,
\end{align*}
where $c_{L, R} ^2=  R \int L^2(x) \mathrm{d}x $ and using that $ h^d R^2\leq c_{L, R}^2  $. The bound \eqref{eq:U2} is thus obtained by applying Corollary \ref{lem:concentration_ineq_Ustat} with $k=2$, $\sigma^2=4\times 4 \times \left(64 M_{\Phi} / (9 b^8) \right)^2 h^{d} R^2  c_{L, R}^2   $ and $\lVert G \rVert_\infty = 4 \times 64M_{\Phi}R \lVert K \rVert_\infty /(9 b^8)$ (using that $h^dR\leq \|K\|_\infty$). The details are similar to the one given before concerning $  U_{n,1}^{(1)}(\varphi)$.

The kernels related to the degenerate $U$-process $\{ h^d U_{n,1}^{(3)}(\varphi):\; (\varphi,h)\in \Phi\times ]0,h_0]\}$ form a {\sc VC} class of functions. Because of (\ref{eq:bound_useful_v}), it holds that
\begin{align*}
h^d  \mathbb{E}[v_{i,j,k}(\varphi)\mid Z_i,\; Z_k]  &\leq \left(\frac{64 M_{\Phi}}{9 b^8}\right) h^{d}  R  \left( K_{ik} +  R \right) ,
\end{align*}
which gives the uniform bound $2 \times 64M_{\Phi}R ( \lVert K \rVert_\infty + h^d R)/(9 b^8)$ and the variance bound
\begin{align*}
\var\left( h^d \mathbb{E}[v_{i,j,k}(\varphi)\mid Z_i,\; Z_k] \right) &\leq h^{2d} 2^2\mathbb E \left[ \mathbb{E}[v_{i,j,k}(\varphi)\mid Z_i,\; Z_k]^2\right] \\
&\leq 2^2 \times \left(\frac{64 M_{\Phi}}{9 b^8}\right)^2 h^{d} R^2 ( R c^2_{K, R} + h^dR^2 )\\
&\leq 2^3 \times \left(\frac{64 M_{\Phi}}{9 b^8}\right)^2 h^{d} R^2 c^2_{K, R}  ,
\end{align*}
because $R^2 h^{d} \leq c^2_{K,R} $. Setting $k=2$, $\sigma^2= 8  \times \left(64 M_{\Phi} / (9 b^8) \right)^2 h^{d} R^2 c^2_{K, R}$ and $\lVert G \rVert_\infty = 4 \times 64M_{\Phi}R \lVert K \rVert_\infty /(9 b^8)$ (using that $h^dR\leq \|K\|_\infty$) in Corollary \ref{lem:concentration_ineq_Ustat} yields \eqref{eq:U3}.

\subsection{Proof of Lemma \ref{lem:V2}}

 For all $\varphi\in \Phi$, we first set
 $$
 w_{ij}(\varphi) = \frac{\delta_i \varphi(\tilde{Y}_i, X_i)\mathbb{I}\{\tilde{Y}_j \leq \tilde{Y}_i\} (1 - \delta_j)  K_{ij}  c_h(\tilde Y_j  \mid X_i)}{S_{C,h}(\tilde{Y}_i \mid X_i)H_h(\tilde{Y}_j \mid X_i)}
 $$
 and observe next that
\begin{multline*}
  V_{n, 2}(\varphi) =  - \frac{1}{n} \sum_{i=1} ^n \frac{\delta_i \varphi(\tilde{Y}_i, X_i)}{S_{C,h}(\tilde{Y}_i \mid X_i)} \int_0^{\tilde Y_i}\frac{c_h (u \mid X_i)}{H_h(u,X_i) }   d \left( \hat H^{(i)}_{0,n} (u, X_i)  - H_{0,h} (u, X_i)\right)\\
  = -  \frac{1}{n(n-1)}\times\\
   \sum_{
  i\neq j} \frac{\delta_i  \varphi(\tilde{Y}_i, X_i)}{S_{C,h}(\tilde{Y}_i \mid X_i)} \times \left( \frac{\mathbb{I}\{\tilde{Y}_j \leq \tilde{Y}_i\} (1 - \delta_j)  K_{ij}  c_h(\tilde Y_j  \mid X_i)  }{H_h(\tilde{Y}_j \mid X_i)} -  \int_0^  {\tilde{Y}_i} \frac{c_h(u  \mid X_i) }{H_h(u,X_i)} d H_{0,h} (u, X_i) \right) \\
  = -  \frac{1}{n(n-1)} \sum_{
  i\neq j} \left\{ w_{ij}(\varphi)  - \mathbb E[  w_{ij}(\varphi)\mid Z_j] \right\}
  = U^{(1)}_{n, 2} (\varphi) + U^{(2)}_{n , 2}(\varphi) ,
\end{multline*}
where
 \begin{align}
  U^{(1)}_{n, 2}(\varphi)  &= -  \frac{1}{n(n-1)} \sum_{
  i\neq j}  \left\{ w_{ij}(\varphi)  - \mathbb E[  w_{ij}(\varphi) \mid Z_j]  - \mathbb{E} [ w_{ij}(\varphi)  \mid Z_i ] + \mathbb{E} [ w_{12}(\varphi)  ]  \right\},\label{eq:U21}\\
  U^{(2)}_{n, 2} (\varphi) &= \frac{1}{n} \sum_{i=1}^n  \left\{\mathbb{E} [ w_{ij}(\varphi)  \mid Z_i ] - \mathbb{E} [ w_{12}(\varphi)   ]\right\}.\label{eq:U22}
\end{align}
Hence, $V_{n, 2}(\varphi)$ can be decomposed as the sum of a degenerate $U$-statistic \eqref{eq:U21} and an i.i.d. average \eqref{eq:U22}. Note also that, by (\ref{eq:bound_useful_v}), we have
\begin{align*}
|w_{ij}(\varphi)| \leq \frac{16M_{\Phi} }{3 b^5} K_{ij}.
\end{align*}

Using the same arguments as in the proof of Lemma \ref{lem:Ustat1}, the kernels of the collection of degenerate $U$-statistics $\{  h^dU^{(1)}_{n, 2}(\varphi): \; (\varphi,h)\in\Phi\times ]0,h_0] \}$ form a class of {\sc VC} type with constants depending only on $(v,A)$, $K$ and $h_0$. In addition, these terms are all bounded by $4 \times  {16 M_\Phi}\lVert K \rVert_\infty /{3 b^5} $ and we have:
$$
\var\left( h^d w_{ij}(\varphi)  \right)\leq \left( \frac{16 M_\Phi}{3 b^5} \right)^2 h^d c^2_{K, R}.
$$
It thus results from the application of Corollary \ref{lem:concentration_ineq_Ustat} with $k=2$ and $\sigma^2=4^2 \times \left( 16 M_\Phi / (3 b^5) \right)^2 h^d c^2_{K, R}$ that, with probability greater than $1-\epsilon$
\begin{equation}\label{eq:bound_U21}
\sup_{\varphi\in \Phi}\left\vert U_{n,1}^{(2)}(\varphi) \right\vert \leq \frac{\tilde{M}_1}{h^d (n-1)} \times  \left( C_1  \log\left( \frac{\tilde{M}_2}{h^{d/2}} \right) + \frac{\log(C_2/\epsilon)}{C_3} \right),
\end{equation}
where $\tilde{M}_1$ and $\tilde{M}_2$ depends on $ M_\Phi$, $ K$, $R$, $b$, provided a condition of the type $ h\leq h_0$ and $nh^d \leq M_2 |\log(\epsilon h^{d/2})|$. We next deal with the uniform control of $U_{n,2}^{(2)}(\varphi)$. Observe that, with probability one,
$$
\mathbb{E} [ w_{ij}(\varphi)  \mid Z_i ] \leq \frac{16M_{\Phi} R}{3 b^5},
$$
and that
$$
\var\left( \mathbb{E} [ w_{ij}(\varphi)  \mid Z_i ] \right)\leq \left( \frac{16M_{\Phi} R}{3 b^5} \right)^2.
$$
 Applying thus Corollary  \ref{lem:concentration_ineq_Ustat} with $k = 1$, $\sigma^2 = 4 \times \left( 16M_{\Phi} R / (3 b^5) \right)^2$ and $\lVert G \rVert_\infty = 2 \times 16M_{\Phi} R / (3 b^5)$, we obtain that, with probability $1 - \epsilon$,
\begin{equation}\label{eq:bound_U22}
\sup_{\varphi\in \Phi}\left\vert U_{n,2}^{(2)}(\varphi) \right\vert \leq \frac{C}{ \sqrt{n} }  \left( C_{1}  \sqrt{ \log( 2 )} + \sqrt{\frac{\log(C_{2}/\epsilon)}{C_{3}} }\right),
\end{equation}
where $C$ depends on $ M_\Phi$, $ K$, $R$, $b$, provided a condition of the type $n \geq M_4 |\log(M_5 / \epsilon )|$ holds true.
The bound stated in the lemma results from rearranging the bounds \eqref{eq:bound_U21} and \eqref{eq:bound_U22}.

\subsection{Proof of Lemma \ref{lemma:res''}}

Observe that, on the event $\mathcal{E}_n$: $\forall \varphi\in \Phi$,
$$
\left\vert R''_n(\varphi) \right\vert\leq \frac{8M_{\Phi}}{b^3}\sup_{(t,x)\in \mathcal{K}}\left| {S}_{C, h} (t \mid x) - \hat{S}^{(i)}_{C,n} (t \mid x) \right|^2.
$$
Then, the event
$$\mathcal{B}_n^{(3)} \overset{def}{=}\bigcap_{i\leq n}\left\{\forall (t,x)\in \mathcal{K},\;\; \hat H^{(i)}_n(t,x)\geq b^3 / 2,\, H_h(t,x) \geq b^3/2 , S_{C,h}(t|x ) > b/2 \right\},$$
occurs with probability at least $1-\epsilon/2$ under the mentioned condition on $(n,h)$. Suppose that $ \mathcal{B}_n^{(3)}  $ is realized.  In a similar fashion as in the proof of Proposition \ref{prop:bound_CKM} (see Eq. \eqref{eq:1},\eqref{eq:2},\eqref{eq:3}), we apply Lemma \ref{lemma:survival} to get that
\begin{align*}
&\sup_{(t,x)\in \mathcal{K}} | \hat S^{(i)}_{C,n} (t | x)- S_{C,h} (t | x) |  \leq ( 4 / b ) \sup_{(t,x)\in \mathcal{K}} | \hat \Lambda_{C,n}^{(i)} (t|x) - \Lambda _{C,h} (t|x)  |  .
\end{align*}
Then, we apply Lemma \ref{lemma:lambda},  with $\theta_1 = b^3/2$, $\theta_2 = b^3/2$, $\beta = 1$, to finally obtain that: $\forall i\in\{1,\; \ldots,\; n  \}$,
\begin{align*}
&\sup_{(t,x)\in \mathcal{K}} | \hat S^{(i)}_{C,n} (t | x)- S_{C,h} (t | x) |  \\
&\leq \frac 4  b   \left( {\frac{4}{b^3}} \sup_{(t,x)\in \mathcal{K}}|\hat H^{(i)}_{0,n}(t , x)- H_{0,h}(t,  x)| +  {\frac{4}{b^6}} \sup_{(t,x)\in \mathcal{K}}|\hat H^{(i)}_{n}(t , x) - H_{h}(t , x) |\right).
\end{align*}
Hence using the triangle inequality and (\ref{eq:4}) and (\ref{eq:5}) with Lemma \ref{lemma:bound_prob}, we obtain that, with probability $1-\epsilon$:
\begin{align*}
\sup_{(t,x)\in \mathcal{K}} | \hat S^{(i)}_{C,n} (t | x)- S_{C,h} (t | x) |  \leq  M_1 \left( \frac{1}{nh^d} + \sqrt{ \frac{ |\log(\epsilon h^{d/2} )|  }{ {nh^d} } }\right),
\end{align*}
provided that $h\leq h_0$ and $nh^d \geq M_2  |\log(\epsilon h^{d/2} )| $.

\subsection{Proof of Lemma \ref{lemma:res'}}

Recall first that
$$
R'_n(\varphi)=-\frac{1}{n} \sum_{i=1}^n  \frac{ \delta_i   \varphi(\tilde Y_i , X_i) }{ {S}_{C, h} (\tilde Y_i \mid X_i)}  \hat b_{n}^{(i)} (\tilde Y_i \mid X_i),
$$
where
\begin{multline*}
\hat b_{n}^{(i)}(t\mid x) =  -\int_0^t  \frac{  c_{h} (u\mid x)   }{  H_h(u,x)^{2}\hat H^{(i)}_n(u,x)} ( \hat H_n^{(i)}(u,x) -H_h(u,x) )^2   d   \hat H^{(i)}_{0,n} (u, x) \\
  -  \int_0^t ( \hat S_{C,n}^{(i)} (u - \mid x) -  S_{C,h} (u - \mid x) ) d  \hat  \Delta^{(i)}_{n} (u\mid x)  .
\end{multline*}
and
\begin{equation*}
 d  \hat  \Delta_{n}^{(i)} (u \mid x)
=   \hat{\Lambda}^{(i)}_{C,n}(u\mid x)-\Lambda_{C,h}(u\mid x).
\end{equation*}
The argument is based on Lemma \ref{lem:res_tool}, stated in section \ref{sec:auxiliary_results}. Note that, on the event $\mathcal{B}_n$, we have:
\begin{multline*}
\left\vert \hat b_{n}^{(i)}(t\mid x) \right\vert\leq \frac{8}{b^5} \int \left( \hat H_n^{(i)}(u,x) -H_h(u,x) \right)^2   d   \hat H_n^{(i)}   (u,x)   \\
+  \left|  \int_0^t \left(\hat S_{C,n}^{(i)} (u - \mid x) -  S_{C,h} (u - \mid x) \right)  d  \hat  \Delta_{n}^{(i)} (u\mid x)\right|\\
\leq (8/b^5)\sup_{(u,x)\in\Gamma_b}\left\vert  \hat H_n^{(i)}(u,x) -H_h(u,x)   \right\vert^2\\
+ \left\vert  \int_0^t \frac{\left(\hat S_{C,n}^{(i)} (u - \mid x) -  S_{C,h} (u - \mid x)\right)}{ S_{C,h} (u  \mid x) }  d  \hat  \Delta_{n}^{(i)} (u\mid x)\right\vert.
\end{multline*}
The application of the Lemma \ref{lem:res_tool},  with $S^{(2)}(u)=S_{C,h}(u\mid x)$, $S^{(1)}(u)=\hat S_{C,n}^{(i)} (u \mid x) $, $\beta=1$,  $\theta=b/2$ and
\begin{eqnarray*}
\Lambda^{(1)}(u)&=&\hat{\Lambda}_{C,n}^{(i)}(u\mid x)=-\int_{s=0}^u \frac{\hat{H}^{(i)}_{0,n}(ds,x)}{\hat{H}^{(i)}_n(s-,x)},\\
\Lambda^{(2)}(u)&=& \Lambda_{C,h}(u\mid x)= -\int_{s=0}^u \frac{H_{0,h}(ds,x)}{H_h(s-,x)},
\end{eqnarray*}
 yields,
\begin{multline*}
\left\vert  \int_0^t \frac{\left(\hat S_{C,n}^{(i)} (u - \mid x) -  S_{C,h} (u - \mid x)\right)}{ S_{C,h} (u  \mid x) }  d  \hat  \Delta_{n}^{(i)} (u\mid x)\right\vert \leq \\
C\left(\sup_{(u,x)\in\Gamma_b}\left\vert \hat H^{(i)}_{n}(u,x) - H_h(u,x) \right\vert^2 +\sup_{(u,x)\in\Gamma_b}\left\vert \hat{H}^{(i)}_{0,n}(u,x) - H_{0,h}(u,x) \right\vert^2+\sup_{(u,x)\in\Gamma_b}\left\vert \hat W^{(i)}_n(u,x) \right\vert \right),
\end{multline*}
where $C>0$ depends on $b$ and
\begin{align*}
 \hat W^{(i)}_n(t,x) =  \int_0^t \int_0^u c_h(s\mid x)  \frac{ d \left( \hat H^{(i)}_{0,n} (s, x)  - H_{0,h} (s, x)\right)}{H_h(s,x)}  \frac{d \left(\hat H^{(i)}_{0,n}(u,x)  -H_{0,h}(u,x) \right)}{ S_{C,h}(u,x)  H_{h}(u,x) } .
\end{align*}
Using \eqref{eq:4} and \eqref{eq:5} combined with Lemma \ref{lemma:bound_prob}, we obtain that with probability at least $1-\epsilon$:
\begin{align*}
\sup_{(u,x)\in\Gamma_b}\left\vert \hat H^{(i)}_{n}(u,x) - H_h(u,x) \right\vert^2 +\sup_{(u,x)\in\Gamma_b}\left\vert \hat{H}^{(i)}_{0,n}(u,x) - H_{0,h}(u,x) \right\vert^2\\
  \leq  M_1 \left( \frac{1}{(n h^d)^2} + \frac{ |\log( \epsilon h^{d/2})|  } {nh^d}  \right),
\end{align*}
as soon as $ h\leq h_0$ and $M_2 |\log(\epsilon h^{d/2})| \leq nh^d$. It remains to show that, with probability at least $1-\epsilon$:
\begin{equation*}
\max_{i\in\{1,\; \ldots,\; n \}}\sup_{(u,x)\in\Gamma_b}\left\vert \hat W^{(i)}_n(u,x) \right\vert  \leq M_1 \left(  \frac{|\log( h^{d/2} \epsilon)| } {nh^d}  \right).
\end{equation*}
We first set, for all $(t,x)\in \mathcal{K}$,
$$
\hat{W}_n(t,x)=  \int_0^t \int_0^u c_h(s\mid x)  \frac{ d \left( \hat H_{0,n} (s, x)  - H_{0,h} (s, x)\right)}{H_h(s,x)}  \frac{d \left(\hat H_{0,n}(u,x)  -H_{0,h}(u,x) \right)}{ S_{C,h}(u,x)  H_{h}(u,x) }
$$
and notice that, since $c_h(s\mid x)/(H_h(s,x)S_{C,h}(u,x)H_h(u,x))\leq 64 / 9b^8$ (using \eqref{eq:lower_bound1}-\eqref{eq:lower_bound2}), we have by virtue of \eqref{eq:4}
$$
\max_{i\in\{1,\; \ldots,\; n \}}\sup_{(t,x)\in \mathcal{K}}\left\vert \hat{W}_n(t,x)-\hat{W}^{(i)}_n(t,x) \right\vert \leq C/(nh^d),
$$
where $C$ is a constant depending on $b$ and $K$ only. In addition, observe that
\begin{equation}\label{eq:decW}
\hat W_n(t,x) =  n^{-2} \sum_{i=1} ^n \sum_{j=1} ^n v_{ij}(t,x)=n^{-2} \sum_{i\neq j}  v_{ij}(t,x) + n^{-2} \sum_{i=1}^n  v_{ii}(t,x):=U_n(t,x)+M_n(t,x),
\end{equation}
where, for all $1\leq i,\; j\leq n$, we set
\begin{align*}
& v_{ij}(t,x) =  u_{ij}(t,x) - \mathbb{E}[u_{ij} (t,x) |Z_i] - \mathbb{E}[u_{ij}(t,x) | Z_j] + \mathbb{E}[u_{1,2}(t,x)],\\
& u_{ij}(t,x) =  \xi_{i,j}(x) \mathbb I \{\tilde Y_i \leq t\}    K_h(X_i-x) K_h(X_j-x),\\
&\xi_{i,j}(x) = \frac{\delta_i \delta_j  c_h(\tilde Y_j \mid x) }{ S_{C,h}(\tilde Y_i,x)  H_{h}(\tilde Y_i,x) H_h(\tilde Y_j ,x) }\mathbb I \{\tilde Y_j \leq \tilde Y_i\}.
\end{align*}
Because we have, for all $(t,x)\in \mathcal{K}$,
\begin{align*}
\mathbb{E}  [v_{12}(t,x)|Z_1]= \mathbb{E} [v_{12}(t,x)|Z_2] = 0 ,
\end{align*}
the collection of random variables $\{n^{2} h^{2d}U_n(t,x):\;\; (t,x, h)\in \mathcal{K}\times ]0,h_0]\}$ is a degenerate $U$-process of order $2$. The related class of kernels is uniformly bounded by $4 \times 64\vert\vert K\vert\vert_{\infty}^2/ (9b^8)$ (\textit{cf} \eqref{eq:lower_bound1}-\eqref{eq:lower_bound2}) and of {\sc VC} type, by virtue of classic permanence properties (\textit{i.e.} a class formed of products of two elements of functions in two given bounded {\sc VC} classes is still of {\sc VC} type, see \textit{e.g.} Chapter 2 in \cite{vandervaart+w:1996}) combined with the results recalled in Appendix \ref{AppendixA}. Observe in addition that
$$
\var\left( h^{2d} v_{12}(t,x) \right) \leq h^{4d}  4 ^2\mathbb{E}[u^2_{12}(t,x)] \leq h^{4d}  4 ^2  \left(\frac{64}{9b^2 }\right)^2  E[K_{1x}^2K_{2x}^2 ]  \leq   h^{2d} 4 ^2  \left(\frac{64}{9b^2 }\right)^2 c^4_{K,R} .
$$
Applying Corollary \ref{lem:concentration_ineq_Ustat} with $k=2$, $\sigma^2=h^{2d} 4 ^2 ({64} / {9b^2 }) c^4_{K,R} $ and $ \|G\|_\infty = 4 \times 64\vert\vert K\vert\vert_{\infty}^2/ (9b^8)$, we obtain that, with probability greater than $1-\epsilon$,
\begin{equation}\label{eq:boundU}
\sup_{(t,x)\in \mathcal{K}}\left\vert U_n(t,x) \right\vert\leq M_1 \frac{ |\log(\epsilon h^{d/2})|  }{nh^d},
\end{equation}
as soon as
$
M_2  |\log( \epsilon h^{d/2})|  \leq {nh^d}
$ and $ h  \leq h_0$. Notice now that, for all $(t,x)\in\mathcal{K}$, $M_n(t,x) =L_n(t,x)+R_n(t,x)$, where
\begin{eqnarray*}
L_n(t,x) &=& n^{-2} \sum_{i=1}^n  \left\{v_{ii}(t,x) - \mathbb{E}[v_{11}(t,x) ]\right\},\\
R_n(t,x)&=&  n^{-1}\mathbb{E}[v_{11}(t,x) ].
\end{eqnarray*}
Observing that, for all $(t,x)\in\mathcal{K}$,  $\vert h^{2d} v_{11} (t,x) \vert\leq 4 \times  64 \vert\vert K\vert\vert_{\infty}^2/ (9b^8)$ and
$$
\var\left(h^{2d} v_{11} (t,x)\right)\leq h^{4d} 4^2 \mathbb{E}[u^2_{11} (t,x) ] \leq h^{4d} 4^2  \left(\frac{64}{9b^2 }\right)^2 \mathbb E [ K_{1x}^4]\leq h^{d} 4^2  \left(\frac{64}{9b^2 }\right)^2 R \int K^4(x) dx.
$$
Hence, the application of Corollary \ref{lem:concentration_ineq_Ustat} with $k=1$ to the empirical sums $\{n^2 h^{2d} L_n(t,x) \, : \, (t,x, h)\in \mathcal{K}\times ]0,h_0]\}$ permits to get that, with probability at least $1-\epsilon$,
\begin{equation}\label{eq:boundL}
\sup_{(t,x)\in \mathcal{K}}\left\vert L_n(t,x) \right\vert\leq \tilde \tilde{M}_1 \frac{ \sqrt{|\log(\tilde{M}_2 / h^{d/2} )|} + \sqrt{\log(C_{2}/\epsilon) / C_3} }{ (nh^d)^{3/2}},
\end{equation}
where $\tilde{M}_1$ and $\tilde{M}_2$ are constants depending on $K$, $b$ and $R$. The previous bound is valid whenever $
M_2 |\log(\epsilon h^{d/2})|  \leq {nh^d}$ and $ h  \leq h_0$.  We also have $n^{-1}\mathbb{E}[\vert u_{11}(t,x)\vert ] \leq (16/b^6)c^2_{K,R}/(nh^d)$. This leads to the stated results.

\clearpage
\section{Numerical Results \label{appendix:results}}
\begin{figure}[htbp]
  \centering
  \begin{subfigure}[b]{0.9\textwidth}
      \centering
        \includegraphics[width=\textwidth]{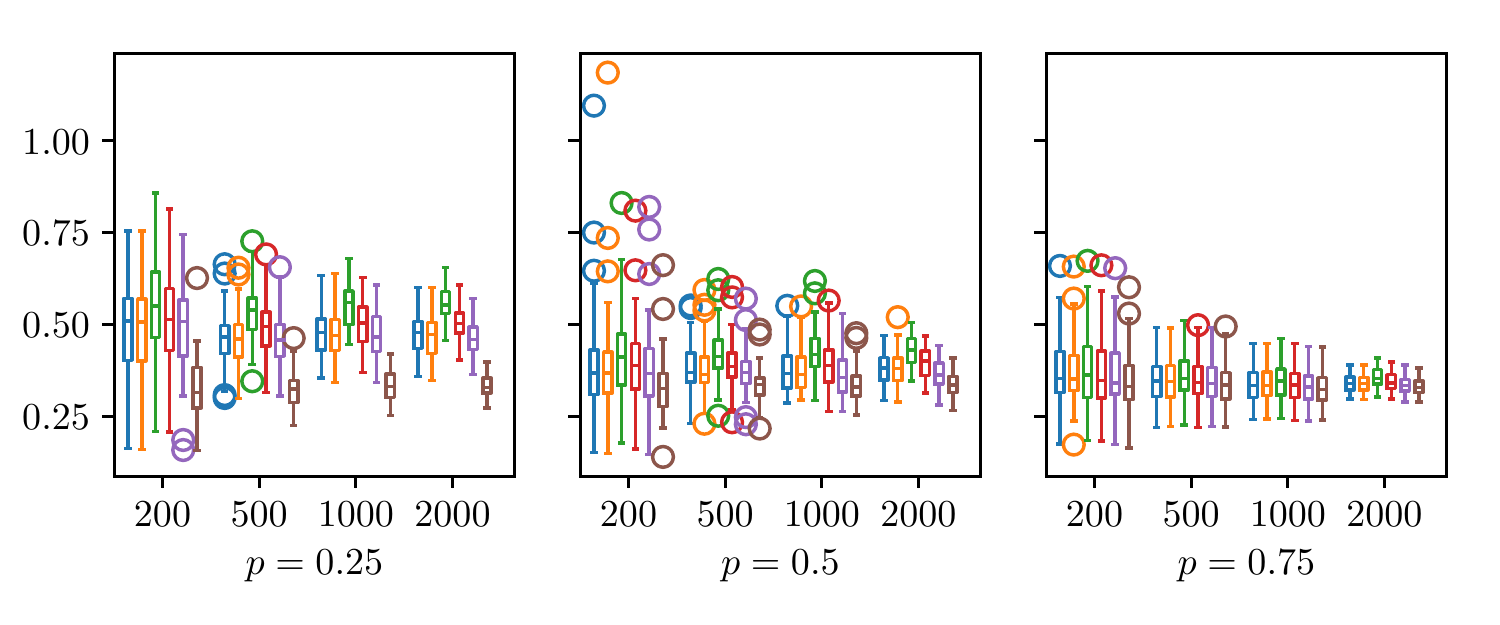}
        \caption{Linear Regression}
  \end{subfigure}
  \begin{subfigure}[b]{0.9\textwidth}
      \centering
        \includegraphics[width=\textwidth]{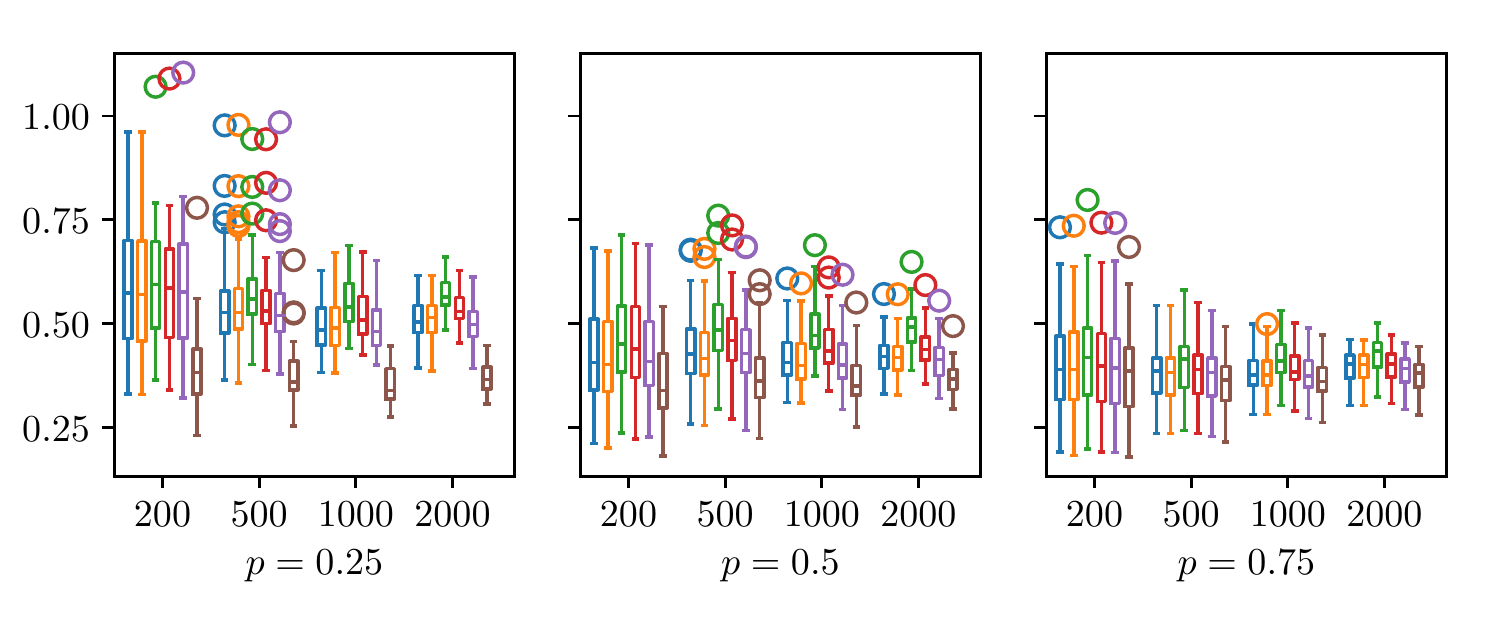}
        \caption{SVR}
  \end{subfigure}
  \begin{subfigure}[b]{0.9\textwidth}
      \centering
        \includegraphics[width=\textwidth]{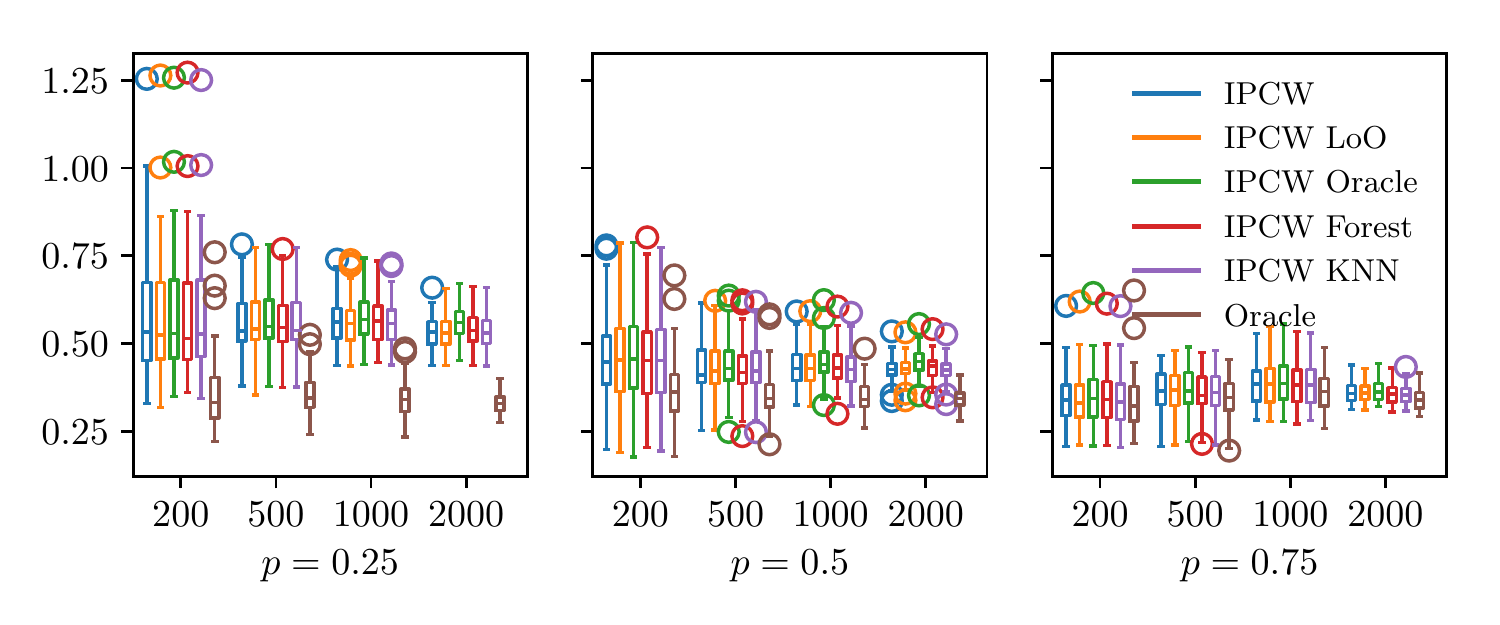}
        \caption{Random Forest}
  \end{subfigure}
  \caption{$L^2$ error of the different IPCW estimators for $d=2$}
\end{figure}

\begin{figure}[htbp]
  \centering
  \begin{subfigure}[b]{0.9\textwidth}
      \centering
        \includegraphics[width=\textwidth]{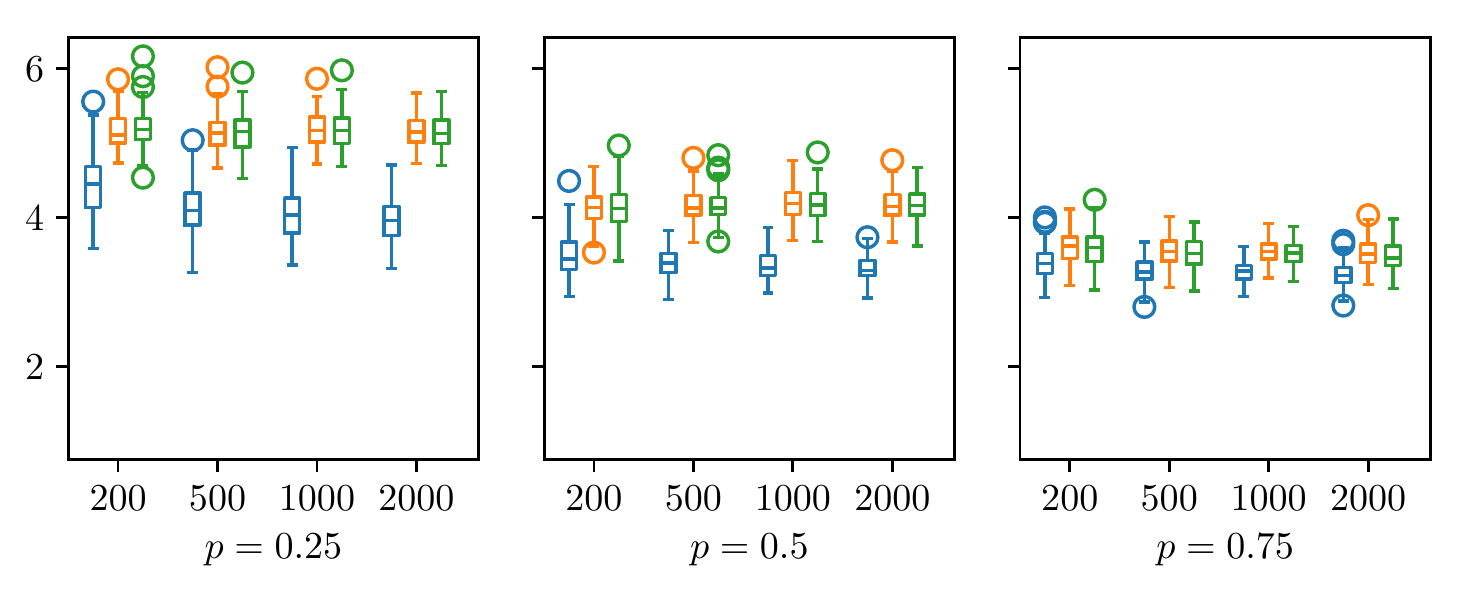}
        \caption{Linear Regression}
  \end{subfigure}
  \begin{subfigure}[b]{0.9\textwidth}
      \centering
        \includegraphics[width=\textwidth]{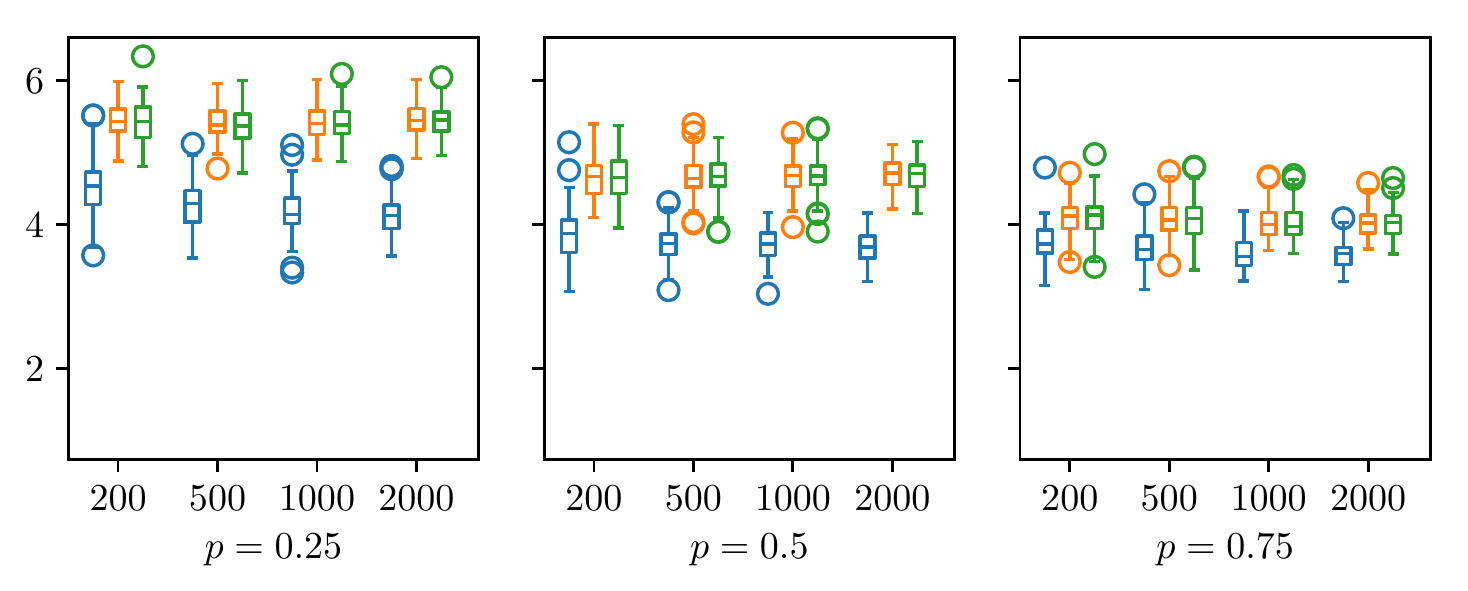}
        \caption{SVR}
  \end{subfigure}
  \begin{subfigure}[b]{0.9\textwidth}
      \centering
        \includegraphics[width=\textwidth]{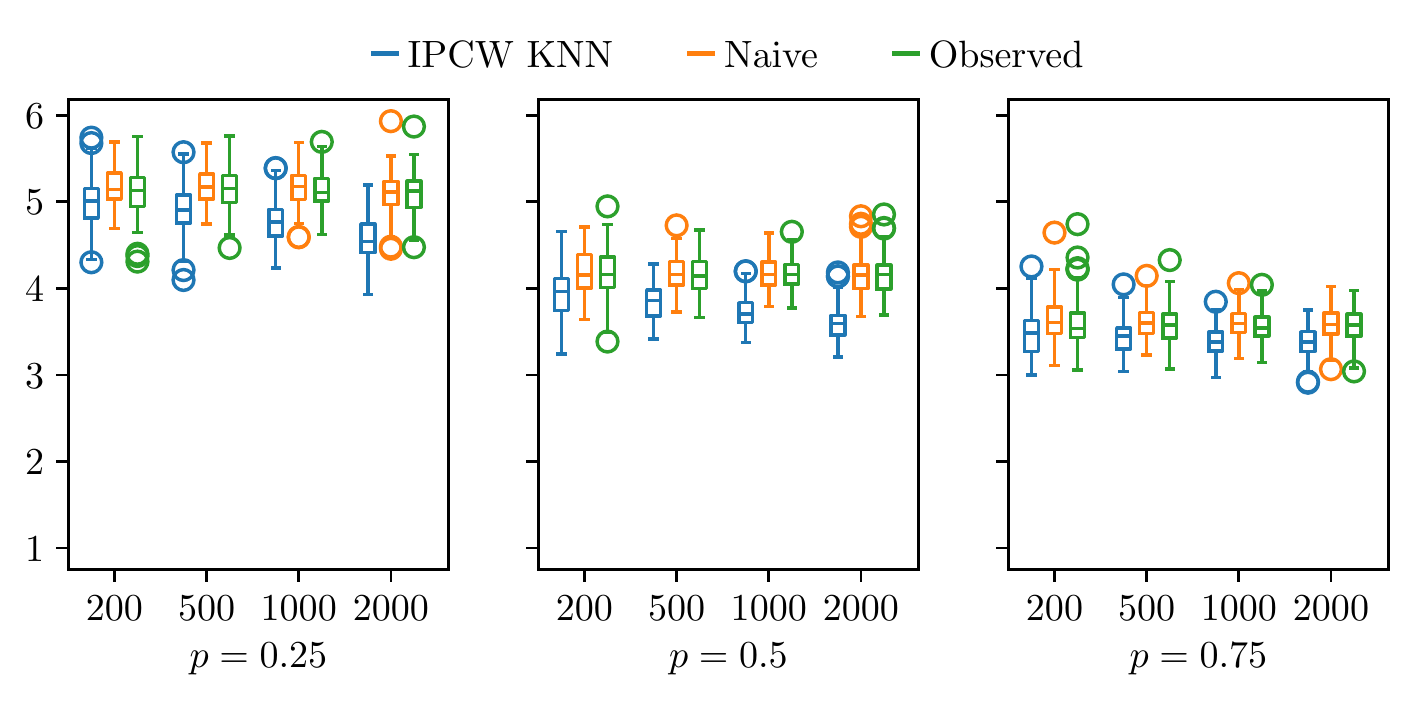}
        \caption{Random Forest}
  \end{subfigure}
  \caption{$L^2$ error of the IPCW estimator compared to naive methods for $d=2$}
\end{figure}

\begin{figure}[htbp]
  \centering
  \begin{subfigure}[b]{0.9\textwidth}
      \centering
        \includegraphics[width=\textwidth]{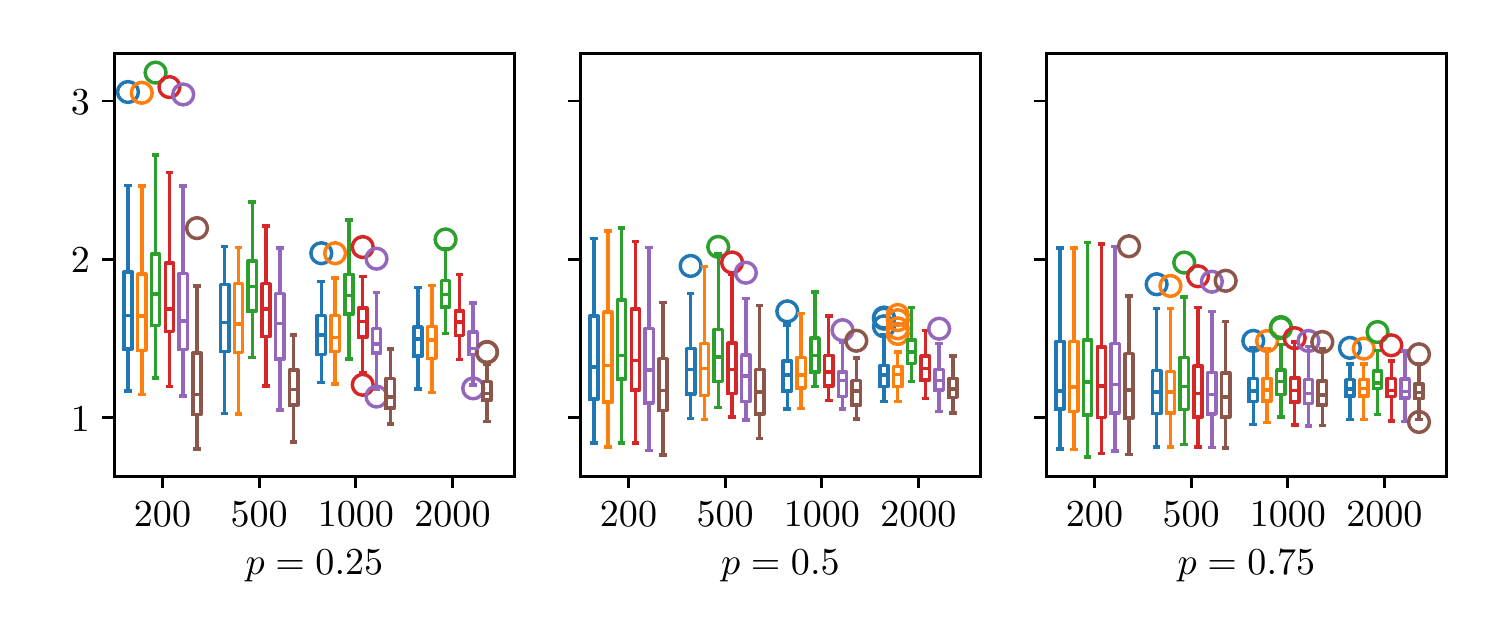}
        \caption{Linear Regression}
  \end{subfigure}
  \begin{subfigure}[b]{0.9\textwidth}
      \centering
        \includegraphics[width=\textwidth]{figs/error_d4_svr}
        \caption{SVR}
  \end{subfigure}
  \begin{subfigure}[b]{0.9\textwidth}
      \centering
        \includegraphics[width=\textwidth]{figs/error_d4_rf}
        \caption{Random Forest}
  \end{subfigure}
  \caption{$L^2$ error of the different IPCW estimators for $d=4$}
\end{figure}

\begin{figure}[htbp]
  \centering
  \begin{subfigure}[b]{0.9\textwidth}
      \centering
        \includegraphics[width=\textwidth]{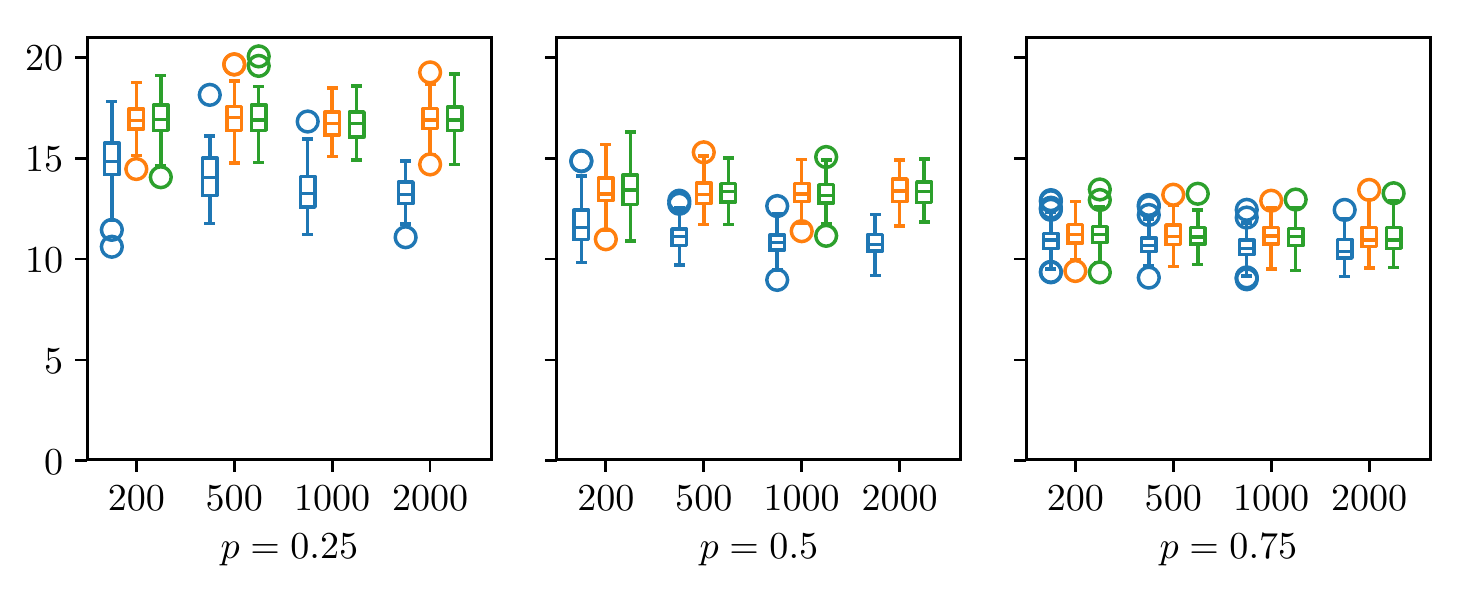}
        \caption{Linear Regression}
  \end{subfigure}
  \begin{subfigure}[b]{0.9\textwidth}
      \centering
        \includegraphics[width=\textwidth]{figs/error_vs_d4_svr}
        \caption{SVR}
  \end{subfigure}
  \begin{subfigure}[b]{0.9\textwidth}
      \centering
        \includegraphics[width=\textwidth]{figs/error_vs_d4_rf}
        \caption{Random Forest}
  \end{subfigure}
  \caption{$L^2$ error of the IPCW estimator compared to naive methods for $d=4$}
\end{figure}

\begin{figure}[htbp]
  \centering
  \begin{subfigure}[b]{0.9\textwidth}
      \centering
        \includegraphics[width=\textwidth]{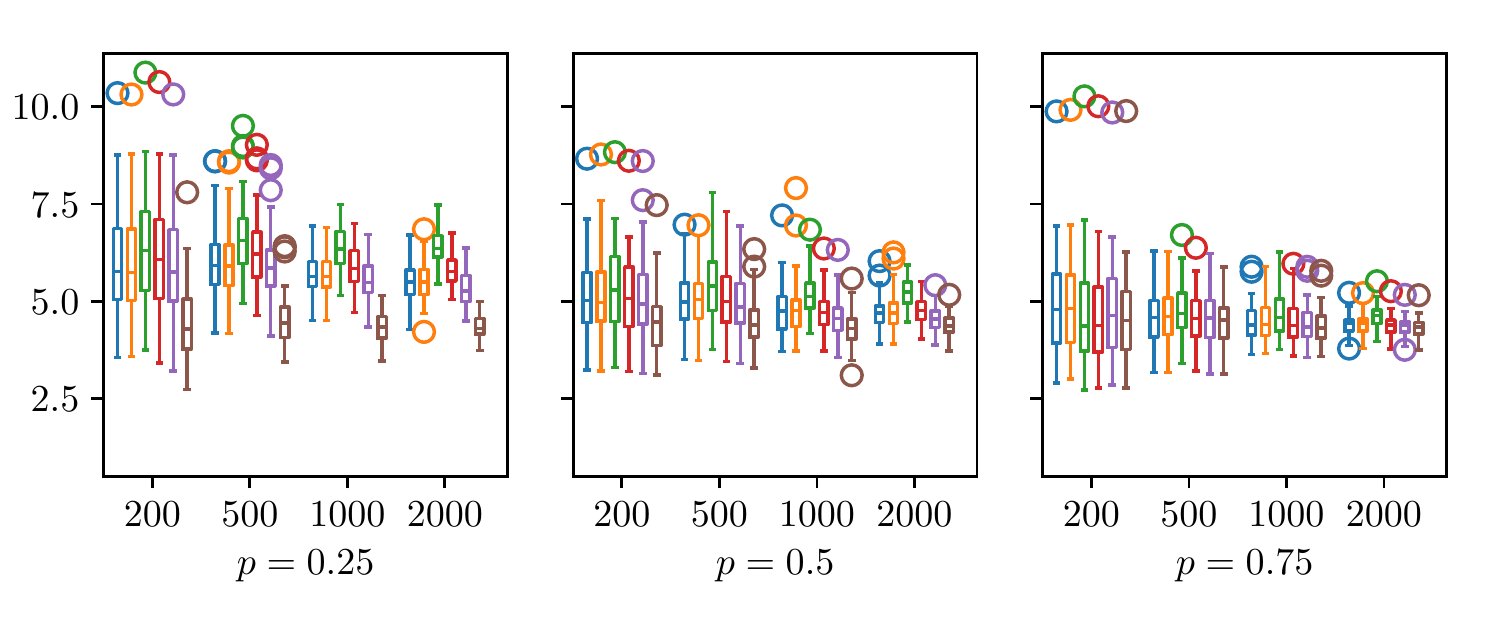}
        \caption{Linear Regression}
  \end{subfigure}
  \begin{subfigure}[b]{0.9\textwidth}
      \centering
        \includegraphics[width=\textwidth]{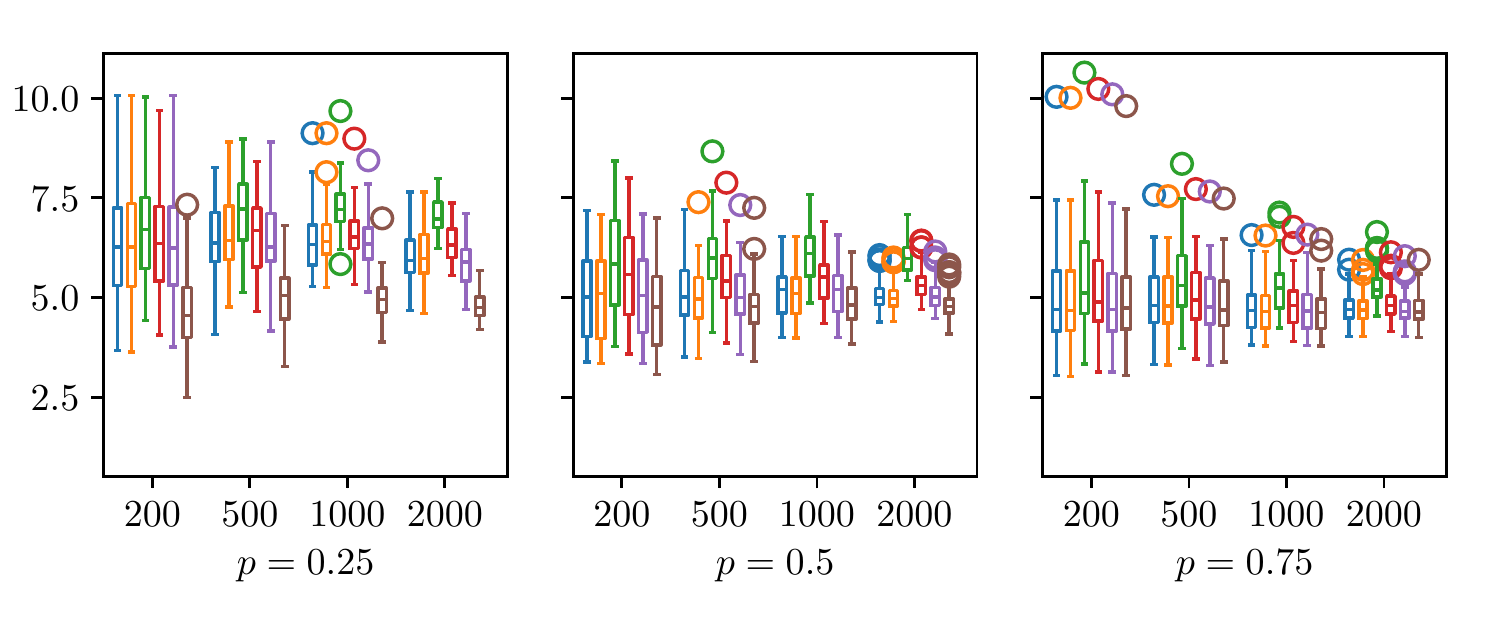}
        \caption{SVR}
  \end{subfigure}
  \begin{subfigure}[b]{0.9\textwidth}
      \centering
        \includegraphics[width=\textwidth]{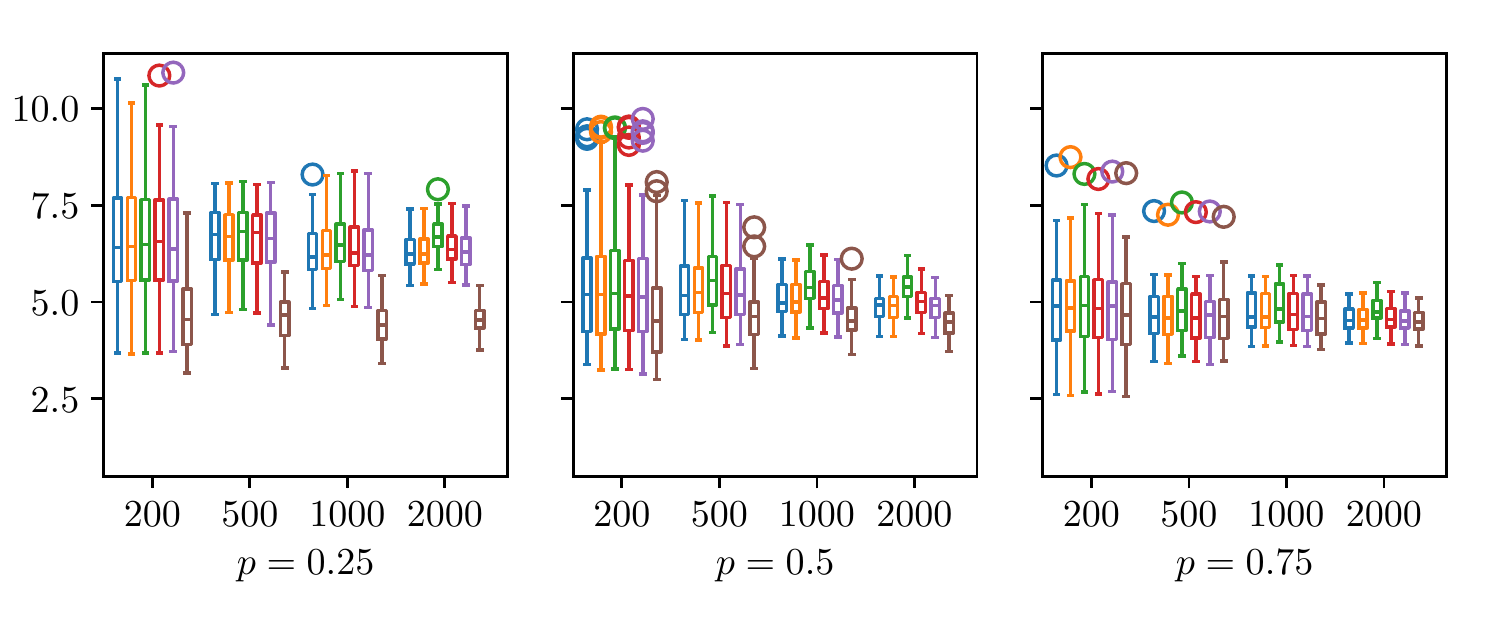}
        \caption{Random Forest}
  \end{subfigure}
  \caption{$L^2$ error of the different IPCW estimators for $d=8$}
\end{figure}

\begin{figure}[htbp]
  \centering
  \begin{subfigure}[b]{0.9\textwidth}
      \centering
        \includegraphics[width=\textwidth]{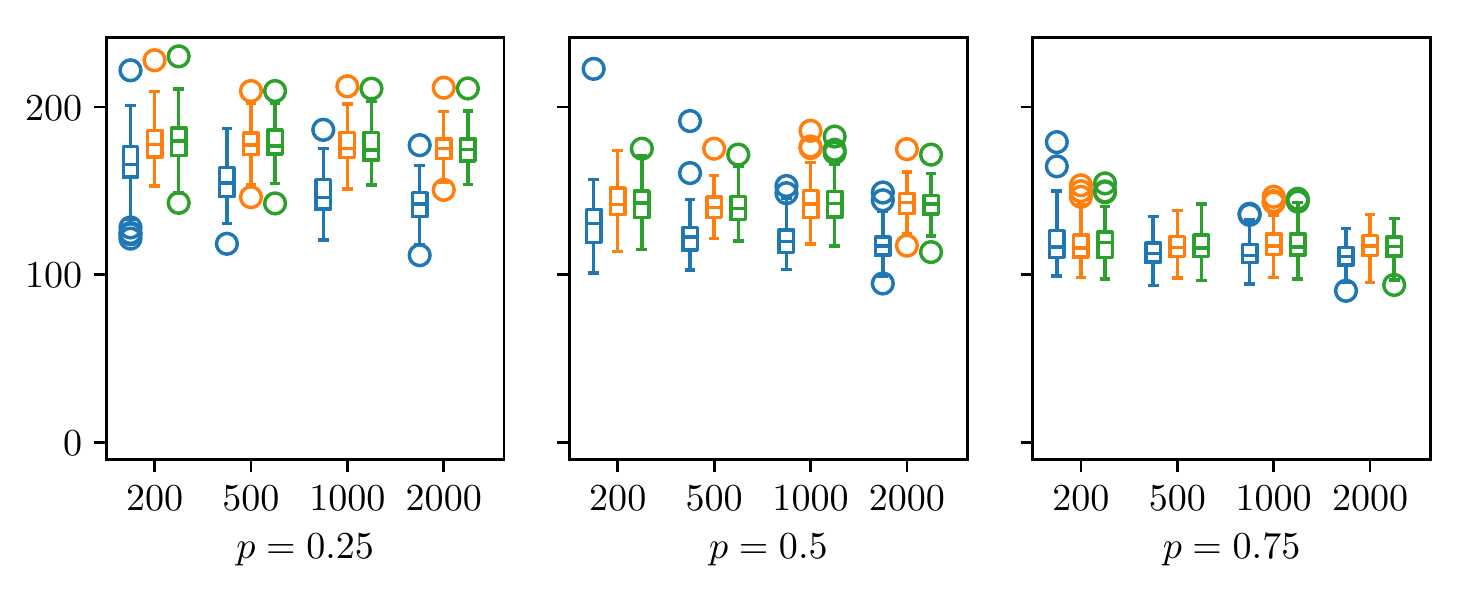}
        \caption{Linear Regression}
  \end{subfigure}
  \begin{subfigure}[b]{0.9\textwidth}
      \centering
        \includegraphics[width=\textwidth]{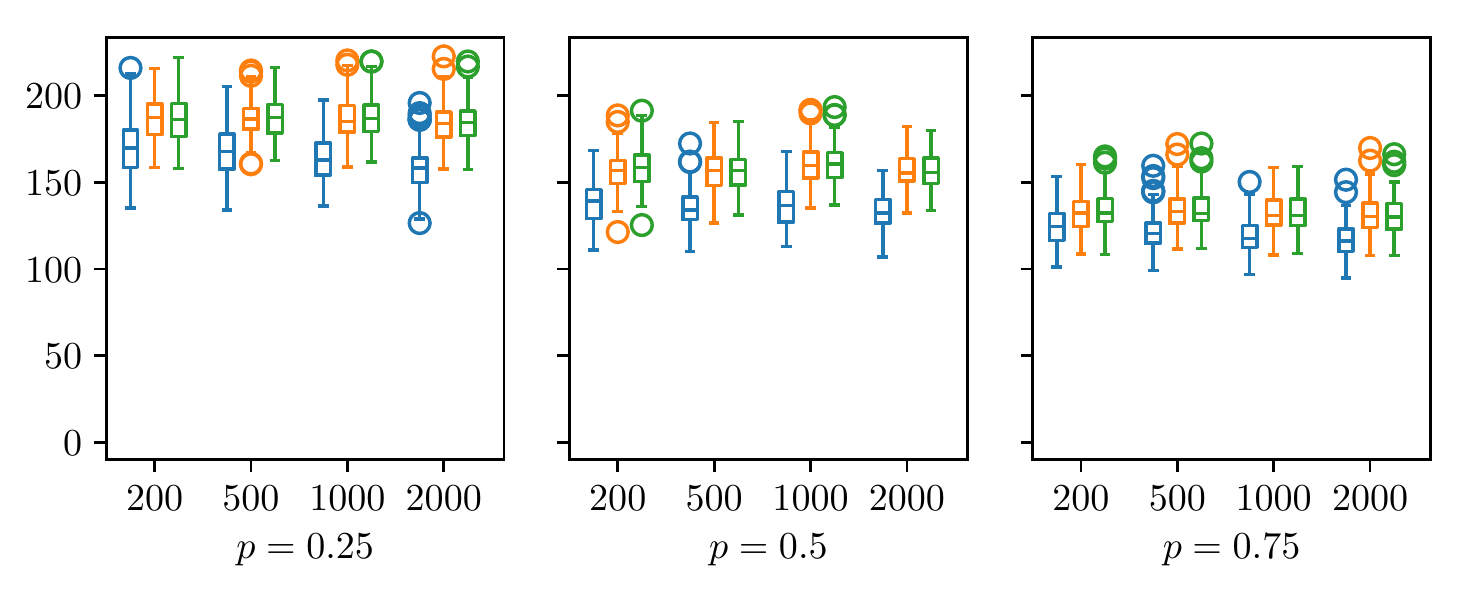}
        \caption{SVR}
  \end{subfigure}
  \begin{subfigure}[b]{0.9\textwidth}
      \centering
        \includegraphics[width=\textwidth]{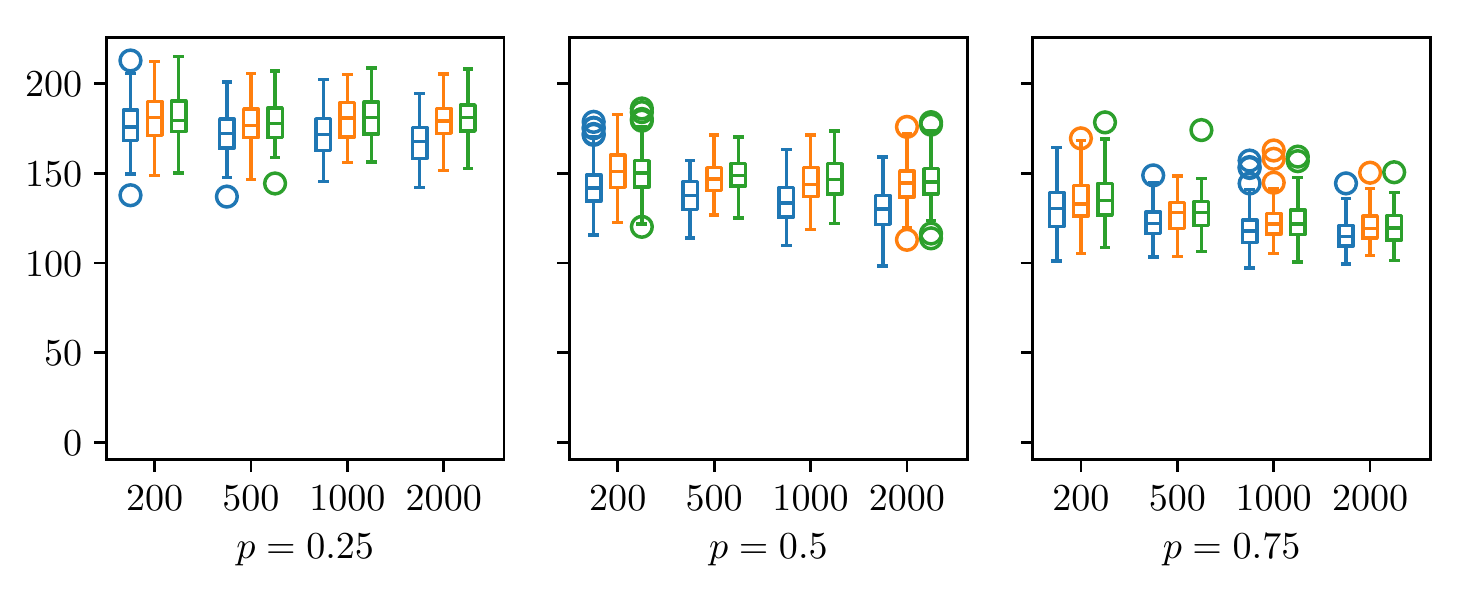}
        \caption{Random Forest}
  \end{subfigure}
  \caption{$L^2$ error of the IPCW estimator compared to naive methods for $d=8$}
\end{figure}

\begin{figure}[htbp]
      \centering
      \begin{subfigure}[b]{0.9\textwidth}
      \centering
            \includegraphics[width=0.9\textwidth]{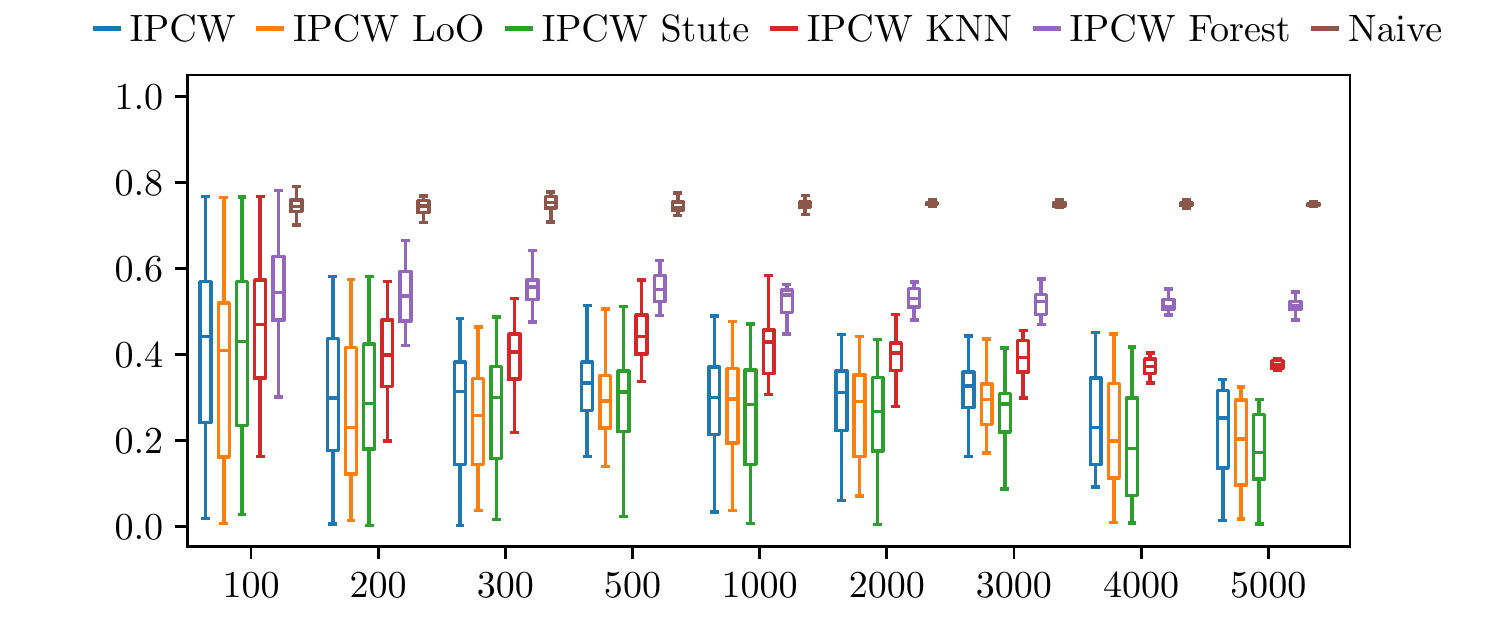}
            \caption{$d = 2$}
      \end{subfigure}
      \begin{subfigure}[b]{0.9\textwidth}
      \centering
            \includegraphics[width=0.9\textwidth]{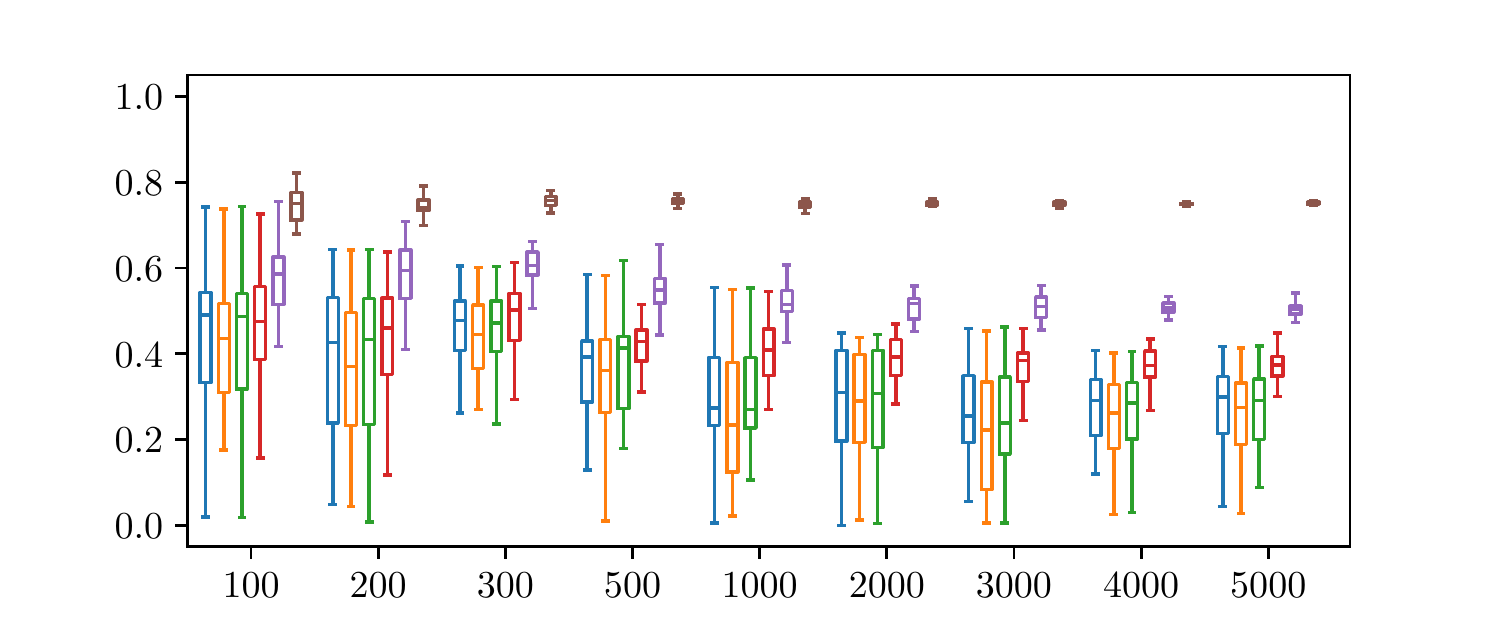}
            \caption{$d = 4$}
      \end{subfigure}
      \begin{subfigure}[b]{0.9\textwidth}
      \centering
            \includegraphics[width=0.9\textwidth]{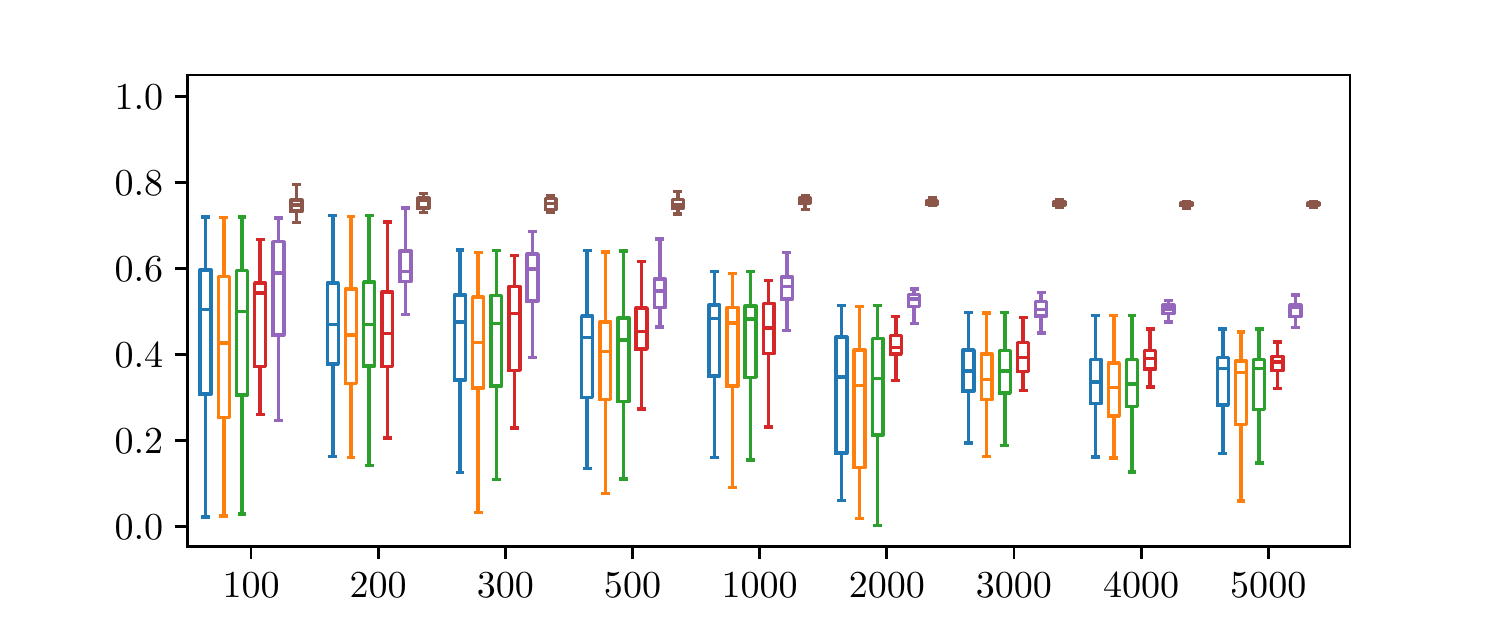}
            \caption{$d = 8$}
      \end{subfigure}
      \caption{Estimated $L^2$ error of the IPCW estimator compared to naive methods for $p = 0.25$}
\end{figure}

\clearpage
\bibliography{censor_erm}

\end{document}